\documentclass[twoside,11pt]{article}

\usepackage{blindtext}

\usepackage[abbrvbib, preprint]{jmlr2e}

\usepackage{lastpage}
\jmlrheading{}{2026}{1-\pageref{LastPage}}{3/26; Revised }{}{25-1917}{Yingzhen Yang and Ping Li}

\hypersetup{colorlinks=true,citecolor=blue,linkcolor=blue}
\newcommand{\circled}[1]{\small{\raisebox{.6pt}{\textcircled{\raisebox{-.8pt}{#1}}}}}

\usepackage{amsmath,mathrsfs,dsfont}
\usepackage{nicefrac}
\usepackage{algorithm}
\usepackage{algorithmicx}
\usepackage{algpseudocode}

\usepackage{booktabs}

\usepackage{color}
\usepackage{enumitem}

\PassOptionsToPackage{square,sort,comma,numbers}{natbib}
\usepackage{graphicx,tikz}
\usepackage[mathscr]{euscript}
\usepackage{amsthm}

\usepackage{bm}
\usepackage{bbm}
\usepackage{color}

\usepackage{color}
\usepackage{epstopdf}
\usepackage{subcaption}
\usepackage[capitalize,noabbrev]{cleveref}
\usepackage{thmtools}
\usepackage{thm-restate}

\usepackage{bbding}
\usepackage{mathtools}
\usepackage{wrapfig}
\usepackage{colortbl}
\usepackage{slashbox}
\usepackage[htt]{hyphenat}
\allowdisplaybreaks
\newcommand{\cfrakR}{\mathfrak{R}} 
\newcommand{\relu}[1]{\sigma\pth{#1}}

\newcommand{\Span}{\mathop\mathrm{Span}}
\newcommand{\tK}{\tilde K}

\newcommand{\bx}{\tilde \bx}
\newcommand{\by}{{\tilde \by}}

\newcommand{\bbx}{\overset{\rightharpoonup}{\bx}}

\newcommand{\bbw}{\overset{\rightharpoonup}{\bw}}
\newcommand{\bbe}{\overset{\rightharpoonup}{\be}}

\newcommand{\cFnn}{\cF_{\mathop\mathrm{NN}}}

\newcommand{\bPr}{\bP^{(r_0)}}
\newcommand{\bUr}{\bU^{(r_0)}}
\newcommand{\bUminusr}{\bU^{(-r_0)}}

\graphicspath{{illustrations/}}
\newcounter{optproblem}

\newtheoremstyle{mytheoremstyle} 
    {\topsep}                    
    {\topsep}                    
    {\normalfont}                
    {}                           
    {\bfseries}                   
    {.}                          
    {.5em}                       
    {}  

\theoremstyle{mytheoremstyle}
\newtheorem{theorem}{Theorem}[section]
\newtheorem{remark}[theorem]{Remark}
\newtheorem{proposition}[theorem]{Proposition}

\newtheorem*{theorem*}{Theorem}
\newtheorem*{lemma*}{Lemma}
\newtheorem*{remark*}{Remark}
\newtheorem*{claim*}{Claim}

\newtheorem{lemma}[theorem]{Lemma}

\theoremstyle{remark}
\newtheorem{definition}{Definition}[section]

\DeclareMathAlphabet{\pazocal}{OMS}{zplm}{m}{n}
\DeclareMathAlphabet{\mathpzc}{OMS}{pzc}{m}{it}

\setlist[itemize]{leftmargin=*}

\renewcommand{\hat}{\widehat}

\newcommand{\bfm}[1]{\ensuremath{\mathbf{#1}}}
\newcommand{\bfsym}[1]{\ensuremath{\boldsymbol{#1}}}

\def\ba{\boldsymbol a}   \def\bA{\bfm A}

   \def\bD{\bfm D}  
\def\be{\bfm e}   \def\bE{\bfm E}  
  \def\bF{\bfm F}  
     
   \def\bH{\bfm H}  
   \def\bI{\bfm I}  
     
   \def\bK{\bfm K}

     \def\NN{\mathbb{N}}
     
   \def\bP{\bfm P}  
   \def\bQ{\bfm Q}  
     \def\RR{\mathbb{R}}
   \def\bS{\bfm S}  
     
\def\bu{\bfm u}   \def\bU{\bfm U}  
\def\bv{\bfm v}     
\def\bw{\bfm w}   \def\bW{\bfm W}  
\def\bx{\bfm x}   \def\bX{\bfm X}  
\def\by{\bfm y}     
   \def\bZ{\bfm Z}  
\def\bzero{\bfm 0}

 \def\cE{{\cal  E}}
 \def\cF{{\cal  F}}
 
 \def\cH{{\cal  H}}

 \def\cN{{\cal  N}}
 \def\cO{{\cal  O}}
 \def\cP{{\cal  P}}

 \def\cS{{\cal  S}}

 \def\cV{{\cal  V}}
 \def\cW{{\cal  W}}
 \def\cX{{\cal  X}}

\def\balpha{\bfsym \alpha}

\def\bmu{\bfsym {\mu}}

\def\btheta{\bfsym {\theta}}

\def\bsigma{\bfsym \sigma}
\def\bSigma{\bfsym \Sigma}

\def\hlambda{\hat{\lambda}}

\def\+#1{\mathcal{#1}}
\def\-#1{\textup{#1}}

\def\set#1{\left\{ #1 \right\}}
\def\pth#1{\left( #1 \right)}
\def\bth#1{\left[ #1 \right]}
\def\abth#1{\left | #1 \right |}

\def\defeq {\coloneqq}

\newcommand{\overbar}[1]{\mkern 1.5mu\overline{\mkern-1.5mu#1\mkern-1.5mu}\mkern 1.5mu}

\newcommand{\vect}[1]{{\textup{vec}\pth{#1}}}

\DeclareMathSymbol{\relcolon}{\mathrel}{operators}{"3A}

\newcommand{\La}{\left\langle\kern-0.64ex\left\langle}
\newcommand{\Ra}{\right\rangle\kern-0.64ex\right\rangle}

\def\Norm#1#2{{\left\vert\kern-0.4ex\left\vert\kern-0.4ex\left\vert #1
    \right\vert\kern-0.4ex\right\vert\kern-0.4ex\right\vert}_{#2}}
\def\norm#1#2{{\left\|#1\right\|}_{#2}}

\def\ltwonorm#1{\norm{#1}{2}}
\def\fnorm#1{\norm{#1}{\textup{F}}}
\def\supnorm#1{\norm{#1}{\infty}}

\def \Proj  {\mathbb{P}}

\def\tr#1{\textup{tr}\left(#1\right)}

\newcommand{\1}{{\rm 1}\kern-0.25em{\rm I}}
\def\indict#1{{\rm 1}\kern-0.25em{\rm I}_{\set{#1}}}

\DeclarePairedDelimiter\floor{\lfloor}{\rfloor}

\def \eps  {\epsilon}
\def \eps {\varepsilon}

\def \diff {{\rm d}}
\def \iprod#1#2{\left\langle #1, #2 \right\rangle}

\def\set#1{\left\{#1\right\}}

\def\ball#1#2#3{\bfm{B}^{#1}\left(#2; #3\right)}

\def\unitsphere#1{\mathbb{S}^{#1}}

\def \E {\mathbb{E}}
\def\Expect#1#2{\E_{#1}\left[#2\right]}

\def \Pr {\textup{Pr}}
\newcommand{\Prob}[1]{\Pr\left[#1\right]}
\def \Var#1{\textup{Var}\left[#1\right]}

\def \lsim {\lesssim}
\def \gsim {\gtrsim}

\newcommand{\Unif}[1]{{\mathop\mathrm{Unif}}\left( #1 \right)}

\newcommand{\beq}{\begin{equation}}
\newcommand{\eeq}{\end{equation}}
\newcommand{\beqa}{\begin{eqnarray}}
\newcommand{\eeqa}{\end{eqnarray}}
\newcommand{\beqas}{\begin{eqnarray*}}
\newcommand{\eeqas}{\end{eqnarray*}}
\def\bal#1\eal{\begin{align}#1\end{align}}
\def\bals#1\eals{\begin{align*}#1\end{align*}}
\def\bsal#1\esal{\begin{small}\begin{align}#1\end{align}\end{small}}
\def\bsals#1\esals{\begin{small}\begin{align*}#1\end{align*}\end{small}}
\def\bsfal#1\esfal{\begin{small}\begin{flalign}#1\end{flalign}\end{small}}

\newcommand{\BlackBox}{\rule{1.5ex}{1.5ex}}  
\ifdefined\proof
    \renewenvironment{proof}{\par\noindent{\bf Proof\ }}{\hfill\BlackBox\\[2mm]}
\else
    \newenvironment{proof}{\par\noindent{\bf Proof\ }}{\hfill\BlackBox\\[2mm]}
\fi

\ShortHeadings{PGD Finds Neural Networks with Optimal Rate for
Nonparametric Regression}{Yang and Li}
\firstpageno{1}

\begin{document}

\title{Gradient Descent with Projection Finds Over-Parameterized Neural Networks for Learning Low-Degree Polynomials with Nearly Minimax Optimal Rate}

\author{\name Yingzhen Yang \email yingzhen.yang@asu.edu \\
       \addr School of Computing and Augmented Intelligence\\
       Arizona State University\\
       Tempe, AZ 85281, USA
       \AND
       \name Ping Li \email pingli98@gmail.com  \\
       \addr VecML Inc.\\
       Bellevue, WA 98004, USA}

\editor{My editor}

\maketitle

\begin{abstract}
We study the problem of learning a low-degree spherical polynomial of degree $k_0 = \Theta(1) \ge 1$ defined on the unit sphere in $\RR^d$ by training
an over-parameterized two-layer neural network with augmented
feature in this paper.
Our main result is the significantly improved sample complexity for learning such low-degree polynomials.
We show that, for any regression risk $\eps \in (0, \Theta(d^{-k_0})]$,
an over-parameterized two-layer neural network trained by a novel Gradient Descent with Projection (GDP) requires a sample complexity of
$n \asymp \Theta( \log(4/\delta) \cdot d^{k_0}/\eps)$ with probability $1-\delta$ for $\delta \in (0,1)$, in contrast with the representative sample complexity $\Theta(d^{k_0} \max\set{\eps^{-2},\log d})$. Moreover, such sample complexity is nearly unimprovable since the trained network renders a nearly optimal rate of the nonparametric regression risk of the order $\log({4}/{\delta}) \cdot \Theta(d^{k_0}/{n})$
with probability at least $1-\delta$. On the other hand, the minimax optimal rate for the regression risk with a kernel of rank $\Theta(d^{k_0})$ is
$\Theta(d^{k_0}/{n})$, so that
the rate of the nonparametric regression risk of the network trained by GDP is nearly minimax optimal. In the case that the ground truth degree $k_0$ is unknown, we present a novel and provable adaptive degree selection algorithm which identifies the true degree and achieves the same nearly optimal regression rate.
To the best of our knowledge, this is the first time that a nearly optimal risk bound is obtained by training an over-parameterized neural network with a popular activation function (ReLU) and algorithmic guarantee for learning low-degree spherical polynomials.
Due to the feature learning capability of GDP, our results are beyond the regular Neural Tangent Kernel (NTK) limit.
\end{abstract}

\vspace{0.3in}

\newpage

\section{Introduction}

With the success of deep learning across machine learning~\citep{YannLecunNature05-DeepLearning}, understanding neural network generalization has become a central topic. Prior work shows that gradient-based methods such as GD and SGD can achieve vanishing training loss in deep networks~\citep{du2018gradient-gd-dnns,AllenZhuLS19-convergence-dnns,DuLL0Z19-GD-dnns,AroraDHLW19-fine-grained-two-layer,ZouG19,SuY19-convergence-spectral}. Beyond optimization, extensive studies provide generalization guarantees for DNNs trained with gradient methods. A key insight is that, under sufficient over-parameterization, training dynamics are well approximated by kernel methods, notably the Neural Tangent Kernel (NTK)~\citep{JacotHG18-NTK}, although infinite-width networks can still exhibit feature learning~\citep{YangH21-feature-learning-infinite-network-width}. In this regime, network weights remain close to initialization, allowing the network to be approximated by a first-order Taylor expansion and enabling tractable generalization analysis~\citep{CaoG19a-sgd-wide-dnns,AroraDHLW19-fine-grained-two-layer,Ghorbani2021-linearized-two-layer-nn}.

The generalization of neural networks can be analyzed through their ability to learn low-degree polynomials, motivated by the spectral bias phenomenon \citep{rahaman19a-spectral-bias,CaoFWZG21-spectral-bias,ChorariaD0MC22-spectral-bias-pnns}, which states that neural networks preferentially learn functions aligned with the top eigenspaces of the NTK integral operator. For data uniformly distributed on the unit sphere $\unitsphere{d-1}\subset\RR^d$ or its scaled variant, any degree-$\ell$ polynomial admits a linear representation via spherical harmonics up to degree $\ell$, corresponding to the largest NTK eigenvalues (see Section~\ref{sec:harmonic-analysis-detail} and Theorem~\ref{theorem:spherical-polynomial-representation-spherical-harmonics} of the appendix). Recent works have focused on feature learning beyond the linear NTK regime. While infinite-width networks can learn features \citep{YangH21-feature-learning-infinite-network-width}, several approaches aim to escape NTK linearization to learn low-degree polynomials. The QuadNTK framework \citep{BaiL20-quadratic-NTK}, based on second-order Taylor expansion, achieves improved generalization and efficiently learns sparse and ``one-directional'' polynomials. Extending this idea, \citet{Nichani0L22-escape-ntk} shows that combining NTK and QuadNTK enables learning dense polynomials with an additional sparse high-degree component. Other approaches include two-stage optimization for polynomial learning \citep{DamianLS22-nn-representation-learning} and mean-field analyses of two-layer networks \citep{TakakuraS24-mean-field-two-layer}.

However, the analysis about the sharpness of the regression risk in the current results about training over-parameterized neural networks to learn low-degree polynomials, such as \citet{Ghorbani2021-linearized-two-layer-nn,BaiL20-quadratic-NTK,Nichani0L22-escape-ntk,DamianLS22-nn-representation-learning,TakakuraS24-mean-field-two-layer}, is largely missing. For example,
\citet{Nichani0L22-escape-ntk} show that a regression risk $\eps$ is achieved when the training data size $n$ satisfies $n \gsim d^{k_0} \max\set{\eps^{-2},\log d}$. Without training a neural network,  \citet{Ghorbani2021-linearized-two-layer-nn}
show that when $\tilde \Theta(d^{k_0}) \le n \le \Theta(d^{k_0+1-\delta})$ where $\tilde \Theta(d^{k_0})/d^{k_0} \to \infty$ as $d \to \infty$, the regression risk achieved by NTK alone under certain restrictive conditions converges to $0$ as $d \to \infty$ without concrete convergence rates or the sharpness of such a risk. Furthermore, under the popular setting where $d$ is fixed used by recent works about sharp rates for
nonparametric regression~\citep{HuWLC21-regularization-minimax-uniform-spherical,SuhKH22-overparameterized-gd-minimax,
yang2024gradientdescentfindsoverparameterized,Li2024-edr-general-domain}, \citet{Ghorbani2021-linearized-two-layer-nn} cannot even show a vanishing regression risk.

Understanding the sharpness of regression risk when learning low-degree polynomials remains an important problem in statistical learning and theoretical deep learning. In this work, we assume the target function $f^*$ lies in the Reproducing Kernel Hilbert Space (RKHS) induced by an over-parameterized two-layer neural network with a bounded RKHS norm, where $f^*$ is a degree-$k_0$ polynomial on the unit sphere $\unitsphere{d-1} \subset \RR^d$ with $k_0 \ge 1$. Our main result (Theorem~\ref{theorem:LRC-population-NN-fixed-point}) shows that when the network is trained using a novel Gradient Descent with Projection (GDP) and the sample size satisfies $n \ge \Theta(\log(4/\delta)\cdot d^{2k_0})$, the resulting estimator achieves a nearly optimal nonparametric regression risk of order $\log(4/\delta)\cdot \Theta(d^{k_0}/n)$ with probability at least $1-\delta$. Since the minimax optimal risk for kernel regression with a PSD kernel of rank $r=\Theta(d^{k_0})$ is $\Theta(r/n)=\Theta(d^{k_0}/n)$ \citep[Theorem~2(a)]{RaskuttiWY12-minimax-sparse-additive}, our bound is nearly minimax optimal. To the best of our knowledge, this is the first nearly optimal risk bound with algorithmic guarantees for learning low-degree spherical polynomials via training an over-parameterized neural network with a popular activation function (ReLU). While related forms of projected gradient methods have been studied for exploiting low-dimensional structure in over-parameterized models \citep{Xu2023-PGD-overparameterized-low-rank-matrix-sensing,Zhang2023-PGD-overparameterized-nonconvex-factorization}, this work is among the first to design a GDP specifically tailored to over-parameterized neural networks that achieves a nearly optimal rate.

\vspace{0.1in}
\noindent \textbf{Feature Learning Capability of GDP.} We remark that our results go beyond the regular NTK limit due to the feature learning capability of GDP from two aspects. First, while conventional NTK-based analysis must consider all eigenspaces of the NTK, our GDP introduces a novel projection operator which ensures that the learned neural network function lies in a low-dimensional subspace of the RKHS associated with the NTK for a nearly optimal regression risk bound, when the ground truth degree of the target function, $k_0$, is known. This constitutes the first main result detailed in Section~\ref{sec:sharp-risk-bound}. When $k_0$ is unknown, based on the first main result, we present a novel and provable adaptive degree selection algorithm which identifies the true degree $k_0$ and trains a neural network with the same nearly optimal regression risk, which is our second main result detailed in Section~\ref{sec:degree-selection}.
Thanks to the feature learning capability of our method, our result is stronger than the literature~\citep{WeiLLM19-regularization,Glasgow24-SGD-features-xor,LeeOSW24-low-dim-polynomials,AbbeAM22-SGD-merged-staircase} in terms of learning general low-degree spherical polynomials, including existing works based on the feature learning capability of neural networks.
For example, existing works~\citep{WeiLLM19-regularization,Glasgow24-SGD-features-xor,LeeOSW24-low-dim-polynomials,AbbeAM22-SGD-merged-staircase} do not address the regression setting in which the target function is a degree-$k_0$ spherical polynomial and the regression risk attains the sharp and minimax-optimal rate $\Theta(d^{k_0}/n)$. In particular, \citet{WeiLLM19-regularization} does not study regression with polynomial target functions at all, while the results of~\citet{Glasgow24-SGD-features-xor} are restricted to a highly specific setting where the target function is a quadratic XOR function. In~\citet{LeeOSW24-low-dim-polynomials}, the target function takes the single-index form $f^*(\bx)=\sigma^*(\langle \bx,\btheta\rangle)$, where $\sigma^*$ has information exponent $p$, thereby limiting $f^*$ to be a polynomial along a single direction parameterized by $\btheta$, rather than a general non-single-index spherical polynomial as considered in this paper. Finally, \citet{AbbeAM22-SGD-merged-staircase} investigates the case where the target function $f^*$ is a low-dimensional latent function of dimension $P$ embedded in an ambient space of dimension $d$ with $P\le d$, and establishes necessary and nearly sufficient conditions under which $f^*$ is strongly SGD-learnable in the mean-field regime.

Beyond feature learning methods that escape the linear NTK regime (Table~\ref{table:main-results-comparison}), sharp minimax convergence rates for nonparametric kernel regression are well established in the statistical learning literature~\citep{Stone1985,Yang1999-minimax-rates-convergence,RaskuttiWY14-early-stopping-kernel-regression,Yuan2016-minimax-additive-models}.
By training over-parameterized shallow \citep[Theorem 5.2]{HuWLC21-regularization-minimax-uniform-spherical} or deep \citep[Theorem 3.11]{SuhKH22-overparameterized-gd-minimax} neural networks with training features following spherical uniform distribution on the unit sphere, these results \citep{HuWLC21-regularization-minimax-uniform-spherical, SuhKH22-overparameterized-gd-minimax} show that minimax optimal rate $\cO(n^{-{d}/(2d-1)})$ is achieved for the regression risks when the target function is in $\cH_{\tilde K}(\gamma_0)$ where $\tilde K$ is the NTK of a specific neural network studied in each work.

We organize this paper as follows. With the necessary notation introduced below, we first introduce in Section~\ref{sec:setup} the problem setup. Our main results are summarized in Section~\ref{sec:summary-main-results}. The roadmap of proofs, the summary of the approaches and the key technical results in the proofs, and the novel proof strategy of this work are presented in Section~\ref{sec:proof-roadmap}. 

\vspace{0.1in}
\noindent \textbf{Notations.} We use bold letters for matrices and vectors, and regular lower letters for scalars throughout this paper. $\bA^{[i]}$ is the $i$-th column of a matrix $\bA$.  A bold letter with subscripts indicates the corresponding rows or elements of a matrix or a vector. We put an arrow on top of
a letter with subscript if it denotes a vector, e.g.,
$\bbx_i$ denotes the $i$-th training
feature. $\norm{\cdot}{F}$ and
$\norm{\cdot}{p}$ denote the Frobenius norm and the vector $\ell^{p}$-norm or the matrix $p$-norm. $[m\colon n]$ denotes all the integers between $m$ and $n$ inclusively, and $[1\colon n]$ is also written as $[n]$. $\Var{\cdot}$ denotes the variance of a random variable. $\bI_n$ is an $n \times n$ identity matrix.  $\indict{E}$ is an indicator function which takes the value of $1$ if event $E$ happens, or $0$ otherwise. The complement of a set $A$ is denoted by $A^c$, and $\abth{A}$ is the cardinality of the set $A$. $\vect{\cdot}$ denotes the vectorization of a matrix or a set of vectors, and $\tr{\cdot}$ is the trace of a matrix.
We denote the unit sphere in $d$-dimensional Euclidean space by $\unitsphere{d-1} \defeq \{\bx \colon  \bx \in \RR^d, \ltwonorm{\bx} =1\}$. Let $\cX$ denote the input space, and
$L^p(\cX, \mu) $ with $p \ge 1$ denote the space of $p$-th power integrable functions on $\cX$ with probability measure $\mu$, and the inner product $\iprod{\cdot}{\cdot}_{L^p(\mu)}$ and $\norm{\cdot}{{L^p(\mu)}}^2$ are defined as $\iprod{f}{g}_{L^p(\mu)} \coloneqq \int_{\cX}f(x)g(x) \diff \mu(x)$ and $\norm{f}{L^p(\mu)}^p \coloneqq \int_{\cX}\abth{f}^p(x) \diff \mu (x) <\infty$. $\ball{}{\bx}{r}$ is the Euclidean closed ball centered at $\bx$ with radius $r$. Given a function $g \colon \cX \to \RR$, its $L^{\infty}$-norm is denoted by $\norm{g}{\infty} \defeq \sup_{\bx \in \cX} \abth{g(\bx)}$, and
$L^{\infty}$ is the function class whose elements have bounded $L^{\infty}$-norm. $\iprod{\cdot}{\cdot}_{\cH}$ and $\norm{\cdot}{\cH}$ denote the inner product and the norm in the Hilbert space $\cH$. $a = \cO(b)$ or $a \lsim b$ indicates that there exists a constant $c>0$ such that $a \le cb$. $\tilde \cO$ indicates there are specific requirements in the constants of the $\cO$ notation. $a = o(b)$ and $a = w(b)$ indicate that $\lim \abth{a/b}  = 0$ and $\lim \abth{a/b}  = \infty$, respectively. $a \asymp b$  or $a = \Theta(b)$ denotes that
there exists constants $c_1,c_2>0$ such that $c_1b \le a \le c_2b$. $\Unif{\unitsphere{d-1}}$ denotes the uniform distribution on $\unitsphere{d-1}$.
The constants defined throughout this paper may change from line to line.
We use $\Expect{P}{\cdot}$ to denote the expectation with respect to the distribution $P$.
$\bP_{\cS}$ denotes the orthogonal projection onto the space $\cS$, and $\Span(\bA)$ denotes the linear space spanned by the columns of the matrix $\bA$. $\overline{A}$ denotes the closure of a set $A$. Throughout this paper we let the input space be $\cX = \unitsphere{d-1}$.

\section{Problem Setup}
\label{sec:setup}
We introduce the problem setup for nonparametric regression with the target function as a low-degree spherical polynomial in this section.
\subsection{Two-Layer Neural Network}
\label{sec:setup-two-layer-nn}

We are given the training data $\set{(\bbx_i,  y_i)}_{i=1}^n$ where each data point is a tuple of feature vector $\bbx_i \in \cX$ and its response $ y_i \in \RR$. Throughout this paper we assume
that no two training features coincide, that is, $\bbx_i \neq \bbx_j$ for all $i,j \in [n]$ and $i \neq j$.
We denote the training feature vectors by $\bS = \set{\bbx_i}_{i=1}^n$, and denote by $P_n$ the empirical distribution over $\bS$. All the responses are stacked as a vector $ \by = [ y_1, \ldots,  y_n]^\top \in \RR^n$.
The response $ y_i$ is given by $ y_i= f^*(\bbx_i) + w_i$ for
$i \in [n]$,
where $\set{w_i}_{i=1}^n$ are i.i.d. sub-Gaussian random noise with mean $0$ and variance proxy $\sigma_0^2$, that is, $\Expect{}{\exp(\lambda w_i)} \le \exp(\lambda^2 \sigma_0^2/2)$ for any $\lambda \in \RR$.
$f^*$ is the target function to be detailed later. We define $\by \defeq \bth{y_1,\ldots,y_n}$, $\bw \defeq \bth{w_1,\ldots,w_n}^{\top}$, and use $f^*(\bS) \defeq \bth{f^*(\bbx_1),\ldots,f^*(\bbx_n)}^{\top}$ to denote the clean target labels.
The feature vectors in $\bS$ are drawn i.i.d. according to the data distribution $P = \Unif{\unitsphere{d-1}}$ with $\mu$ being the probability measure for $P$.
We consider a two-layer neural network (NN) with an augmented feature in this paper whose
mapping function is
\bal\label{eq:two-layer-nn}
&f(\cW,\bx) =  \frac{1}{\sqrt{m}}\sum_{r=1}^{m}
a_r \relu{{\bbw_r}^\top \bx} +
\frac{1}{\sqrt{m}}
\bbw_{m+1}^{\top} \bF(\bW(0),\bx),
\eal%
where $\bx \in \cX$ is the input, $\sigma(\cdot) = \max\set{\cdot,0}$ is the ReLU activation function. $\cW = \set{\bW,\bbw_{m+1}}$ denotes the weights of the network, $\bW = \set{\bbw_r}_{r=1}^{m}$ and $\bbw_{m+1}$ are the weight vectors of the first layer, and $m$ is the number of neurons which is also termed the network width, where  $\bbw_r \in \RR^d$ for $r \in [m]$ and $\bbw_{m+1} \in \RR^m$.  $\bF(\bW(0),\bx) \in \RR^m$ is a feature vector computed at the initialization with
$\bth{\bF(\bW(0),\bx)}_r = \indict{{\bbw_r(0)}^\top \bx \ge 0}$ for $r \in [m]$, which is termed the augmented feature.
$\ba = \bth{a_1, \ldots, a_m} \in \RR^m$ denotes the weights of the second layer. Throughout this paper we also write $\bW,\bw_r$ as $\bW_{\bS},\bw_{\bS,r}$ from time to time so as to indicate that the weights are trained on the training features $\bS$.

\vspace{0.1in}
\noindent \textbf{Novel Augmented Feature Compared to the Regular ReLU Network. } It can be observed from (\ref{eq:two-layer-nn}) that compared to the regular two-layer
ReLU network $f^{\textup{(vanilla)}}(\bW,\bx) =  \frac{1}{\sqrt{m}}\sum_{r=1}^{m}
a_r \relu{{\bbw_r}^\top \bx}$, our network has the additional augmented feature map by
${1}/{\sqrt{m}} \cdot
\bbw_{m+1}^{\top} \bF(\bW(0),\bx)$. Such additional feature map ensures that the NTK associated with the network (\ref{eq:two-layer-nn}) is a PSD kernel $K$, to be defined in (\ref{eq:kernel-two-layer}), and all of its eigenvalues are strictly positive as shown in Theorem~\ref{theorem:eigenvalue-NTK} deferred to Section~\ref{sec:NTK-eigenvalues} of the appendix. In contrast, as shown in~\citet[Proposition 5]{BiettiM19}, the eigenvalues $\set{\tilde \lambda_j}_{j \ge 0}$ of the integral operator associated with the NTK of the regular network $f^{\textup{(vanilla)}}$ have the property that $\tilde \lambda_{2t+1} = 0$ for $t \ge 1$. As a result, the eigenspaces of the NTK of the network $f^{\textup{(vanilla)}}$ corresponding to nonzero eigenvalues do not cover all the spherical harmonics of order $2t+1$ for all $t \ge 1$, limiting its capability of learning spherical polynomials with spherical harmonics of odd degrees ($\ge 3$) as their components.

\subsection{Kernel and Kernel Regression for Nonparametric Regression}
\label{sec:setup-kernels-target-function}

We define the following kernel functions. For all
$\bu, \bv \in \cX$,
\bal\label{eq:kernel-two-layer}
K^{(0)}(\bu,\bv) \defeq  \frac{\pi -\arccos (\bu^{\top}\bv) }{2\pi}, \,\,\,
K^{(1)}(\bu,\bv) \defeq \bu^{\top}\bv K^{(0)}(\bu,\bv), \,\,\,
K = K^{(0)} + K^{(1)},
\eal%
which is in fact the NTK associated with
the two-layer NN (\ref{eq:two-layer-nn}) with constant second layer weights $\ba$, and $K$ is a PSD kernel. Let the Gram matrix of $K$ over the training features $\bS$ be $\bK \in \RR^{n \times n}, \bK_{ij} = K(\bbx_i,\bbx_j)$ for $i,j \in [n]$,
and $\bK_n \defeq \bK/n$ is the empirical NTK matrix. $\bK^{(\alpha)},\bK_n^{(\alpha)}$ for $\alpha = 0,1$ are defined similarly. Let the eigendecomposition  of $\bK_n$ be $\bK_n = \bU \bSigma {\bU}^{\top}$ where $\bU$ is a $n \times n$ orthogonal matrix, and $\bSigma$ is a diagonal matrix with its diagonal elements $\set{\hlambda_i}_{i=1}^n$ being eigenvalues of $\bK_n$ and sorted in a non-increasing order.  It is proved in existing works, such as~\citet{du2018gradient-gd-dnns}, that $\bK_n$ is non-singular.
Since $\sup_{\bx \in \cX} K(\bx,\bx) = 1$, it can be verified that $\hlambda_1 \in (0, 1)$. Let $\cH_{K}$ be the Reproducing Kernel Hilbert Space (RKHS) associated with
$K$. Because $K$ is continuous on the compact set $\cX \times \cX$, the integral operator $T_K \colon L^2(\cX,\mu) \to L^2(\cX,\mu), \pth{T_K f}(\bx) \defeq \int_{\cX} K(\bx,\bx') f(\bx') \diff \mu(\bx')$ is a positive, self-adjoint, and compact operator on $L^2(\cX,\mu)$. By the spectral theorem, there is a countable orthonormal basis $\set{e_j}_{j \ge 0} \subseteq L^2(\cX,\mu)$ and $\set{\lambda_j}_{j \ge 0}$ with $1 \ge \lambda_0 \ge \lambda_1 \ge \ldots > 0$ such that $e_j$ is the eigenfunction of $T_K$ with
$\lambda_j$ being the corresponding eigenvalue. That is, $T_K e_j = \lambda_j e_j, j \ge 0$. Let $\set{\mu_{\ell}}_{\ell \ge 0}$ be the distinct eigenvalues
associated with $T_K$, and let
$m_{\ell}$ be the sum of multiplicities of
the eigenvalue $\set{\mu_{\ell'}}_{\ell'=0}^{\ell}$.
That is, $m_{\ell'} - m_{\ell'-1}$ is the multiplicity
of $\mu_{\ell'}$. It is well known that $\set{v_j = \sqrt {\lambda_j} e_j}_{j\ \ge 0}$ is an orthonormal basis of $\cH_K$. For a positive constant $\gamma_0$, we define $\cH_{K}(\gamma_0) \defeq \set{f \in \cH_{K} \colon \norm{f}{\cH} \le \gamma_0}$ as the closed  ball in $\cH_K$ centered at $0$ with radius $\gamma_0$. We note that $\cH_{K}(\gamma_0)$ is also specified by $\cH_{K}(\gamma_0) = \set{f \in L^2(\cX,\mu)\colon f = \sum_{j =0}^{\infty} \beta_j e_j,
\sum_{j = 0}^{\infty} \beta_j^2/\lambda_j \le \gamma_0^2}$. To be shown in Theorem~\ref{theorem:eigenvalue-NTK}, the eigenfunctions $\set{e_j}_{j \ge 0}$ are the spherical harmonics on $\cX$.
Let $\cH_{\bS} \defeq
\overbar{\set{\sum\limits_{i=1}^n
K(\cdot,\bbx_i) \alpha_i \colon \set{\alpha_i}_{i=1}^n
\subseteq \RR}}$ be the usual RKHS spanned by $\set{K(\cdot,\bbx_i)}_{i=1}^n$
on the data $\bS = \set{\bbx_i}_{i=1}^n$. For $r \in [n-1]$, we let $\bU = \bth{\bU^{(r)}  \,\, \bU^{(-r)} }$ where $\bU^{(r)} \in \RR^{n \times r}$ is the submatrix of $\bU$ whose columns are the
first $r$ eigenvectors of $\bK$, and the columns of $\bU^{(-r)} \in \RR^{n \times (n-r)}$ are the remaining $n-r$ eigenvectors of $\bK$. We define
\bal\label{eq:space-H-S-rank-r}
\cH_{\bS,r} \defeq \overbar{\set{\sum\limits_{i=1}^n K(\cdot,\bbx_i) \alpha_i \colon \balpha = \bth{\alpha_1, \ldots, \alpha_n}^{\top} \in \Span(\bU^{(r)})}},
\eal
which is a subspace of the Hilbert space $\cH_{\bS}$ of dimension $r$.

\vspace{0.1in}\noindent \textbf{The task of nonparametric regression.}
We consider the target function
\bal
f^*(\bx) = \sum\limits_{\ell=0}^{k_0} \sum\limits_{j=1}^{N(d,\ell)} a_{\ell j} Y_{\ell j}(\bx), \quad
\forall \bx \in \cX,
\eal
where $\set{Y_{\ell j}}_{j \in [N(d,\ell)]}$ are the spherical harmonics of degree $\ell$ which form an orthogonal basis of $\cH_{\ell}$ of dimension $N(d,\ell)$, and $\cH_{\ell}$ denotes the space of degree-$\ell$ homogeneous harmonic polynomials on $\cX$. The background about harmonic analysis on $\unitsphere{d-1}$ is deferred to
 Section~\ref{sec:harmonic-analysis-detail} of the appendix. It also follows from the discussion in  Section~\ref{sec:harmonic-analysis-detail} that
the eigenfunctions $\set{e_j}_{j\ge 0}$ associated with the kernel $K$ are in fact spherical harmonics of all degrees: $\set{e_j}_{j\ge 0} = \set{Y_{\ell j}}_{\ell \ge 0, j \in [N(d,k)]}$.
Throughout this paper we assume $f^* \in \cF^*$ where $\cF^*$ is a function class defined on $\cX$ specified by
\bal\label{def:func-class-target}
\cF^* = \set{\sum\limits_{\ell=0}^{k_0} \sum\limits_{j=1}^{N(d,\ell)} a_{\ell j} Y_{\ell j} \colon \sum_{\ell=0}^{k_0} \sum_{j=1}^{N(d,\ell)} a_{\ell j}^2/\mu_{\ell}\le \gamma_0^2 }.
\eal
It follows from Theorem~\ref{theorem:spherical-polynomial-representation-spherical-harmonics} in Section~\ref{sec:harmonic-analysis-detail} of the appendix that  $\cF^*$
comprises all polynomials of degree up to $k_0$ defined on $\cX$ with a finite $\cH_K$-norm of $\gamma_0$.
The task of the analysis for nonparametric regression is to find an estimator $\hat f$ from the training data $\set{(\bbx_i,  y_i)}_{i=1}^n$ so that
the risk $\Expect{P}{\pth{\hat f - f^*}^2}$ vanishes at a fast rate. In this work, we aim to establish a sharp rate of the risk where the over-parameterized neural network (\ref{eq:two-layer-nn}) trained by GDP serves as the estimator $\hat f$.

\vspace{0.1in}
\noindent \textbf{Minimax Lower Risk Bound for Learning a Low-Degree Spherical Polynomial.} It follows from the definition of $\cF^*$ in (\ref{def:func-class-target}) and Theorem~\ref{theorem:eigenvalue-NTK} that $\cF^*$ belongs to the subspace formed by the union of the spaces of homogeneous harmonic polynomials up to degree $k_0 = \Theta(1)$, that is, $\cF^* \subseteq \cup_{\ell=0}^{k_0} \cH_{\ell}$. As a result, if we define a low-rank kernel of finite rank $r_0 \defeq m_{k_0} = \sum_{\ell'=0}^{k_0} N(d,\ell')$ by
\bal\label{eq:kernel-low-rank}
K^{(r_0)}(\bx,\bx') \defeq \sum\limits_{\ell=0}^{k_0} \sum\limits_{j=1}^{N(d,\ell)} \mu_{\ell} Y_{\ell j}(\bx)Y_{\ell j}(\bx')
\eal
for any $\bx,\bx' \in \cX$, then it can be verified that $\cF^* \subseteq \cH_{K^{(r_0)}}(\gamma_0)$. It is shown in Lemma~\ref{lemma:r0-estimate} in Section~\ref{sec:harmonic-analysis-detail} of the appendix that $r_0 = \Theta(d^{k_0})$ with $k_0 = \Theta(1)$ and $d > \Theta(1)$. The established result in
\citet[Theorem 2(a)]{RaskuttiWY12-minimax-sparse-additive}
shows that the minimax lower bound for the regression risk with $K^{(r_0)}$ is then $\Theta(r_0/{n}) =\Theta(d^{k_0}/{n})$.

%

\subsection{Training by Gradient Descent with Projection}
\label{sec:training}
In the training process of our two-layer NN (\ref{eq:two-layer-nn}), only $\bW$ is optimized, while the elements of $\ba$ are randomly initialized to $\pm 1$ with equal probabilities and then fixed during the training. The following quadratic loss function is minimized during the training process:
\bal \label{eq:obj-dnns}
L(\cW) \defeq \frac{1}{2n} \sum_{i=1}^{n} \pth{f(\cW,\bbx_i) - y_i}^2.
\eal%
In the $(t+1)$-th step of GDP with $t \ge 0$, the weights of the neural network, $\bW_{\bS}$, are updated by one-step of GDP through
\bal\label{eq:GD-two-layer-nn}
&\vect{\bW_{\bS}(t+1)} - \vect{\bW_{\bS}(t)} = - \frac{\eta}{n} \bZ_{\bS}(t)
\bPr (\hat \by(t) -  \by), \nonumber \\
&\bbw_{m+1}(t+1) - \bbw_{m+1}(t) = - \frac{\eta}{n{\sqrt m}} \bF(\bW(0),\bS)^{\top}
\bP^{(r)} (\hat \by(t) -  \by),
\eal
where $\by_i = y_i$, $\bF(\bW(0),\bS) \in \RR^{n \times m}$ denotes the feature computed at the initialization with
$\bth{\bF(\bW(0),\bS)}_i = \bF(\bW(0),\bbx_i)^{\top}$,  $\hat \by(t) \in \RR^n$ with $\bth{\hat \by(t)}_i = f(\cW(t),\bbx_i)$. The notation with the subscript $\bS$ indicates the dependence on the training features $\bS$. We also denote the neural network function $f(\cW(t),\cdot)$ as
$f_t(\cdot)$ with weights $\cW(t) = \set{\bW_{\bS}(t),\bbw_{m+1}(t)}$ obtained right after the $t$-th step of GDP.
We define $\bZ_{\bS}(t) \in \RR^{md \times n}$ which is computed by
\bal\label{eq:GD-Z-two-layer-nn}
\bth{\bZ_{\bS}(t)}_{[(r-1)d+1:rd]i} = \frac {1}{{\sqrt m}}
\indict{\bbw_r(t)^\top \bbx_i \ge 0}  \bbx_i a_r
\eal%
for all $i \in [n], r \in [m]$,
where $\bth{\bZ_{\bS}(t)}_{[(r-1)d+1:rd]i} \in \RR^d$ is a vector with elements in the $i$-th column of $\bZ_{\bS}(t)$ with indices in
$[(r-1)d+1:rd]$. $\bP^{(r)} \in \RR^{n \times n}$ is the projection matrix with $\bP^{(r)} = \bU \bSigma^{(r)} \bU^{\top}$ for $r \in [n]$,
where $\bSigma^{(r)} \in \RR^{n \times n}$ is a diagonal matrix with
$\bSigma^{(r)}_{ii} = 1$ for $i \in [r]$ and $\bSigma^{(r)}_{ii} = 0$ otherwise. With a known degree $k_0$, we set $r = r_0$ so that $\bP^{(r)} = \bPr$. We note that $\bP^{(r)}$ would not appear in regular GD updates if vanilla GD is used, and $\bP^{(r)}$ is introduced as a projection matrix so that the learned neural network function lies on a $r_0$-dimensional subspace of the RKHS $\cH_{\bS}$ to be detailed in Section~\ref{sec:proof-roadmap} for a sharp regression risk.
We employ the following symmetric random initialization also employed in
\citet{Chizat2019-lazy-training-differentiable-programming,DamianLS22-nn-representation-learning} so that $\hat \by(0) = \bzero$. In our two-layer NN, $m$ is even, $\set{\bbw_{2r'}(0)}_{r'=1}^{m/2}$ and $\set{a_{2r'}}_{r'=1}^{m/2}$ are initialized randomly and independently according to $\bbw_{2r'}(0) \sim \cN(\bzero,\kappa^2 \bI_d), a_{2r'} \sim {\textup {unif}}\pth{\left\{-1,1\right\}},\forall r' \in [m/2]$,
where $\cN(\bmu,\bSigma)$ denotes a Gaussian distribution with mean $\bmu$ and covariance $\bSigma$, ${\textup {unif}}\pth{\left\{-1,1\right\}}$ denotes a uniform distribution over $\set{1,-1}$, $0<\kappa \le 1$ controls the magnitude of initialization. We set $\bbw_{2r'-1}(0) = \bbw_{2r'}(0)$ and $a_{2r'-1} = -a_{2r}$ for all $r' \in [m/2]$. It can then be verified that $\hat \by(0) = \bzero$, that is, the initial output of the two-layer NN (\ref{eq:two-layer-nn}) is zero. We use $\bW(0)$ to denote the set of all the random weight vectors at initialization, that is, $\bW(0) = \set{\bbw_r(0)}_{r=1}^m$, and $\bbw_{m+1}(0) = \bzero$.
We run Algorithm~\ref{alg:GDP} to train the two-layer NN by GDP, where
$T$ is the total number of steps for GDP and the projection dimension $r = r_0$.

\begin{algorithm}[b!]
\renewcommand{\algorithmicrequire}{\textbf{Input:}}
\renewcommand\algorithmicensure {\textbf{Output:} }
\caption{Training the Two-Layer NN by GDP}
\label{alg:GDP}
\begin{algorithmic}[1]
\State $\bW(T) \leftarrow$ Training-by-GDP ($r,T,\bW(0)$)
\State \textbf{\bf input: } $r,T,\bW(0),\bbw_{m+1}=\bzero,\eta$
\For{$t=1,\ldots,T$}
\State  Perform the $t$-th step of GDP by (\ref{eq:GD-two-layer-nn})
\EndFor
\State \textbf{\bf return} $\bW(T),\bbw_{m+1}(T)$
\end{algorithmic}
\end{algorithm}

\section{Summary of Main Result}
\label{sec:summary-main-results}

We present the main results of this paper in this section, with the nearly optimal regression risk bound in Section~\ref{sec:sharp-risk-bound} when the ground truth degree $k_0$ is known, and the adaptive degree selection algorithm with theoretical guarantee in Section~\ref{sec:degree-selection} for unknown $k_0$.

\subsection{Nearly Optimal Risk Bound for Regression Risk with Known $k_0$}
\label{sec:sharp-risk-bound}
We first suppose that the ground truth degree of the target function $f^*$, $k_0$, is known.
As widely studied in the literature such as~\citet{DamianLS22-nn-representation-learning,Ghorbani2021-linearized-two-layer-nn}, because the dimension of the subspace in which the target function $f^*$ lies, the subspace spanned by all spherical harmonics of degree up to $k_0$, is $\Theta(d^{k_0})$, so that at least $\Theta(d^{k_0})/\eps$ training samples are required for any regression risk of $\eps > 0$. As a result, we let $n \ge \Theta(d^{k_0})$ throughout this paper.
We present our main result about the sharp risk bound in Theorem~\ref{theorem:LRC-population-NN-fixed-point}, with its proof deferred to Section~\ref{sec:proofs-main-results}.
\begin{theorem}\label{theorem:LRC-population-NN-fixed-point}
Suppose that $n \ge \Theta(\log({4}/{\delta})\cdot d^{2k_0})$, $\delta \in (0,1)$, and $c_t \in (0,1]$ is an arbitrary positive constant.
Suppose the network width $m$ satisfies
\bal\label{eq:m-cond-LRC-population-NN-fixed-point}
m \gsim \pth{\frac{n}{d^{k_0}}}^{\frac {25}{2}} d^{\frac 52},
\eal
and the neural network
$f(\cW(t),\cdot)$ is
trained by GDP using Algorithm~\ref{alg:GDP} with $r = r_0$, the constant learning rate
$\eta = \Theta(1) \in (0,1)$, and $T \asymp n/d^{k_0}$.
Then for every $t \in [c_t T \colon T]$, with probability at least
$1-\delta-\exp\pth{-\Theta(n)}-2\exp\pth{-\Theta(r_0)} - 2/n$ over the random noise $\bw$, the random training features $\bS$ and
the random initialization $\bW(0)$, $f(\cW(t),\cdot) = f_t$ satisfies
\bal\label{eq:LRC-population-NN-fixed-point}
&\Expect{P}{(f_t-f^*)^2}
\lsim \log{\frac{4}{\delta}} \cdot \Theta\pth{\frac{d^{k_0}}{n}}.
\eal
Here $r_0 = m_{k_0} = \Theta(d^{k_0})$.
\end{theorem}

Theorem~\ref{theorem:LRC-population-NN-fixed-point} establishes that the two-layer neural network~(\ref{eq:two-layer-nn}) trained by GDP described in Algorithm~\ref{alg:GDP} achieves a regression risk bound of order $\Theta(\log(4/\delta)\cdot d^{k_0}/n)$ when learning a degree-$k_0$ spherical polynomial, which is nearly minimax optimal up to a logarithmic factor compared to the lower bound $\Theta(r_0/n)$ shown in~\citet{Zhang2015-low-rank-kernel} and reviewed in Section~\ref{sec:setup-kernels-target-function}. Moreover, from~(\ref{eq:LRC-population-NN-fixed-point}) it follows that this rate implies a sample complexity of $n \asymp \Theta(\log(4/\delta)\cdot d^{k_0}/\eps)$ for achieving regression risk $\eps\in(0,\Theta(d^{-k_0})]$, which is substantially smaller than the $\Theta(d^{k_0}\max\{\eps^{-2},\log d\})$ sample complexity required by prior representative work such as~\citet{Nichani0L22-escape-ntk}. A detailed comparison with existing results on learning low-degree spherical polynomials, emphasizing algorithmic guarantees and the sharpness of the risk bounds, is provided in Table~\ref{table:main-results-comparison}.

\begin{table*}[t!]
        \centering
        \caption{Comparison between our result and the existing works on learning low-degree polynomials on the spheres of $\RR^d$
by training over-parameterized neural networks with or without algorithmic guarantees. Almost all the results here are under a common and popular setup that $f^* \in \cH_{\tK}$ where $\tK$ is the NTK of a specific neural network studied in each work,
and the responses $\set{y_i}_{i=1}^n$ are corrupted by i.i.d. Gaussian or sub-Gaussian noise with zero mean, with \citet{Nichani0L22-escape-ntk} being the only exception where the responses are noise-free. It is remarked that the sample complexity can be straightforwardly obtained from the regression risk. The regression risk of \citet[Theorem 1]{DamianLS22-nn-representation-learning} is for the risk less than $1/\sqrt{\log d}$, with the meaning of $r$ explained in Section~\ref{sec:sharp-risk-bound}, and $\tilde \Theta$ hides a logarithmic factor of $\log (mnd)$.
}
        \resizebox{1\linewidth}{!}{
        \begin{tabular}{|l|c|c|c|c|}
                \hline
                \textbf{Existing Works and Our Result}
                & \textbf{Finite-Width NN is Trained} & \textbf{Sharpness of the Regression Risk}
\\ \hline
\citep[Theorem 4]{Ghorbani2021-linearized-two-layer-nn}
 & No  & \begin{tabular}{@{}c@{}}Only matching the lower bound for pointwise kernel learning, \\not minimax optimal \end{tabular}  \\ \hline
\begin{tabular}{@{}c@{}}
\citep[Theorem 7]{BaiL20-quadratic-NTK}
\end{tabular}
&Yes &Not minimax optimal  \\ \hline
\citep[Theorem 1]{Nichani0L22-escape-ntk}
&\begin{tabular}{@{}c@{}}Yes\end{tabular}
  &$\Theta(\sqrt{d^{k_0}/n})$, not minimax optimal   \\ \hline
\citep[Theorem 1]{DamianLS22-nn-representation-learning}
&\begin{tabular}{@{}c@{}}Yes\end{tabular}
  &\begin{tabular}{@{}c@{}} $L^1$-norm regression risk $\tilde \Theta(\sqrt{dr^{k_0}/n}+\sqrt{r^p/m})$,
  \\ not minimax optimal
\end{tabular}
\\ \hline
Our Result (Theorem~\ref{theorem:LRC-population-NN-fixed-point})
 &\cellcolor{blue!15}Yes &\cellcolor{blue!15}
{Nearly minimax optimal, $\log{\frac{6}{\delta}} \cdot \Theta\pth{\frac{d^{k_0}}{n}}$ }
\\ \hline
       \end{tabular}
}
\label{table:main-results-comparison}
\end{table*}

It is proved in \citet[Theorem 1]{Nichani0L22-escape-ntk} that the regression risk $\eps > 0$ can be achieved with the sample complexity $n \gsim d^{k_0} \max\set{\eps^{-2},\log d}$, suggesting a convergence rate of the order
$\Theta(\sqrt{d^{k_0}/n})$ when the regression risk is less than $1/\sqrt{\log d}$, which is much less sharp
than our risk bound. The two-stage feature learning method \citep{DamianLS22-nn-representation-learning} requires a restrictive assumption that the target function only depends on $r \ll d$ directions of the input, as a result, the vanilla GD can naturally ensure that the learned neural network function is mostly in a subspace of rank $r$ in the RKHS. Without such assumption, we have $r = d$, and the $L^1$-norm risk bound of \citet[Theorem 1]{DamianLS22-nn-representation-learning}
 is then at least $\tilde \Theta(\sqrt{d^{k_0+1}/n})$. On the other hand, as the $L^p$-norm is always non-decreasing in terms of $p$, using our $L^2$-norm risk bound
in Theorem~\ref{theorem:LRC-population-NN-fixed-point}, we have a sharper $L^1$-norm risk bound of $\Theta(\sqrt{d^{k_0}/n})$.

As discussed in Section~\ref{sec:setup-kernels-target-function}, because the target function $f^*$ as a degree-$k_0$ spherical polynomial lies in the union of the eigenspaces up to degree $k_0$, we need to learn the subspace $\cup_{\ell=0}^{k_0} \cH_{\ell}$ of dimension $r_0 = m_{k_0}$ instead of the entire RKHS $\cH_K(\gamma_0)$ for a sharp regression risk. However, it is difficult for the vanilla GD algorithm to learn such a subspace in $\cH_K(\gamma_0)$.
Such observation motivates the design of the novel GDP algorithm, which fits the target function with a neural network function in a subspace of dimension $r_0$ of the Hilbert space
$\cH_{\bS,r_0}$ defined in (\ref{eq:space-H-S-rank-r}) with $r = r_0$.
 The next section details the roadmap for the proof of our main result.





\subsection{Adaptive Degree Selection with Unknown Degree $k_0$}
\label{sec:degree-selection}
In this section, we consider the case that $k_0$ is unknown.
We propose an adaptive degree selection algorithm, described in Algorithm~\ref{alg:degree-selection}, which both identifies the ground truth degree $k_0$ and trains the two-layer NN (\ref{eq:two-layer-nn}) which achieves the nearly optimal rate as that in (\ref{eq:LRC-population-NN-fixed-point}) with high probability.

With a constant $k_0 \in \Theta(1)$, it is always feasible to set $L \ge k_0$ as a suitably large constant $L$ such that $m_{L} \le n$. Starting with the initial degree $L$, the $\ell$-th iteration of Algorithm~\ref{alg:degree-selection} runs Algorithm~\ref{alg:GDP} to train the two-layer NN (\ref{eq:two-layer-nn}) with the projection dimension $r = m_{\ell}$ and $T = T_{\ell} = n/d^{\ell}$ steps. Suppose $E_{\ell}$ is the training loss of the trained network at the $\ell$-th iteration of Algorithm~\ref{alg:degree-selection}. If $\ell$ is the first integer such that $E_{\ell-1}/{\mu_{\ell}} \ge \beta_0^2/4$ and
$E_{\ell}/{\mu_{\ell+1}} \le \beta_0^2/8$ which is returned by Algorithm~\ref{alg:degree-selection}, then according to Theorem~\ref{theorem:degree-selection},
with high probability, $\ell = k_0$. We note that Theorem~\ref{theorem:degree-selection} needs the minimum absolute value condition on the target function that $\min_{\ell \in [0\relcolon k_0],j \in [N(d,\ell)]]}
{\abth{a_{\ell j}}}/\sqrt{\mu_{\ell}} \ge \beta_0$ for some positive constant $\beta_0$. Due to the presence of noise in the response vector $\by$, similar minimum absolute value conditions on the target signal are in fact necessary and broadly used in standard compressive sensing literature
such as~\cite{Aeron2010-information-theoretic-bound-cs} for signal recovery.
\begin{algorithm}[!htbp]
        \renewcommand{\algorithmicrequire}{\textbf{Input:}}
\renewcommand\algorithmicensure {\textbf{Output:} }
\caption{Degree Selection}
\label{alg:degree-selection}
\begin{algorithmic}[1]
\State $\ell+1, \cW^{(\ell+1)} \leftarrow$ Degree-Selection ($L,\eps_0,\bW(0)$)
\State \textbf{\bf input: } $L,\eps_0,\bW(0)$
\For{$\ell=L,\ldots,0$}
\State Set $r = m_{\ell}$, $T = T_{\ell} = n/d^{\ell}$
\State Run Algorithm~\ref{alg:GDP} with Training-by-GDP ($r,T,\bW(0)$) to train the two-layer NN (\ref{eq:two-layer-nn}) with GDP
\State Store the training loss $E_{\ell} = \Expect{P_n}{(f_{T_{\ell}}-f^*)^2}$ and the trained network weights $\cW^{(\ell)}$
\If{$E_{\ell}/{\mu_{\ell+1}} \ge \beta_0^2/4$ and
$E_{\ell+1}/{\mu_{\ell+2}} \le \beta_0^2/8$}
\State {\bf return} $\ell+1, \cW^{(\ell+1)}$
\EndIf
\EndFor
\end{algorithmic}
\end{algorithm}
\begin{theorem}\label{theorem:degree-selection}
Assume that the minimum absolute value condition on the target function holds, that is, $\min_{\ell \in [0\relcolon k_0],j \in [N(d,\ell)]]}
{\abth{a_{\ell j}}}/\sqrt{\mu_{\ell}} \ge \beta_0$ holds for some positive constant $\beta_0$. $\eps_0$ is a positive threshold such that $\eps_0 \in (0,\beta_0^2]$.
Suppose that $\delta \in (0,1)$, $\ell \in [L]$,  the network width $m$ satisfies
$m \gsim \pth{\frac{n}{d^{\ell}}}^{\frac {25}{2}} d^{\frac 52}$,  $n \ge \Theta\pth{\log({8}/{\delta})d^{\ell}/(\gamma_0^2\eps_0^2 \mu^2_{\ell+1})}$,
and the neural network
$f(\cW(t),\cdot)$ is
trained by GDP using Algorithm~\ref{alg:GDP} with the projection dimension $r = m_{\ell}$ and $T = T_{\ell} = n/d^{\ell}$ steps.
Then with probability at least
$1-\delta-\exp\pth{-\Theta(n)}-2\exp\pth{-\Theta(m_{\ell})} - 2/n$ over the random noise $\bw$, the random training features $\bS$ and
the random initialization $\bW(0)$, the training loss of the network $f(\cW(T_{\ell}),\cdot) = f_{T_{\ell}}$, $\Expect{P_n}{(f_{T_{\ell}}-f^*)^2}$, satisfies
\bal
\begin{cases}
\label{eq:degree-selection}
\Expect{P_n}{(f_{T_{\ell}}-f^*)^2}/{\mu_{\ell+1}}  \le \beta_0^2 /8, & k_0 \le \ell \le L,\\
\Expect{P_n}{(f_{T_{\ell}}-f^*)^2}/{\mu_{\ell+1}} \ge \beta_0^2/4, &0 \le \ell < k_0.
\end{cases}
\eal
\end{theorem}

As a direct consequence of Theorem~\ref{theorem:degree-selection}, with a high probability, Algorithm~\ref{alg:degree-selection} terminates at the $k_0-1$-th iteration. In particular,
if $n \ge \Theta\pth{\log({8(L-k_0+2)}/{\delta})d^{\ell}/(\gamma_0^2\eps_0^2 \mu^2_{L})}$, then Algorithm~\ref{alg:degree-selection} always returns the ground truth degree
$k_0$ with probability at least $1-\delta-(L-k_0+2)\exp\pth{-\Theta(n)}-2\sum_{\ell=k_0-1}^{L}\exp\pth{-\Theta(m_{\ell})} - 2L/n$. We note that $m_{L} \asymp d^{L}$ with $d > \Theta(1)$ and $L = \Theta(1)$ by Lemma~\ref{lemma:r0-estimate} in the appendix, so Algorithm~\ref{alg:degree-selection}  returns the ground truth degree
$\ell = k_0$ with large probability. We also note that the two-layer NN with the network weights $\cW^{(\ell+1)}$ returned by  Algorithm~\ref{alg:degree-selection} achieves the nearly optimal rate  in (\ref{eq:LRC-population-NN-fixed-point}) with high probability according to  Theorem~\ref{theorem:LRC-population-NN-fixed-point}.

\section{Roadmap of Proofs}
\label{sec:proof-roadmap}
We first introduce the basic definitions in Section~\ref{sec:proofs-definitions}, then present the results about uniform convergence for the NTK (\ref{eq:kernel-two-layer})
in Section~\ref{sec:uniform-convergence-ntk-more}. The proofs of the main results, Theorem~\ref{theorem:LRC-population-NN-fixed-point} and
Theorem~\ref{theorem:degree-selection}, are presented
in Section~\ref{sec:proofs-main-results}.
We then present the roadmap of our theoretical results which lead to the first main result, Theorem~\ref{theorem:LRC-population-NN-fixed-point}, in this
section. The proof of the second main result, Theorem~\ref{theorem:degree-selection}, directly follows from Theorem~\ref{theorem:LRC-population-NN-fixed-point}.
We first 
detail the roadmap and key technical results in Section~\ref{sec:detailed-roadmap-key-results}, then present our novel proof strategy in Section~\ref{sec:novel-proof-strategy}. The proofs of
Theorem~\ref{theorem:LRC-population-NN-fixed-point}  and  Theorem~\ref{theorem:degree-selection} are presented in
Section~\ref{sec:proofs-main-results}. The proofs of the key results in
Section~\ref{sec:detailed-roadmap-key-results} are also deferred to the appendix.

\subsection{Basic Definitions}
\label{sec:proofs-definitions}
We introduce the following definitions for our analysis.
We define
\bal\label{eq:ut}
\bu(t) \defeq \hat \by(t) - \by
\eal
as the difference between the network output
$\hat \by(t)$ and the training response vector $\by$ right after the $t$-th step of GDP.
Let $\tau \le 1$ be a positive number. For $t \ge 0$ and $T \ge 1$ we define the following quantities:
$c_{\bu} \defeq \Theta(\gamma_0) + \sigma_0 + \tau + 1$,
\bal\label{eq:def-R}
&R \defeq \frac{\eta c_{\bu}  T}{\sqrt m},
\eal
\bal\label{eq:cV_S}
\cV_t \defeq \set{\bv \in \RR^n \colon \bv = -\pth{\bI_n- \eta \bK_n \bPr}^t f^*(\bS)},
\eal
\bal\label{eq:cE_S}
\cE_{t,\tau} \defeq &\set{\be \colon \be = \bbe_1 + \bbe_2 \in \RR^n, \bbe_1 = -\pth{\bI_n-\eta\bK_n\bPr}^t \bw,
\ltwonorm{\bbe_2} \le {\sqrt n} \tau }.
\eal
In particular, Lemma~\ref{lemma:empirical-loss-convergence} in the appendix shows that with high probability over the random noise $\bw$, the distance of every weighting vector $\bw_r(t)$ to its initialization $\bw_r(0)$  is bounded by $R$. In addition, $\bu(t)$ can be composed into two vectors,
$\bu(t) = \bv(t) + \be(t)$ such that $\bv(t) \in \cV_t$
and $\be(t) \in \cE_{t,\tau}$. We then define the set of the neural network weights during the training by GDP using Algorithm~\ref{alg:GDP} as follows:
\bal\label{eq:weights-nn-on-good-training}
&\cW(\bS,\bW(0),T) \defeq \left\{\bW \colon \exists t \in [T] {\textup{ s.t. }}\vect{\bW} = \vect{\bW(0)} - \sum_{t'=0}^{t-1} \frac{\eta}{n}  \bZ_{\bS}(t')\bPr \bu(t'), \right. \nonumber \\
& \left. \bu(t') \in \RR^{n}, \bu(t') = \bv(t') + \be(t'),
\bv(t') \in \cV_{t'}, \be(t') \in \cE_{t',\tau}, {\textup { for all } } t' \in [0,t-1] \vphantom{\frac12}  \right\}.
\eal%

We will also show by Lemma~\ref{lemma:empirical-loss-convergence} that with high probability over $\bw$, $\cW(\bS,\bW(0),T)$ is the set of the weights of the two-layer NN  (\ref{eq:two-layer-nn}) trained by GDP on the training features $\bS$ with the random initialization $\bW(0)$ and the number of steps of GDP not greater than $T$.
The set of the functions represented by the neural network with weights in $\cW(\bS,\bW(0),T)$ is then defined as
\bal\label{eq:random-function-class}
\cFnn(\bS,\bW(0),T) \defeq \set{f_t = f(\cW(t),\cdot) \colon \exists \, t \in [T], \bW(t) \in \cW(\bS,\bW(0),T)}.
\eal%
We also define the function class $\cF(B,w)$ for any $B,w > 0$ as
\bal
\cF(B,w,\bS,r_0) \defeq \set{f \colon f = h+e, h \in \cH_K(B) \cap \cH_{\bS,r_0}, \supnorm{e} \le w}. \label{eq:def-cF-ext-low-rank}
\eal
We will show by
Theorem~\ref{theorem:bounded-NN-class} in the next subsection
that with high probability over $\bw$,
$\cFnn(\bS,\bW(0),T)$ is a subset of $\cF(B,w,\bS,r_0)$, where a smaller
$w$ requires a larger network width $m$, and $B_h > \gamma_0$ is an absolute positive constant defined by
\bal
B_h &\defeq \gamma_0 + \Theta(1). \label{eq:B_h}
\eal

\subsection{Uniform Convergence to the NTK (\ref{eq:kernel-two-layer}) and More}
\label{sec:uniform-convergence-ntk-more}
We define the following functions with $\bW = \set{\bw_r}_{r=1}^m$:
\bal
h(\bw,\bu,\bv) &\defeq \indict{\bw^{\top} \bu \ge 0} \indict{\bw^{\top} \bv \ge 0}, \quad &\hat h(\bW,\bu,\bv) &\defeq \frac {1}{m} \sum\limits_{r=1}^m h(\bbw_r,\bu,\bv), \label{eq:h-hat-h}  \\
v_R(\bw,\bu) &\defeq \indict{\abth{\bw^{\top}\bu} \le R}, \quad &\hat v_R(\bW,\bu) &\defeq  \frac 1m \sum\limits_{r=1}^m v_R(\bbw_r,\bu), \label{eq:v-hat-v}
\eal%
where $\bu,\bv \in \RR^d$.
Then we have the following theorem stating the uniform convergence of $\hat h(\bW(0),\cdot,\cdot)$ to $K(\cdot,\cdot)$ and uniform convergence
of $\hat v_R(\bW(0),\cdot)$ to $\frac{2R}{\sqrt {2\pi} \kappa}$ for a
positive number $R\lsim \eta T /{\sqrt m}$, and $R$ is formally defined in (\ref{eq:def-R}).
It is remarked that while existing works such as \citet{Li2024-edr-general-domain} also have uniform convergence results for over-parameterized neural network,
our result does not depend on the H\"older continuity of the NTK.
\begin{theorem}\label{theorem:good-random-initialization}
The following results hold with $\eta \lsim 1$, $m \gsim \max\set{n^{2/d},\Theta(T^{\frac 53})}$, and $m/\log m \ge d$.
\begin{itemize}[leftmargin=.26in]
\item[(1)] With probability at least $1-1/n$ over the random initialization $\bW(0) = \set{\bbw_r(0)}_{r=1}^m$,
\bal\label{eq:good-initialization-sup-hat-h}
\sup_{\bu \in \cX,\bv \in \cX} \abth{ K^{(\alpha)}(\bu,\bv) -
\hat h(\bW(0),\bu,\bv) }
\le C_1(m/2,d,1/n) \lsim \sqrt{\frac{d \log m}{m}},
\alpha \in \set{0,1}.
\eal%
\item[(2)] With probability at least $1-1/n$ over the random initialization $\bW(0) = \set{\bbw_r(0)}_{r=1}^m$,
\bal\label{eq:good-initialization-sup-hat-V_R}
&\sup_{\bu \in \cX}\hat v_R(\bW(0),\bu)
\le \frac{2R}{\sqrt {2\pi} \kappa} + C_2(m/2,d,1/n) \lsim
\sqrt{d} m^{-\frac 15} T^{\frac 12},
\eal%
 \end{itemize}
 where $C_1(m/2,d,1/n),C_2(m/2,d,1/n)$ are two positive numbers depending on $(m,d,n)$, with their formal definitions deferred to
(\ref{eq:C1}) and (\ref{eq:C2}) in Section~\ref{sec:proofs-key-results}.
\end{theorem}
\begin{proof}
This theorem follows from Theorem~\ref{theorem:sup-hat-g}
and Theorem~\ref{theorem:V_R} in Section~\ref{sec:proofs-key-results}.
We note that
\bals
\hat h(\bW,\bu,\bv) = \frac {1}{m} \sum\limits_{r=1}^m h(\bbw_r,\bu,\bv) = \frac {1}{m/2} \sum\limits_{r'=1}^{m/2} h(\bbw_{2r'}(0),\bu,\bv),
\eals
then the first bound in part (1) for $K^{(0)}$ directly follows from Theorem~\ref{theorem:sup-hat-g}.  Moreover,
since $K^{(1)}(\bu,\bv) = \bu^{\top} \bv K^{(0)}(\bu,\bv)$, we have
\bals
\abth{ K^{(1)}(\bu,\bv) -
\bu^{\top} \bv \cdot
\hat h(\bW(0),\bu,\bv) } \le\sup_{\bu \in \cX,\bv \in \cX} \abth{ K^{(0)}(\bu,\bv) -
\hat h(\bW(0),\bu,\bv) },
\eals
which leads to the second bound in part (1) for $K^{(1)}$.
Part (2) directly follows from Theorem~\ref{theorem:V_R}.
\end{proof}
We define
\bal\label{eq:set-of-good-random-initialization}
\cW_0 \defeq \set{\bW(0) \colon (\ref{eq:good-initialization-sup-hat-h}) ,(\ref{eq:good-initialization-sup-hat-V_R}) \textup { hold}}
\eal%
as the set of all the good random initializations which satisfy (\ref{eq:good-initialization-sup-hat-h})  and (\ref{eq:good-initialization-sup-hat-V_R}) in Theorem~\ref{theorem:good-random-initialization}.
Theorem~\ref{theorem:good-random-initialization} shows that we have good random initialization with high probability, that is, $\Prob{\bW(0)
\in \cW_0} \ge 1-2/n$. When $\bW(0) \in \cW_0 $, the uniform convergence results, (\ref{eq:good-initialization-sup-hat-h})
and (\ref{eq:good-initialization-sup-hat-V_R}), hold with high probability, which are important for the analysis of the training dynamics of the two-layer NN
(\ref{eq:two-layer-nn}) by GD.

\subsection{Proofs of the Main Results,
Theorem~\ref{theorem:LRC-population-NN-fixed-point} and
Theorem~\ref{theorem:degree-selection}}
\label{sec:proofs-main-results}

\begin{proof}
{\textbf{\textup{\hspace{-3pt}of
Theorem~\ref{theorem:LRC-population-NN-fixed-point}}}.}
We use Theorem~\ref{theorem:LRC-population-NN-eigenvalue} and
Theorem~\ref{theorem:empirical-loss-bound} to prove this
theorem.

First of all, it follows by
Theorem~\ref{theorem:empirical-loss-bound} that with probability at least
$1-2\delta-\exp\pth{-\Theta(r_0)} $ over $\bS$ and $\bw$,
\bals
\Expect{P_n}{(f_t-f^*)^2} \le \Theta\pth{\frac {\gamma_0^2}{\eta t}} +\gamma_0^2 \log{\frac{2}{\delta}} \cdot \Theta\pth{\frac{ d^{k_0}   }{n}}.
\eals
Plugging such bound for $\Expect{P_n}{(f_t-f^*)^2}$ in
(\ref{eq:LRC-population-NN-bound-eigenvalue})
of Theorem~\ref{theorem:LRC-population-NN-eigenvalue}
leads to
\bal\label{eq:LRC-population-NN-fixed-point-seg1}
&\Expect{P}{(f_t-f^*)^2} \lsim
\Theta\pth{\frac {\gamma_0^2}{\eta t}}+
\log{\frac{2}{\delta}}\cdot \frac{d^{k_0}}{n} + w.
\eal
Due to the definition of $T \asymp n/d^{k_0}$, we have
\bal\label{eq:LRC-population-NN-fixed-point-seg2}
\frac{1}{\eta t} \asymp \frac{1}{\eta T}
\asymp \frac{d^{k_0}}{n}.
\eal
We also have $\Prob{\cW_0} \ge 1-2/n$. Let $w = {d^{k_0}}/{n}$, then $w \in (0,1)$ with $n > d^{k_0}$. (\ref{eq:LRC-population-NN-fixed-point}) then follows from
(\ref{eq:LRC-population-NN-fixed-point-seg1}) with $w={d^{k_0}}/{n}$,
(\ref{eq:LRC-population-NN-fixed-point-seg2}) and the union bound.
We note that $c_{\bu}$ is bounded by  a positive constant, so
that the condition on $m$ in (\ref{eq:m-cond-bounded-NN-class})
in Theorem~\ref{theorem:bounded-NN-class}, together with $w ={d^{k_0}}/{n}$ and
(\ref{eq:LRC-population-NN-fixed-point-seg2}) leads to
the condition on $m$ in (\ref{eq:m-cond-LRC-population-NN-fixed-point}).
\end{proof}

\begin{proof}
{\textbf{\textup{\hspace{-3pt}of Theorem~\ref{theorem:degree-selection}}}.}
We first decompose the target function $f^*$ by
\bal\label{eq:degree-selection-target-decomp}
f^* = f^*_{\ell} + \overline{f}^*_{\ell},
\quad
f^*_{\ell} = \sum\limits_{\ell=0}^{\min\set{k_0,\ell}} \sum\limits_{j=1}^{N(d,\ell)} a_{\ell j} Y_{\ell j}(\bx),
\overline{f}^*_{\ell} = f^* - f^*_{\ell}.
\eal
That is, $f^*_{\ell}$ is the projection of $f^*$ onto the subspace
spanned by all spherical harmonics of degree up to $\ell \ge 0$, and
$\overline{f}^*_{\ell}$ is the residue.

If $\ell \ge k_0$, then $f^*_{\ell} = 0$ and $ \overline{f}^*_{\ell} = f^*$. In this case, it follows from Theorem~\ref{theorem:LRC-population-NN-fixed-point} and by repeating its proof that
\bal\label{eq:degree-selection-seg1}
\Expect{P_n}{(f_{T_{\ell}}-f^*)^2} \le\log{\frac{4}{\delta}} \cdot \Theta\pth{\frac{d^{\ell}}{n}}
\eal
holds with probability at least $1-\delta-\exp\pth{-\Theta(n)}-2\exp\pth{-\Theta(m_{\ell})} - 2/n$ with $\delta \in (0,1)$. It follows from (\ref{eq:degree-selection-seg1}) that when $n \ge \Theta\pth{\log({4}/{\delta}) \max\set{d^{\ell}/(\eps_0 \mu_{\ell+1}),d^{2\ell}}  }$,
\bal\label{eq:degree-selection-seg1-post}
\Expect{P_n}{(f_{T_{\ell}}-f^*)^2}  \le\beta_0^2 \mu_{\ell+1}/8.
\eal

We now consider the case that $1 \le \ell < k_0$. In this case, it follows from Theorem~\ref{theorem:LRC-population-NN-fixed-point} again that we have $\Expect{P}{\pth{f_{T_{\ell}}- f^*_{\ell} }^2} \le \log{\frac{4}{\delta}} \cdot \Theta\pth{\frac{d^{\ell}}{n}}$
with probability at least $1-\delta-\exp\pth{-\Theta(n)}-2\exp\pth{-\Theta(m_{\ell})} - 2/n$ .
As a result, we have
\bal\label{eq:degree-selection-risk-decomp}
&\Expect{P}{(f_{T_{\ell}}-f^*)^2} = \Expect{P}{\pth{f_{T_{\ell}}- f^*_{\ell} - \overline{f}^*_{\ell}}^2}
\nonumber \\
&= \Expect{P}{\pth{\overline{f}^*_{\ell}}^2} - 2\Expect{P}{\pth{f_{T_{\ell}}- f^*_{\ell}} \overline{f}^*_{\ell}} + \Expect{P}{\pth{f_{T_{\ell}}- f^*_{\ell} }^2} \nonumber \\
&\ge \Expect{P}{\pth{\overline{f}^*_{\ell}}^2}
-2 \sqrt{\Expect{P}{\pth{f_{T_{\ell}}- f^*_{\ell} }^2}} \sqrt{\Expect{P}{\pth{\overline{f}^*_{\ell}}^2} } + \Expect{P}{\pth{f_{T_{\ell}}- f^*_{\ell} }^2} \nonumber \\
&\stackrel{\circled{1}}{\ge} a_{\ell+1,1}^2 - 2 \sqrt{\log{\frac{4}{\delta}} \cdot \Theta\pth{\frac{d^{\ell}}{n}}}
\cdot \gamma_0
\nonumber \\
&\stackrel{\circled{2}}{\ge}
\beta_0^2 \mu_{\ell+1}- 2 \sqrt{\log{\frac{4}{\delta}} \cdot \Theta\pth{\frac{d^{\ell}}{n}}}
\cdot \gamma_0.
\eal
Here $\circled{1}$ follows from
$\Expect{P}{\big({\overline{f}^*_{\ell}}\big)^2} \le \sum_{\ell=0}^{k_0} \sum_{j=1}^{N(d,\ell)} a^2_{\ell,j}/\mu_{\ell} = \gamma_0^2$, and $\circled{2}$ follows from the fact that $\Expect{P}{\pth{\overline{f}^*_{\ell}}^2} \ge a_{\ell+1,1}^2 \ge \beta_0^2 \mu_{\ell+1}$.
It follows from (\ref{eq:degree-selection-risk-decomp}) that with
\bals
n \ge \Theta\pth{\log({4}/{\delta})\max\set{d^{\ell}/(\gamma_0^2\eps_0^2 \mu^2_{\ell+1}),d^{2\ell}}  },
\eals
we have
\bal\label{eq:degree-selection-seg2}
\Expect{P}{(f_{T_{\ell}}-f^*)^2}  \ge \beta_0^2 \mu_{\ell+1}/2.
\eal
Furthermore, it follows from the standard Hoeffding's inequality that
with probability at least $1-\delta$,
\bal\label{eq:degree-selection-seg3}
\abth{\Expect{P_n}{(f_{T_{\ell}}-f^*)^2} - \Expect{P}{(f_{T_{\ell}}-f^*)^2}}
\le \pth{\gamma_0 + \Theta(1)} \sqrt{\frac{\log(2/\delta)}{n}}.
\eal
It then follows from (\ref{eq:degree-selection-seg2}) and
(\ref{eq:degree-selection-seg3}) that with  probability at least $1-\delta$,
when $n \ge \Theta\big\{\log({4}/{\delta}) \newline \max\big\{d^{\ell}/(\gamma_0^2\eps_0^2 \mu^2_{\ell+1}),d^{2\ell} \big\}  \}$,
\bal\label{eq:degree-selection-seg4}
\Expect{P_n}{(f_{T_{\ell}}-f^*)^2} \ge \beta_0^2 \mu_{\ell+1}/4.
\eal
(\ref{eq:degree-selection}) then follows from (\ref{eq:degree-selection-seg1-post}) and  (\ref{eq:degree-selection-seg4}) and the fact that $\mu_{\ell} \asymp d^{-\ell}$ according to Theorem~\ref{theorem:eigenvalue-NTK}.
\end{proof}

\subsection{Detailed Roadmap and Key Results}
\label{sec:detailed-roadmap-key-results}

The summary of the  approaches and key technical results in the proofs are presented as follows.
Our main result,
Theorem~\ref{theorem:LRC-population-NN-fixed-point}, is built upon the following three significant technical results of independent interest.

First, using the novel GDP algorithm and the uniform convergence to the NTK (\ref{eq:kernel-two-layer}) during the training process by GDP,
we can have a nice decomposition of the neural network function at any step of GDP into a function in a $r_0$-dimensional subspace of the RKHS associated with the NTK (\ref{eq:kernel-two-layer}), which is $\cH_{K}(B_h) \cap \cH_{\bS,r_0}$, and an error function with a small $L^{\infty}$-norm. Formally, Theorem~\ref{theorem:bounded-NN-class} states that with high probability over $\bw$,
$\cFnn(\bS,\bW(0),T) \subseteq \cF(B_h,w,\bS,r_0)$.

\begin{theorem}\label{theorem:bounded-NN-class}
Suppose $n \ge \Theta(\log({2}/{\delta})\cdot d^{2k_0})$, $\delta \in (0,1/2)$, $w  \in (0,1)$, the network width $m$ satisfies
\bal\label{eq:m-cond-bounded-NN-class}
m \gsim
\max\set{T^{\frac {15}{2}} d^{\frac 52}/{w^5}, T^{\frac {25}{2}} d^{\frac 52}},
\eal
and the neural network
$f_t = f(\cW(t),\cdot) $ is trained by GDP using Algorithm~\ref{alg:GDP} with the constant learning rate $\eta = \Theta(1) \in (0,1)$ and the random initialization $\bW(0) \in \cW_0$. Then for every $t \in [T]$ and every $\delta \in (0,1/2)$, with probability at least
$1-2\delta-\exp\pth{-\Theta(n)}-\exp\pth{-\Theta(r_0)}$ over the random training features $\bS$ the random noise $\bw$,
$f_t\in \cFnn(\bS,\bW(0),T)$, and $f_t$ has the following decomposition on $\cX$: $f_t = h_t + e_t$,
where $h_t \in \cH_{K}(B_h) \cap \cH_{\bS,r_0}$ with $B_h$ defined in (\ref{eq:B_h}),
$e_t  \in L^{\infty}$ with $\supnorm{e_t } \le w$.
\end{theorem}

In particular, with the uniform convergence by
Theorem~\ref{theorem:good-random-initialization} and the optimization results in Lemma~\ref{lemma:empirical-loss-convergence} and
Lemma~\ref{lemma:bounded-Linfty-vt-sum-et} in the appendix,
Theorem~\ref{theorem:bounded-NN-class} shows that with high probability, the neural network function $f(\cW(t),\cdot)$ right after the $t$-th step of GDP can be decomposed into
two functions by $f(\cW(t),\cdot) = f_t = h  + e$, where $h \in \cH_{K}(B_h) \cap \cH_{\bS,r_0}$ is a function in a subspace of finite dimension $r_0$ of the RKHS associated with $K$ with a bounded $\cH_{K}$-norm.
The error function $e$ has a small $L^{\infty}$-norm, that is, $\supnorm{e} \le w$ with $w$ being a small number controlled by the network width
$m$, and larger $m$ leads to smaller $w$.

Second, local Rademacher complexity is employed
to tightly bound the risk of nonparametric regression in
Theorem~\ref{theorem:LRC-population-NN-eigenvalue} below, which is based on
 the Rademacher complexity of a localized subset of the function class $\cF(B_h,w,\bS,r_0)$ in
Lemma~\ref{lemma:pupulartion-RC-low-rank-proj} deferred the appendix. We use Theorem~\ref{theorem:bounded-NN-class}, Lemma~\ref{lemma:pupulartion-RC-low-rank-proj}, and Lemma~\ref{lemma:LRC-population-NN} deferred to the appendix to prove Theorem~\ref{theorem:LRC-population-NN-eigenvalue}.

\begin{theorem}\label{theorem:LRC-population-NN-eigenvalue}
Suppose $n \ge \Theta(\log({2}/{\delta})\cdot d^{2k_0})$, $\delta \in (0,1/2)$, $w  \in (0,1)$,
$m$ satisfies (\ref{eq:m-cond-bounded-NN-class}),
and the neural network
$f_t = f(\cW(t),\cdot)$ is
trained by GDP using Algorithm~\ref{alg:GDP} with the constant learning rate
$\eta = \Theta(1) \in (0,1)$ on the random initialization $\bW(0) \in \cW_0$.
Then for every $t \in [T]$ and every $\delta \in (0,1/2)$, with probability at least
$1-2\delta-\exp\pth{-\Theta(n)}-2\exp\pth{-\Theta(r_0)} $
over the random noise $\bw$, the random training features $\bS$ and
the random initialization $\bW(0)$,
\bal\label{eq:LRC-population-NN-bound-eigenvalue}
&\Expect{P}{(f_t-f^*)^2} - 2 \Expect{P_n}{(f_t-f^*)^2}
\nonumber \\
&\lsim
\sqrt{\log{\frac{2}{\delta}} } \cdot \frac{d^{k_0}}{n} + w.
\eal

\end{theorem}
Third, we have the following sharp upper bound for the training loss $\Expect{P_n}{(f_t-f^*)^2}$.
\begin{theorem}\label{theorem:empirical-loss-bound}
 Suppose the neural network trained after the $t$-th step of GDP, $f_t = f(\cW(t),\cdot)$, satisfies $\bu(t) = f_t(\bS) - \by = \bv(t) + \be(t)$
with $\bv(t) \in \cV_t$, $\be(t) \in \cE_{t,\tau}$. Let $n \ge \Theta(\log({2}/{\delta})\cdot d^{2k_0})$ and $\delta \in (0,1/2)$. If
$\eta \in (0,1), \quad \tau \le \sqrt{\frac{ d^{k_0} }{n}}$,
then for every $t \in [T]$, with probability at least
$1-2\delta-\exp\pth{-\Theta(r_0)} $ over the random training features $\bS$ and the random noise $\bw$, we have
$\Expect{P_n}{(f_t-f^*)^2} \le \Theta\pth{\frac {\gamma_0^2}{\eta t}} +\gamma_0^2 \log{\frac{2}{\delta}} \cdot \Theta\pth{\frac{ d^{k_0}   }{n}}$.
\end{theorem}
We then obtain Theorem~\ref{theorem:LRC-population-NN-fixed-point} using the upper bound for the regression risk in (\ref{eq:LRC-population-NN-bound-eigenvalue}) of Theorem~\ref{theorem:LRC-population-NN-eigenvalue} where $w$ is set to ${d^{k_0}}/{n}$, with the empirical loss $\Expect{P_n}{(f_t-f^*)^2}$ bounded by
$\Theta( \log ({2}/{\delta}) \cdot {d^{k_0}}/{n})$ with high probability by Theorem~\ref{theorem:empirical-loss-bound}.

\subsection{Novel Proof Strategy}
\label{sec:novel-proof-strategy}

We remark that the proof strategy of our main result,
Theorem~\ref{theorem:LRC-population-NN-fixed-point}, summarized above is significantly different from the existing works in training over-parameterized neural networks for nonparametric regression with minimax optimal rates
\citep{HuWLC21-regularization-minimax-uniform-spherical,
SuhKH22-overparameterized-gd-minimax,Li2024-edr-general-domain} and existing works about learning low-degree polynomials
\citep{Ghorbani2021-linearized-two-layer-nn,BaiL20-quadratic-NTK,Nichani0L22-escape-ntk,DamianLS22-nn-representation-learning}.

First, GDP is carefully incorporated into the analysis of the uniform convergence results for NTK, leading to the crucial decomposition of the neural network function $f_t$ in Theorem~\ref{theorem:bounded-NN-class}. It is remarked that while existing works such as \citet{Li2024-edr-general-domain} also has uniform convergence results for over-parameterized neural network, our results about the uniform convergence (in Section~\ref{sec:uniform-convergence-ntk-more} of the appendix) do not depend on the H\"older continuity of the NTK.

Second, to the best of our knowledge, Theorem~\ref{theorem:LRC-population-NN-eigenvalue} is the first result about the sharp upper bound
of the order $\Theta(\log({2}/{\delta}) \cdot {d^{k_0}}/{n})$ (with $w={d^{k_0}}/{n}$) for the regression risk of the neural network function which has the decomposition in Theorem~\ref{theorem:bounded-NN-class}. We note that the RHS of this upper bound (\ref{eq:LRC-population-NN-bound-eigenvalue}) is nearly
$\Theta({d^{k_0}}/{n})$, which has the expected and the desired order since the target function is in a $r_0$-dimensional subspace of the RKHS $\cH_K(\gamma_0)$ with $r_0 = \Theta(d^{k_0})$.

Third, a novel method based on the operator theory in RKHS has been developed to derive the sharp upper bound for the training loss in Theorem~\ref{theorem:empirical-loss-bound}. As shown in Theorem~\ref{theorem:bounded-NN-class}, the network function $f_t$ at every step $t$ of GDP is approximately a function in the $r_0$-dimensional subspace, $\cH_{K}(B_h) \cap \cH_{\bS,r_0}$. We emphasize that while it is intuitive to only learn the $r_0$-dimensional subspace by projection, $\cH_{K}(B_h) \cap \cH_{\bS,r_0}$, since the target function lies in that subspace, it has been an open problem in the research community how to handle the incurred training loss by such projection. In particular, as pointed out by the existing work \citep{Nichani0L22-escape-ntk}, learning in such a subspace leads to better alignment with the target function $f^*$, however, such alignment incurs additional training loss because the network function $f_t$ only learns the information in such a subspace of dimension $r_0 < n$, and the information in the ground truth signal $f^*(\bS)$ not in the $r_0$-dimensional subspace is not learned by $f_t$. We manage to show that the information of $f^*(\bS)$ not in the $r_0$-dimensional subspace, which is $\Proj_{\bUminusr}(f^*(\bS))$ where $\Proj_{\bUminusr} = \Proj_{\Span({\bUr})^{\perp}}$, is sharply bounded in Lemma~\ref{lemma:residue-f-star-low-rank} of the appendix:
$\ltwonorm{\Proj_{\bUminusr}(f^*(\bS))}^2 \le n\gamma_0^2 \log{\frac{2}{\delta}} \cdot \Theta\pth{\frac{ d^{k_0} }{n}}$. The proof of  Lemma~\ref{lemma:residue-f-star-low-rank} relies on a novel result in operator theory developed in this work which is of independent interest in functional analysis. Let $\set{{\Phi}^{(k)}}_{k \ge 0}$ be an orthonormal basis of the RKHS $\cH_K$ as an extension of the orthonormal basis of the RKHS $\cH_{\bS} \subseteq \cH_K$, $\set{{\Phi}^{(k)}}_{k \in [0:n-1]}$. 
Using the bounded Hilbert-Schmidt norm of $P^{T_K}_{m_{k_0}}-P^{T_n}_{m_{k_0}}$, where the two operators are defined as $P^{T_K}_{m_{k_0}} h = \sum_{j = 0}^{m_{k_0}-1} \iprod{h}{v_j}_{\cH} v_j,  P^{T_n}_{m_{k_0}} h = \sum_{j = 0}^{m_{k_0}-1} \iprod{h}{\Phi^{(j)}}_{\cH} \Phi^{(j)}$ for all $h \in \cH_K$, we can prove the following theorem showing the bounded projection of $f^*$ onto the eigenfunctions $\set{{\Phi}^{(q)}}_{q \ge r_0}$:
\begin{theorem}
\label{theroem:low-dim-target-function-satisfy-assump-informal}
With probability at least $1-\delta$,
$\sum\limits_{q=r_0}^{\infty} \iprod{f^*}{{\Phi}^{(q)}}_{\cH_K}^2
\le \zeta_{n,\gamma_0,r_0,\delta} \defeq \frac{32\gamma_0^2 \log{\frac{2}{\delta}} }{\pth{\mu_{k_0} - \mu_{k_0+1}}^2 n}$.
\end{theorem}
Theorem~\ref{theroem:low-dim-target-function-satisfy-assump-informal} proves Lemma~\ref{lemma:residue-f-star-low-rank}, which in turn proves
Theorem~\ref{theorem:empirical-loss-bound}.

\subsection{Beyond the Regular NTK Limit}
\label{sec:beyond-NTK}
We remark that while an over-parameterized neural network is trained, our result goes beyond the regular NTK limit due to our new GDP algorithm. As shown in Theorem~\ref{theorem:bounded-NN-class}, the novel projection operator $\bPr$ in GDP ensures that the neural network function almost lies in a $r_0$-dimensional subspace of the RKHS $\cH_K(\gamma_0)$ with $r_0 = \Theta(d^{k_0})$. Although such projection loses all the information of the ground truth signal $f^*(\bS)$ not lying in such a subspace, Theorem~\ref{theorem:empirical-loss-bound} shows that such information loss due to the projection is small enough to ensure a sharp regression risk bound. In contrast, the regular NTK-based analysis with vanilla GD must account for all eigenspaces associated with the NTK, and therefore cannot achieve such a sharp rate.

\section{Conclusion}

We study nonparametric regression
by training an over-parameterized two-layer neural network
where the target function is in the RKHS
associated with the NTK of the  neural network and also a degree-$k_0$ spherical polynomial on the unit sphere in $\RR^d$.
We show that, if the neural network is trained by a novel Gradient Descent with Projection (GDP), a nearly minimax optimal rate of the order
$\log({4}/{\delta}) \cdot \Theta(d^{k_0}/{n})$ can be obtained. We further present a novel and provable adaptive degree selection algorithm which obtains the same nearly optimal rate when the ground truth degree $k_0$ is unknown.




\newpage
\appendix

The appendix of this paper is organized as follows.
We present the basic mathematical results employed in our proofs in Section~\ref{sec::math-tools}, and then introduce the detailed technical background about harmonic analysis on spheres in Section~\ref{sec:harmonic-analysis-detail}. Detailed proofs are presented in Section~\ref{sec:proofs}. In particular, more results about the eigenvalue decay rates are presented in Section~\ref{sec:NTK-eigenvalues}.

\section{Mathematical Tools}
\label{sec::math-tools}

\subsection{Concentration Inequalities for Supremum of Empirical Processes}
\label{sec:concentration-sup-emp-process}

The Rademacher complexity of a function class and its empirical version are defined below.
\begin{definition}\label{def:RC}
Let $\bsigma = \set{\sigma_i}_{i=1}^n$ be $n$ i.i.d. random variables such that $\Pr[\sigma_i = 1] = \Pr[\sigma_i = -1] = \frac{1}{2}$. The Rademacher complexity of a function class $\cF$ is defined as
\bal\label{eq:RC}
&\cfrakR(\cF) = \Expect{\set{\bbx_i}_{i=1}^n, \set{\sigma_i}_{i=1}^n}{\sup_{f \in \cF} {\frac{1}{n} \sum\limits_{i=1}^n {\sigma_i}{f(\bbx_i)}} }.
\eal%
The empirical Rademacher complexity is defined as
\bal\label{eq:empirical-RC}
&\hat \cfrakR(\cF) = \Expect{\set{\sigma_i}_{i=1}^n} { \sup_{f \in \cF} {\frac{1}{n} \sum\limits_{i=1}^n {\sigma_i}{f(\bbx_i)}} },
\eal%
For simplicity of notation, Rademacher complexity and empirical Rademacher complexity are also denoted by $\Expect{}{\sup_{f \in \cF} {\frac{1}{n} \sum\limits_{i=1}^n {\sigma_i}{f(\bbx_i)}} }$ and $\Expect{\bsigma}{\sup_{f \in \cF} {\frac{1}{n} \sum\limits_{i=1}^n {\sigma_i}{f(\bbx_i)}}}$, respectively. 
\end{definition}

For data $\set{\bbx}_{i=1}^n$ and a function class $\cF$, we define the notation $R_n \cF$ by $R_n \cF \coloneqq \sup_{f \in \cF} \frac{1}{n} \sum\limits_{i=1}^n \sigma_i f(\bbx_i)$.
We have the contraction property for Rademacher complexity, which is due to Ledoux
and Talagrand~\citep{Ledoux-Talagrand-Probability-Banach}.

\begin{theorem}\label{theorem:RC-contraction}
Let $\phi$ be a contraction,that is, $\abth{\phi(x) - \phi(y)} \le \mu \abth{x-y}$ for $\mu > 0$. Then, for every function class $\cF$,
\bal\label{eq:RC-contraction}
&\Expect{\set{\sigma_i}_{i=1}^n} {R_n \phi \circ \cF} \le \mu \Expect{\set{\sigma_i}_{i=1}^n} {R_n \cF},
\eal%
where $\phi \circ \cF$ is the function class defined by $\phi \circ \cF = \set{\phi \circ f \colon f \in \cF}$.
\end{theorem}

\begin{definition}[Sub-root function,{\citep[Definition 3.1]{bartlett2005}}]
\label{def:sub-root-function}
A function $\psi \colon [0,\infty) \to [0,\infty)$ is sub-root if it is nonnegative,
nondecreasing and if $\frac{\psi(r)}{\sqrt r}$ is nonincreasing for $r >0$.
\end{definition}

\begin{theorem}[{\citet[Theorem 3.3]{bartlett2005}}]
\label{theorem:LRC-population}
Let $\cF$ be a class of functions with ranges in $[a, b]$ and assume
that there are some functional $T \colon \cF \to \RR+$ and some constant $\bar B$ such that for every $f \in \cF$ , $\Var{f} \le T(f) \le \bar B P(f)$. Let $\psi$ be a sub-root function and let $r^*$ be the fixed point of $\psi$.
Assume that $\psi$ satisfies that, for any $r \ge r^*$,
$\psi(r) \ge \bar B \cfrakR(\set{f \in \cF \colon T (f) \le r})$. Fix $x > 0$,
then for any $K_0 > 1$, with probability at least $1-e^{-x}$,
\bals
\forall f \in \cF, \quad \Expect{P}{f} \le \frac{K_0}{K_0-1} \Expect{P_n}{f} + \frac{704K_0}{\bar B} r^*
+ \frac{x\pth{11(b-a)+26 \bar B K_0}}{n}.
\eals
Also, with probability at least $1-e^{-x}$,
\bals
\forall f \in \cF, \quad \Expect{P_n}{f} \le \frac{K_0+1}{K_0} \Expect{P}{f}  + \frac{704K_0}{\bar B} r^*
+ \frac{x\pth{11(b-a)+26 \bar B K_0}}{n}.
\eals
\end{theorem}

\section{Detailed Technical Background about Harmonic Analysis on Spheres}
\label{sec:harmonic-analysis-detail}
In this section, we provide background materials on spherical
harmonic analysis needed for our study of the RKHS. We refer the reader to \citet{chihara2011introduction,Efthimiou-spherical-harmonics-in-p-dim,
szego1975orthogonal} for further information
on these topics. As mentioned above, expansions in spherical harmonics were used in the past in the statistics literature, such as
\citet{Bach17-breaking-curse-dim,BiettiM19}.

With $\ell \ge 0$, let $\cP^{\textup{(hom)}}_{\ell}$ denote the space of all the degree-$\ell$ homogeneous polynomials on $\cX = \unitsphere{d-1}$, and let $\cH_{\ell}$ denote the space of degree-$\ell$ homogeneous harmonic polynomials on $\cX$, or the degree-$\ell$  spherical harmonics. That is,
\bal\label{eq:spherical-harmonic-degree-ell}
\cH_{\ell}  = \set{P \colon \cX \to \RR \colon P(\bx)
=\sum_{\abth{\alpha} = \ell} c_{\alpha} \bx^{\alpha},  \Delta P = 0},
\eal
where $\alpha = \bth{\alpha_1,\ldots,\alpha_d}$, $\bx^{\alpha} = \prod_{i=1}^d \bx_i^{\alpha_i}$, $\abth{\alpha} = \sum_{i=1}^d \alpha_i$, and $\Delta$ is the Laplacian operator.  For $\ell \neq \ell'$, the elements of $\cH_{\ell}$ and $\cH_{\ell'}$ are orthogonal to each other.
All the functions in the following text of this section are assumed to be elements of
$L^2(\cX,v_{d-1})$, where $v_{d-1}$ stands for the uniform distribution
on the sphere $\cX = \unitsphere{d-1}$. We have
$\iprod{f}{g}_{L^2} \coloneqq \int_{\cX}f(x)g(x)
 \diff v_{d-1}(x)$. We denote by $\set{Y_{kj}}_{j \in [N(d,k)]}$ the spherical harmonics of degree $k$ which form an orthogonal basis of $\cH_k$, where
 $N(d,k) = \frac{2k+d-2}{k} {k+d-3 \choose d-2}$ is the dimension of $\cH_k$. They form a orthonormal basis of $L^2(\cX,v_{d-1})$. We have
 $\sum_{j=1}^{N(d,k)} Y_{kj}(\bx) Y_{kj}(\bx') = N(d,k) P_k(\iprod{\bx}{\bx'}) $ for all $\bx,\bx' \in \cX$, where
 $P_k$ is the $k$-th Legendre polynomial in dimension $d$, which are also known as Gegenbauer polynomials,
given by the Rodrigues formula:
\bals
P_k(t) = (-\frac 12)^k \frac{\Gamma\pth{\frac{d-1}{2}}}{\Gamma\pth{k+\frac{d-1}{2}}} \pth{1-t^2}^{(3-d)/2} \pth{\frac{\diff }{\diff t}}^k \pth{1-t^2}^{k+(d-3)/2}.
\eals
The polynomials $\set{P_k}$ are orthogonal in $L^2(\cX,\diff v_{d-1})$ where the measure $\diff v_{d-1}$ is given by $\diff v_{d-1}(t) = (1-t^2)^{(d-3)/2} \diff t$, and we have
\bals
\int_{-1}^1 P_k^2(t) (1-t^2)^{(d-3)/2} \diff t = \frac{w_{d-1}}{w_{d-2}} \frac{1}{N(d,k)},
\eals
where $w_{d-1} \defeq \frac{2 \pi^{d/2}}{\Gamma(d/2)}$ denotes the surface of the unit sphere $\unitsphere{d-1}$. It follows from the
 orthogonality of spherical harmonics that
\bals
\int_{\cX} P_j(\iprod{\bx}{\bw}) P_j(\iprod{\bx'}{\bw}) \diff v_{d-1}(\bw) = \frac{\delta_{jk}}{N(d,k)} P_k(\iprod{\bx}{\bx'}),
\eals
where $\delta_{jk} = \indict{j=k}$. We have the following recurrence relation \citep[Equation 4.36]{Efthimiou-spherical-harmonics-in-p-dim},
\bals
t P_k(t) = \frac{k}{2k+d-2} P_{k-1}(t) + \frac{k+d-2}{2k+d-2} P_{k+1}(t)
\eals
for all $k \ge 1$, and $tP_0(t) = P_1(t)$.

The Funk-Hecke formula is helpful for computing Fourier coefficients in the basis of spherical
harmonics in terms of Legendre polynomials. For any $j \in [N(d,k)]$, we have
\bals
\int_{\cX} f(\iprod{\bx}{\bx'}) Y_{kj}(\bx') \diff v_{d-1}(\bx') = \frac{w_{d-2}}{w_{d-1}}  Y_{kj}(\bx) \int_{-1}^1 f(t) P_k(t)  (1-t^2)^{(d-3)/2} \diff t.
\eals
For a positive-definite kernel $\tK (\bx,\bx') = \kappa (\iprod{\bx}{\bx'})$ defined on $\cX$, we have its Mercer decomposition as follows.
\bals
\tK(\bx,\bx') = \sum\limits_{\ell \ge 0} \mu_{\ell} \sum\limits_{j=1}^{N(d,\ell)} Y_{\ell j}(\bx)Y_{\ell j}(\bx')
=  \sum\limits_{\ell \ge 0} \mu_{\ell} N(d,\ell) P_{\ell}(\iprod{\bx}{\bx'}),
\eals
where $\mu_{\ell}$ is the eigenvalue of the integral operator $T_{\tK}$ associated with $\tK$ corresponding to $\cH_{\ell}$.
It follows that
\bals
\mu_{\ell} = \frac{w_{d-2}}{w_{d-1}} \int_{-1}^1 \kappa(t) P_{\ell}(t) (1-t^2)^{(d-3)/2} \diff t.
\eals
The above equation will be used to compute the eigenvalues of the PSD kernels defined in (\ref{eq:kernel-two-layer}) in Section~\ref{sec:NTK-eigenvalues} of this appendix.

\begin{proposition}
[{\citet[Theorem 4.2]{Krylov-harmonic-polynomials}}]
\label{prop:docomp-homogeneous-poly}
Let $p \in \cP^{\textup{(hom)}}_{\ell}$. Then there exists unique $h_{n-2i} \in \cH_{n-2i}$ for $i \in \set{0,1,\ldots,\floor{n/2}}$ such that
\bals
p(\bx) = h_n + h_{n-2} + \ldots + h_{n-2k}.
\eals
\end{proposition}
\begin{theorem}
\label{theorem:spherical-polynomial-representation-spherical-harmonics}
Every polynomial $p$ defined on $\unitsphere{d-1}$ of degree $k$ for $k \ge 0$ can be represented as a linear combination of homogeneous harmonic polynomials up to degree $k$, that is,
\bals
p = \sum\limits_{i=0}^k c_i p_i,
\eals
where $p_i \in \cH_i$ for $i \in \set{0,1,\ldots,k}$.
\end{theorem}
\begin{proof}
Every polynomial $p$ defined on $\unitsphere{d-1}$ of degree $k$ can be represented as the sum of homogeneous polynomials on $\unitsphere{d-1}$ by grouping the terms of $p$ of the same degree together. It follows from Proposition~\ref{prop:docomp-homogeneous-poly} that every homogeneous polynomial is a linear combination of homogeneous harmonic polynomials up to degree $k$. As a result, the conclusion holds.
\end{proof}

\begin{lemma}
[Estimation for $r_0 = m_{k_0}$]
\label{lemma:r0-estimate}
For $k_0 = \Theta(1)$ and $d > \Theta(1)$, we have
\bal\label{eq:r0-estimate}
r_0 = \Theta(d^{k_0}).
\eal
\end{lemma}
\begin{proof}
It follows from the direct calculation that $N(d,\ell) \asymp d^{\ell}$ under the given conditions, so that $r_0 = \sum_{\ell=0}^{k_0} N(d,\ell) \asymp d^{k_0}$.
\end{proof}

\section{Detailed Proofs}
\label{sec:proofs}

Proofs for results in Section~\ref{sec:detailed-roadmap-key-results} are presented in Section~\ref{sec:proofs-key-results}, and the proofs of the lemmas required for the proofs in Section~\ref{sec:proofs-key-results} are presented in Section~\ref{sec:lemmas-main-results}.

\subsection{Proofs for Results in Section~\ref{sec:detailed-roadmap-key-results}}
\label{sec:proofs-key-results}

We present our key technical results regarding optimization and generalization of the two-layer NN (\ref{eq:two-layer-nn}) trained by GDP in this section.
The following theorem,

\begin{theorem}
\label{theroem:low-dim-target-function-satisfy-assump}
Suppose $K$ is a  continuous and positive definite kernel on $\cX \times \cX$, and the target function $f^* \in \cH_{K}(\gamma_0)$ is spanned by the orthogonal set $\set{v_j}_{j = 0}^{r_0-1}$ in the first $k_0$ eigenspaces of $T_K$ with $k_0 \ge 1$ and $r_0 = m_{k_0}$. That is,
\bal\label{eq:low-dim-target-function-satisfy-assump}
f^* = \sum\limits_{j =0}^{r_0-1} \beta_j v_j, \quad
\sum_{j=0}^{r_0-1} \beta_j^2 \le \gamma_0^2.
\eal
Then
with probability at least $1-\delta$ over the random training features $\bS$,
\bal\label{eq:target-func-low-rank-residue}
\sum\limits_{q=r_0}^{\infty} \iprod{f^*}{{\Phi}^{(q)}}_{\cH_K}^2
\le \frac{32\gamma_0^2 \log{\frac{2}{\delta}} }{\pth{\mu_{k_0} - \mu_{k_0+1}}^2 n} \defeq \zeta_{n,\gamma_0,r_0,\delta}.
\eal
Similarly, for every $f \in \cF(B_h,w,\bS,r_0)$, with probability at least $1-\delta$ over the random training features $\bS$,
\bal\label{eq:training-func-low-rank-residue}
\sum\limits_{q=r_0}^{\infty} \iprod{f}{v_q}_{\cH_K}^2
\le \zeta_{n,B_h,r_0,\delta}.
\eal

\end{theorem}

\newtheorem{innercustomthm}{{\bf{Theorem}}}
\newenvironment{customthm}[1]
  {\renewcommand\theinnercustomthm{#1}\innercustomthm}
  {\endinnercustomthm}


\subsubsection{Results about Uniform Convergence}

We have the following two theorems, Theorem~\ref{theorem:sup-hat-g} and Theorem~\ref{theorem:V_R}, regarding the uniform convergence to the PSD kernel $K^{(0)}$ defined in (\ref{eq:kernel-two-layer}) and the uniform convergence of $\hat v_R$ to $\frac{2R}{\sqrt {2\pi} \kappa}$ on the unit sphere $\cX$.

\begin{theorem}
[Adapted from {\citet[Theorem 6.1]{Yang2025-generalization-two-layer-regression}},{\citet[Theorem VI.7]{yang2024gradientdescentfindsoverparameterized}}]
\label{theorem:sup-hat-g}
Let ${\bW(0)} = \set{{{\bbw_r(0)}}}_{r=1}^m$, where each ${{\bbw_r(0)}} \sim \cN(\bzero,\kappa^2 \bI_d)$ for $r \in [m]$. Then for any $\delta \in (0,1)$, with probability at least $1-\delta$ over ${\bW(0)}$,
\bal\label{eq:sup-hat-h}
\sup_{\bu \in \cX,\bv \in \cX} \abth{ K^{(0)}(\bu,\bv) - \hat h({\bW(0)},\bu,\bv) } \le C_1(m,d,\delta),
\eal%
where
\bal\label{eq:C1}
&C_1(m,d,\delta) \defeq \frac{1}{\sqrt m} \pth{6(1+2B\sqrt{d}) + \sqrt{2\log{\frac {(1+2m)^{2d}} \delta}}} + \frac{7 {\log{\frac {(1+2m)^{2d}} \delta}} }{3m},
\eal%
and $B$ is an absolute positive constant. In addition, when $m \gsim n^{1/(2d)}$, $m/\log m \ge d$, and $\delta \asymp 1/n$,
$C_1(m,d,\delta) \lsim \sqrt{\frac{d \log m}{m}} + \frac{d \log m}{m} \lsim \sqrt{\frac{d \log m}{m}}$.
\end{theorem}

\begin{theorem}
[{\citep[Theorem 6.1]{Yang2025-generalization-two-layer-regression}},{\citep[Theorem VI.8]{yang2024gradientdescentfindsoverparameterized}}]
\label{theorem:V_R}
Let ${\bW(0)} = \set{{{\bbw_r(0)}}}_{r=1}^m$, where each ${{\bbw_r(0)}} \sim \cN(\bzero,\kappa^2 \bI_d)$ for $r \in [m]$. Suppose $\eta \lsim 1$, $m \gsim 1$. Then for any $\delta \in (0,1)$, with probability at least $1-\delta$ over ${\bW(0)}$,
\bal\label{eq:sup-hat-V_R}
&\sup_{\bx \in \cX}\abth{\hat v_R({\bW(0)},\bx)-\frac{2R}{\sqrt {2\pi} \kappa}} \le C_2(m,d,\delta),
\eal%
where
\bal\label{eq:C2}
C_2(m,d,\delta) \defeq 3 \sqrt{\frac{d}{\kappa}} m^{-\frac 15} T^{\frac 12}
+ \sqrt{\frac{2{\log{\frac {(1+2{\sqrt m})^{d}} \delta}}}{m}} + \frac{7 {\log{\frac {(1+2{\sqrt m})^{d}} \delta}} }{3m}.
\eal
In addition, when $m \gsim n^{2/d}$, $m/\log m \ge d$, and $\delta \asymp 1/n$,
$C_2(m,d,\delta) \lsim  \sqrt{d} m^{-\frac 15} T^{\frac 12}$.
\end{theorem}

\subsubsection{Proof of
Theorem~\ref{theorem:bounded-NN-class}}

We prove Theorem~\ref{theorem:bounded-NN-class} in this subsection. The proof requires the following theorem, Lemma~\ref{lemma:empirical-loss-convergence}, about our main result about the optimization of the network (\ref{eq:two-layer-nn}). Lemma~\ref{lemma:empirical-loss-convergence} states that with high probability over the random noise $\bw$, the weights of the network $\bW(t)$ obtained right after the $t$-th step of GD
using Algorithm~\ref{alg:GDP} belongs to
$\cW(\bS,\bW(0),T)$. Furthermore, every weighing vector $\bw_r$ has bounded distance to the initialization $\bw_r(0)$.
The proof of Lemma~\ref{lemma:empirical-loss-convergence} is
based on
Lemma~\ref{lemma:yt-y-bound},
Lemma~\ref{lemma:empirical-loss-convergence-contraction},
and Lemma~\ref{lemma:weight-vector-movement} deferred to
Section~\ref{sec:lemmas-main-results} of the appendix.

\begin{lemma}\label{lemma:empirical-loss-convergence}
Suppose $\delta \in (0,1/2)$,
\bal\label{eq:m-cond-empirical-loss-convergence}
m \gsim T^{\frac {15}{2}} d^{\frac 52}/{\tau^5},
\eal
the neural network $f(\cW(t),\cdot)$ trained by GDP
using Algorithm~\ref{alg:GDP} with the constant learning rate $\eta = \Theta(1) \in (0,1)$, the random initialization $\bW(0) \in \cW_0$. Then
for every $\delta \in (0,1/2)$
with probability at least
$1-2\delta-\exp\pth{-\Theta(n)}$ over the random training features $\bS$ and the random noise $\bw$,
$\bW(t) \in \cW(\bS,\bW(0),T)$ for every $t \in [T]$.
Moreover, for every $t \in [0,T]$, $\bu(t) = \bv(t) + \be(t)$ where
$\bu(t) = \hat \by(t) -  \by$, $\bv(t) \in \cV_t$,
$\be(t) \in \cE_{t,\tau}$,  $\ltwonorm{\bu(t)} \le c_{\bu} \sqrt n$, and
$\ltwonorm{\bbw_r(t) - \bbw_r(0)} \le R$.
\end{lemma}

\begin{proof}
{\textbf{\textup{\hspace{-3pt}of
Theorem~\ref{theorem:bounded-NN-class}}}.}
In this proof we abbreviate $f_t$ as $f$ and $\bW(t)$ as $\bW$.
It follows from Lemma~\ref{lemma:empirical-loss-convergence}
and its proof that conditioned on an event $\Omega$ with probability at
least $1-2\delta-\exp\pth{-\Theta(n)}$,
$f \in \cFnn(\bS,\bW(0),T)$ with
$\bW(0) \in \cW_0$. Moreover, $f = f(\cW,\cdot)$ with $\bW = \set{\bbw_r}_{r=1}^m \in \cW(\bS,\bW(0),T)$, and $\vect{\bW} = \vect{\bW_{\bS}} = \vect{\bW(0)} - \sum_{t'=0}^{t-1} \eta/n \cdot \bZ_{\bS}(t') \bu(t')$ for some $t \in [T]$, where $\bu(t') \in \RR^n, \bu(t') = \bv(t') + \be(t')$ with $\bv(t') \in \cV_{t'}$ and $\be(t') \in \cE_{t',\tau}$ for all $t' \in [0,t-1]$. It also follows from Lemma~\ref{lemma:empirical-loss-convergence}
that conditioned on $\Omega$, $\ltwonorm{\bbw_r(t) - \bbw_r(0)} \le R$ for all $t \in [T]$.

$\bbw_r$ is expressed as
\bal\label{eq:bounded-Linfty-function-class-wr}
\bbw_r = \bbw_{\bS,r}(t) &= \bbw_r(0) - \sum_{t'=0}^{t-1} \frac{\eta}{n} \bth{\bZ_{\bS}(t')}_{[(r-1)d+1:rd]} \bPr\bu(t'),
\eal%
where the notation $\bbw_{\bS,r}$ emphasizes that $\bbw_r$ depends on the training features $\bS$.
We define the event
\bals
&E_r(R) \defeq \set{ \abth{ \bbw_r(0)^\top \bx} \le R },
\quad
\bar E_r(R) \defeq \set{ \abth{ \bbw_r(0)^\top \bx} > R },
\quad r\in [m].
\eals%
We now approximate $ f(\cW,\bx)$ by $g(\bx) \defeq \frac{1}{\sqrt m} \sum_{r=1}^m a_r
\indict{\bbw_r(0)^{\top} \bx \ge 0} \bbw_r^\top \bx$.
 We have
\bal\label{eq:bounded-Linfty-function-class-seg1}
&\abth{f(\cW,\bx) - g(\bx)}=\frac{1}{\sqrt m} \abth{\sum\limits_{r=1}^m a_r \relu{\bbw_r^\top \bx} - \sum_{r=1}^m a_r \indict{\bbw_r(0)^{\top} \bx \ge 0} \bbw_r^\top \bx}   \nonumber \\
&\le \frac{1}{\sqrt m}  \sum_{r=1}^m  \abth{ a_r \pth{ \indict{E_r(R) } + \indict{\bar E_r(R) }}  \pth{ \relu{\bbw_r^\top \bx} - \indict{\bbw_r(0)^{\top} \bx \ge 0} \bbw_r^\top \bx } } \nonumber \\
&=\frac{1}{\sqrt m} \sum_{r=1}^m   \indict{E_r(R)} \abth{ \relu{\bbw_r^\top \bx} - \indict{\bbw_r(0)^{\top} \bx \ge 0} \bbw_r^\top \bx  }\nonumber \\
&=\frac{1}{\sqrt m} \sum_{r=1}^m   \indict{E_r(R)}\abth{ \relu{\bbw_r^\top \bx} - \relu{\bbw_r(0)^\top \bx}   - \indict{\bbw_r(0)^{\top} \bx \ge 0} (\bbw_r-\bbw_r(0))^\top \bx } \nonumber \\
&\le \frac{2R }{\sqrt m} \sum_{r=1}^m   \indict{E_r(R)},
\eal%
where first inequality follows from $\indict{\bar E_r(R)} \pth{ \relu{\bbw_r^\top \bx} - \indict{\bbw_r(0)^{\top} \bx \ge 0} \bbw_r^\top \bx } = 0$.
Plugging $R = \frac{\eta c_{\bu}  T }{\sqrt m}$ in (\ref{eq:bounded-Linfty-function-class-seg1}), since $\bW(0) \in \cW_0$, we have
\bal\label{eq:bounded-Linfty-function-class-seg2}
&\sup_{\bx \in \cX} \abth{f(\cW,\bx) - g(\bx)} \le  2\eta c_{\bu} T \cdot \frac 1m \sum_{r=1}^m
 \indict{E_r(R)}  \le
2 \eta c_{\bu} T \pth{\frac{2R}{\sqrt {2\pi} \kappa} + C_2(m/2,d,1/n)}.
\eal
Using (\ref{eq:bounded-Linfty-function-class-wr}), $g(\bx)$
is expressed as
\bal\label{eq:bounded-Linfty-function-class-seg3}
&g(\bx) \nonumber \\
&= \frac{1}{\sqrt m} \sum_{r=1}^m a_r \sigma(\bbw_r(0)^\top \bx) {-} \sum_{t'=0}^{t-1} \frac{1}{\sqrt m}
\sum\limits_{r=1}^m  \indict{\bbw_r(0)^{\top} \bx \ge 0}\pth{ \frac{\eta}{n} \bth{\bZ_{\bS}(t')}_{[(r-1)d+1:rd]}
 \bPr\bu(t')
 }^\top \bx  \nonumber \\
&\stackrel{\circled{1}}{=} -\sum_{t'=0}^{t-1}
\underbrace{\frac{\eta}{nm}
\sum\limits_{r=1}^m  \indict{\bbw_r(0)^{\top} \bx \ge 0}
\sum\limits_{j=1}^n\indict{\bbw_r(t')^{\top} \bbx_j \ge 0}
\bth{\bPr \bu(t')}_j\bbx_j^{\top} \bx}_{\defeq G_{t'}(\bx)},
\eal
where $\circled{1}$ follows from the fact that $\frac{1}{\sqrt m} \sum_{r=1}^m a_r  \sigma(\bbw_r(0)^\top \bx) = f(\cW(0),\bx) = 0$ due to the particular initialization of the two-layer NN
(\ref{eq:two-layer-nn}). For each $G_{t'}$ in the RHS of
 (\ref{eq:bounded-Linfty-function-class-seg3}), we have
\bal\label{eq:bounded-Linfty-function-class-Gt}
&G_{t'}(\bx) \stackrel{\circled{2}}{=} \frac{\eta}{nm}
\sum\limits_{r=1}^m  \indict{\bbw_r(0)^{\top} \bx \ge 0}
\sum\limits_{j=1}^n \pth{d_{t',r,j}+\indict{\bbw_r(0)^{\top} \bbx_j \ge 0}} \bth{\bPr \bu(t')}_j\bbx_j^{\top} \bx
\nonumber \\
&\stackrel{\circled{3}}{=}
\frac{\eta}{n}
\sum\limits_{j=1}^n K(\bx,\bbx_j)\bth{\bPr \bu(t')}_j +
\underbrace{\frac{\eta}{n} \sum\limits_{j=1}^n q_j \bth{\bPr \bu(t')}_j}_{\defeq E_1(\bx)} \nonumber \\
&\hspace{1cm}\phantom{=}+ \underbrace{\frac{\eta}{nm}
\sum\limits_{r=1}^m  \indict{\bbw_r(0)^{\top} \bx \ge 0}
 \sum\limits_{j=1}^nd_{t',r,j} \bth{\bPr \bu(t')}_j\bbx_j^{\top} \bx}_{\defeq E_2(\bx)}.
\eal
where $d_{t',r,j} \defeq \indict{\bbw_r(t')^{\top} \bbx_j \ge 0}
- \indict{\bbw_r(0)^{\top} \bbx_j \ge 0}$
in $\circled{2}$, and
$q_j \defeq \hat h(\bW(0),\bbx_j,\bx) - K(\bbx_j,\bx)$
for all $j \in [n]$ in $\circled{3}$.
We now analyze each term on the RHS of (\ref{eq:bounded-Linfty-function-class-Gt}).
Let $h(\cdot,t') \colon \cX \to \RR$ be defined by
$h(\bx,t') \defeq \frac{\eta}{n} \sum\limits_{j=1}^n K(\bx,\bbx_j) \bth{\bPr \bu(t')}_j$,
then $h(\cdot,t') \in \cH_{\bS,r_0}$ for each $t' \in [0,t-1]$. We further define
\bal\label{eq:bounded-Linfty-function-class-h}
h_t(\cdot) \defeq \sum_{t'=0}^{t-1} h(\cdot,t') \in \cH_K,
\eal
Since $\bW(0) \in \cW_0$, $q_j \le C_1(m/2,d,1/n)$ for all $j \in [n]$ with
$C_1(m/2,d,1/n)$ defined in (\ref{eq:C1}). Moreover,
$\bu(t') \le c_{\bu} \sqrt n$ with high probability,  so that we have
\bal\label{eq:bounded-Linfty-function-class-E1-bound}
\supnorm{ E_1} = \supnorm{\frac{\eta}{n} \sum\limits_{j=1}^n q_j \bu_j(t')} &\le \frac{\eta}{n} \ltwonorm{\bu(t')}  \sqrt{n} C_1(m/2,d,1/n)
\le \eta c_{\bu} C_1(m/2,d,1/n) .
\eal
We now bound the last term on the RHS of (\ref{eq:bounded-Linfty-function-class-Gt}).
Define $\bX' \in \RR^{d \times n}$ with its $j$-column being ${\bX'}^{[j]}
= \frac{1}{m} \sum_{r=1}^m \indict{\bbw_r(0)^{\top} \bx \ge 0}
d_{t',r,j}\bbx_j$ for all $j \in [n]$, then
$E_2(\bx)= \frac{\eta}{n}\pth{\bX' \bPr \bu(t')}^{\top}\bx$.

We need to derive the upper bound for $\ltwonorm{\bX'}$. Because $\ltwonorm{\bbw_r(t') - \bbw_r(0)} \le R$, it follows that  $\indict{\bbw_r(t')^{\top}
\bbx_j \ge 0} = \indict{\bbw_r(0)^{\top} \bbx_j \ge 0}$ when $\abth{\bbw_r(0)^{\top} \bbx_j} > R$ for all $j \in [n]$. Therefore,
\bals
\abth{d_{t',r,j'}} = \abth{ \indict{\bbw_r(t')^{\top} \bbx_j \ge 0}
- \indict{\bbw_r(0)^{\top} \bbx_j \ge 0} } \le
\indict{\abth{\bbw_r(0)^{\top} \bbx_j } \le R},
\eals
and it follows that
\bal\label{eq:bounded-Linfty-function-class-E2-1}
\frac{\abth{ \sum\limits_{r=1}^m \indict{\bbw_r(0)^{\top} \bbx_i \ge 0}
d_{t',r,j}  } }{m}  &\le \frac{ \sum\limits_{r=1}^m \abth{d_{t',r,j}}  }{m}
\le \frac{ \sum\limits_{r=1}^m \indict{\abth{\bbw_r(0)^{\top} \bbx_j } \le R}  }{m} = \hat v_R(\bW(0),\bbx_j) \nonumber \\
&\le \frac{2R}{\sqrt {2\pi} \kappa} + C_2(m/2,d,1/n),
\eal
where $\hat v_R$ is defined by (\ref{eq:v-hat-v}), and the last inequality
follows from Theorem~\ref{theorem:V_R}.

It follows from (\ref{eq:bounded-Linfty-function-class-E2-1}) that
$\ltwonorm{\bX'} \le {\sqrt n}  \pth{\frac{2R}{\sqrt {2\pi} \kappa}
+ C_2(m/2,d,1/n)}$,
and we have
\bal\label{eq:bounded-Linfty-function-class-E2-bound}
\supnorm{ E_2(\bx)} &\le \frac{\eta}{n}\ltwonorm{\bX'}
\ltwonorm{\bPr}\ltwonorm{\bu(t')}
\ltwonorm{\bx} \le \eta  c_{\bu} \pth{\frac{2R}{\sqrt {2\pi} \kappa}
+ C_2(m/2,d,1/n)}.
\eal
Combining (\ref{eq:bounded-Linfty-function-class-Gt}),
(\ref{eq:bounded-Linfty-function-class-E1-bound}), and
(\ref{eq:bounded-Linfty-function-class-E2-bound}), for
any $t' \in [0,t-1]$,
\bal\label{eq:bounded-Linfty-function-class-Gt-ht-bound}
\sup_{\bx \in \cX} \abth{G_{t'}(\bx)-h(\bx,t')}
&\le \supnorm{E_1} + \supnorm{E_2} \nonumber \\
&\le \eta c_{\bu} \pth{ C_1(m/2,d,1/n)
+\frac{2R}{\sqrt {2\pi} \kappa}
+ C_2(m/2,d,1/n) }.
\eal
Define $e_t(\bx) = f(\cW,\bx) - h_t(\bx)$ for $\bx \in \cX$. It then follows from (\ref{eq:bounded-Linfty-function-class-seg2}),
(\ref{eq:bounded-Linfty-function-class-seg3}), and
(\ref{eq:bounded-Linfty-function-class-Gt-ht-bound})
that
\bal\label{eq:bounded-Linfty-function-class-f-h-bound}
&\supnorm{e_t} \le \sup_{\bx \in \cX} \abth{f(\cW,\bx) - g(\bx)}
+\sup_{\bx \in \cX} \abth{g(\bx)-h_t(\bx)} \nonumber \\
&\le  \sup_{\bx \in \cX} \abth{f(\cW,\bx) - g(\bx)} + \sum\limits_{t'=0}^{t-1}
\sup_{\bx \in \cX} \abth{G_{t'}(\bx)-h(\bx,t')} \nonumber \\
& \stackrel{\circled{4}}{\le}
2 \eta c_{\bu} T \pth{\frac{2R}{\sqrt {2\pi} \kappa} + C_2(m/2,d,1/n)}
+ \eta c_{\bu} T \pth{ C_1(m/2,d,1/n)
+\frac{2R}{\sqrt {2\pi} \kappa}
+ C_2(m/2,d,1/n) }
\nonumber \\
&\le \eta c_{\bu} T \pth{
C_1(m/2,d,1/n)
+3\pth{\frac{2R}{\sqrt {2\pi} \kappa}
+ C_2(m/2,d,1/n)} }
\defeq \Delta_{m,n,\eta,T},
\eal
where $\circled{4}$ follows from
(\ref{eq:bounded-Linfty-function-class-seg2}) and
(\ref{eq:bounded-Linfty-function-class-Gt-ht-bound}).
We now give an estimate for $\Delta_{m,n,\eta,T}$.
Since $\bW(0) \in \cW_0$, it follows from
Theorem~\ref{theorem:good-random-initialization} that
\bals
\Delta_{m,n,\eta,T}
&\lsim {\sqrt {d}} m^{-\frac 15} T^{\frac 32}.
\eals
It follows that, for any $w \in (0,1)$, when
$m \gsim T^{\frac {15}{2}} d^{\frac 52}/{w^5}$,
we have $\Delta_{m,n,\eta,T} \le w$.

It follows from Lemma~\ref{lemma:bounded-Linfty-vt-sum-et} that
with probability at least $
1 - \delta -\exp\pth{-\Theta(r_0)}$ over
$\bS,\bw$
$\norm{h_t}{\cH_K} \le B_h$,
where $B_h$ is defined in (\ref{eq:B_h}),
and $\tau$ is required to satisfy $\tau \le 1/(\eta T)$.
Since $h(\bx,t') \in \cH_{\bS,r_0}$ for all $t' \in [0,t-1]$,
$h_t \in \cH_{\bS,r_0}$, so that $h_t \in \cH_K(B_h) \cap \cH_{\bS,r_0}$.
Lemma~\ref{lemma:empirical-loss-convergence} requires that
$m \gsim T^{\frac {15}{2}} d^{\frac 52}/{\tau^5}$. As a result,
we also have $m \gsim T^{\frac {25}{2}} d^{\frac 52}$.
\end{proof}

\subsubsection{Proof of Theorem~\ref{theorem:LRC-population-NN-eigenvalue}}
We prove Theorem~\ref{theorem:LRC-population-NN-eigenvalue} using
Theorem~\ref{theorem:bounded-NN-class}, Lemma~\ref{lemma:pupulartion-RC-low-rank-proj},  and Lemma~\ref{lemma:LRC-population-NN}.

\vspace{0.1in}
\begin{proof}
{\textbf{\textup{\hspace{-3pt}of
Theorem~\ref{theorem:LRC-population-NN-eigenvalue}}}.}
It follows from Lemma~\ref{lemma:empirical-loss-convergence}
and Theorem~\ref{theorem:bounded-NN-class} that for every $t \in [T]$, conditioned on an event $\Omega$ with probability at least
$1-3\delta-\exp\pth{-\Theta(n)}-\exp\pth{-\Theta(r_0)}$
over $\bS$ and $\bw$, we have $\bW(t) \in \cW(\bS,\bW(0),T)$,
and $f(\cW(t),\cdot) = f_t \in \cFnn(\bS,\bW(0),T)$.
Moreover, conditioned on the event $\Omega$, $f_t \in \cF(B_h,w,\bS,r_0)$, $f_t = h_t + e_t$ where $h_t \in \cH_{K}(B_h) \cap \cH_{\bS,r_0}$ and
$e_t  \in L^{\infty}$ with $\supnorm{e_t } \le w$.
We then derive the sharp upper bound for $\Expect{P}{(f_t-f^*)^2} $ by
applying Theorem~\ref{theorem:LRC-population} to the function class
$\cF = \set{F=\pth{f- f^*}^2 \colon f \in
\cF(B_h,w,\bS,r_0)   }$.

Since $B_0 \defeq {(B_h+\gamma_0) } + 1
\ge {(B_h+\gamma_0) } + w$, we have
$\supnorm{F} \le B^2_0$ with $F \in \cF$, so that
$\Expect{P}{F^2} \le B^2_0\Expect{P}{F}$.
Let $T(F) = B^2_0\Expect{P}{F}$ for $F \in \cF$. Then
$\Var{F} \le \Expect{P}{F^2} \le T(F) = B^2_0\Expect{P}{F}$.
We have
\bal\label{eq:LRC-population-NN-seg1}
& B^2_0 \cfrakR \pth{\set{F \in \cF \colon T(F) \le r}} =  B^2_0 \cfrakR
\pth{ \set{(f-f^*)^2 \colon  f \in \cF(B_h,w,\bS,r_0),\Expect{P}{(f-f^*)^2}
\le \frac r{B^2_0}}} \nonumber \\
&\stackrel{\circled{1}}{\le} 2B_0^3 \cfrakR \pth{\set{f-f^* \colon
  f\in \cF(B_h,w,\bS,r_0), \Expect{P}{(f-f^*)^2} \le \frac{r}{B_0^2}}} \nonumber \\
&\stackrel{\circled{2}}{\le}  2B_0^3 \sqrt{\log{\frac{2}{\delta}} } \cdot \Theta\pth{\frac{d^{k_0}}{n}} + 2B_0^2\sqrt{\frac{rr_0}{n}}  + 4B_0^3w.
\defeq \psi(r).
\eal
Here $\circled{1}$ is due to the contraction property of
Rademacher complexity in Theorem~\ref{theorem:RC-contraction}. $\circled{2}$ holds with probability at least $1-\delta$ over $\bS$, following from Lemma~\ref{lemma:pupulartion-RC-low-rank-proj}. $\psi$ is a sub-root function since it is nonnegative, nondecreasing and
$\psi(r)/{\sqrt r}$ is nonincreasing. Let $r^*$ be the fixed point of $\psi$, and $r$ be any nonnegative number such that $0 \le r \le r^*$. It follows from {\citet[Lemma 3.2]{bartlett2005}} that
$0 \le r \le \psi(r)$. Therefore, by the definition of $\psi$ in (\ref{eq:LRC-population-NN-seg1}),
we have
\bal\label{eq:LRC-population-NN-seg2}
r &\lsim
\frac {r_0}{n} +
 \sqrt{\log{\frac{2}{\delta}} } \cdot
\frac{d^{k_0}}{n}  + w
\lsim \sqrt{\log{\frac{2}{\delta}} } \cdot
\frac{d^{k_0}}{n} + w
\eal
since $r_0 = \Theta(d^{k_0})$ and $B_0 = \Theta(1)$.
It then follows from Theorem~\ref{theorem:LRC-population} that with probability at least
$1-\exp(-x)$
over the random training features $\bS$,
\bal\label{eq:LRC-population-NN-risk-E1-bound}
&\Expect{P}{(f_t-f^*)^2} - \frac{K_0}{K_0-1} \Expect{P_n}{(f_t-f^*)^2}-\frac{x\pth{11B_0^2+26B_0^2 K_0}}{n} \le \frac{704K_0}{B_0^2} \cdot B_0^4 r^*,
\eal
or
\bal\label{eq:LRC-population-NN-risk-E1-bound-simple}
&\Expect{P}{(f_t-f^*)^2} - 2 \Expect{P_n}{(f_t-f^*)^2} \lsim  r^* + \frac {x}n,
\eal
with $K_0 = 2$ in (\ref{eq:LRC-population-NN-risk-E1-bound}).
It follows from (\ref{eq:LRC-population-NN-seg2}) and (\ref{eq:LRC-population-NN-risk-E1-bound-simple}) that
\bals
&\Expect{P}{(f_t-f^*)^2} - 2 \Expect{P_n}{(f_t-f^*)^2} \lsim
\sqrt{\log{\frac{2}{\delta}} } \cdot
\frac{d^{k_0}}{n} + w + \frac {x}n,
\eals
which proves (\ref{eq:LRC-population-NN-bound-eigenvalue})
with $x = d^{k_0}$.
\end{proof}

\begin{proof}
{\textbf{\textup{\hspace{-3pt}of
Theorem~\ref{theorem:empirical-loss-bound}}}.}
We have
\bal\label{eq:empirical-loss-bound-seg1}
f_t(\bS) = f^*(\bS) + \bw + \bv(t) + \be(t),
\eal
where $\bv(t) \in \cV_{t}$, $\be(t) \in \cE_{t,\tau}$,
$\be(t) = \bbe_1(t) + \bbe_2(t)$ with $\bv(t) = -\pth{\bI_n-\eta\bK_n \bPr}^{t} f^*(\bS)$,
$\bbe_1(t) = -\pth{\bI_n-\eta\bK_n \bPr}^{t} \bw$
and $\ltwonorm{\bbe_2(t)} \le {\sqrt n} \tau$. Since $\hlambda_1 \in (0,1)$, we have $\eta \hlambda_1 \in (0,1)$ if $\eta \in (0,1)$.
We use the simplified
notation $\Proj_{\bUr} = \Proj_{\Span(\bUr)}$
and $\Proj_{\bUminusr} = \Proj_{\Span({\bUr})^{\perp}}$,
we then have
\bal\label{eq:bounded-by-empirical-loss-seg2}
f_t(\bS) - \by &= f_t(\bS) -
\Proj_{\bUr}\pth{f^*(\bS)+\bw} -
\Proj_{\bUminusr}\pth{f^*(\bS)+\bw}
\nonumber \\
&= \Proj_{\bUr}\pth{\bv(t)+\be(t)} +
\Proj_{\bUminusr}\pth{\bv(t)+\be(t)} \nonumber \\
&= \Proj_{\bUr}\pth{\bv(t)+\be(t)}
- \Proj_{\bUminusr}\pth{f^*(\bS))+\bw}
+ \Proj_{\bUminusr}(\bbe_2(t)).
\eal
It follows from (\ref{eq:bounded-by-empirical-loss-seg2})
that $f_t(\bS) - \Proj_{\bUr}\pth{f^*(\bS)+\bw} =\Proj_{\bUr}\pth{\bv(t)+\be(t)} + \Proj_{\bUminusr}(\bbe_2(t))$,
or equivalently,
\bal\label{eq:bounded-by-empirical-loss-seg3}
f_t(\bS) &= \Proj_{\bUr}(f_t(\bS)) + \Proj_{\Span(\bUr)^{\perp}}(f_t(\bS))
\eal
with
\bals
\Proj_{\bUr}(f_t(\bS)) &= \Proj_{\bUr}\pth{f^*(\bS) + \bv(t) + \bw + \be(t)}, \nonumber \\
\Proj_{\Span(\bUr)^{\perp}}(f_t(\bS)) &= \Proj_{\bUminusr}(\bbe_2(t)). \nonumber
\eals
It follows from (\ref{eq:bounded-by-empirical-loss-seg3}) that
\bsals
&\Expect{P_n}{(f_t-f^*)^2}
=\frac 1n \ltwonorm{f_t(\bS) - f^*(\bS)}^2
= \frac 1n  \ltwonorm{f_t(\bS) - \Proj_{\bUr}(f^*(\bS))
-\Proj_{\bUminusr}(f^*(\bS))}^2 \nonumber \\
&=\frac 1n \ltwonorm{\Proj_{\bUr}(\bv(t)+\bw+\be(t))}^2
+ \frac 1n \ltwonorm{\Proj_{\bUminusr}
(\bbe_2(t)-f^*(\bS))}^2 \nonumber \\
&\stackrel{\circled{1}}{\le} \frac 3n \sum\limits_{i=1}^{r_0} \pth{1-\eta \hlambda_i }^{2t}
\bth{{\bUr}^{\top} f^*(\bS)}_i^2 + \frac 3n \sum\limits_{i=1}^{r_0} \pth{1-
\pth{1-\eta \hlambda_i }^t}^2
\bth{{\bUr}^{\top} \bw}_i^2 \nonumber \\
&\phantom{=}+ \frac 2n\ltwonorm{\Proj_{\bUminusr}
(f^*(\bS))}^2 + 5\tau^2 \nonumber \\
&\stackrel{\circled{2}}{\le} \frac {3\gamma_0^2}{2e\eta t} + \frac 3n \ltwonorm{{\bUr}^{\top} \bw}^2 + \gamma_0^2 \log{\frac{2}{\delta}} \cdot \Theta\pth{\frac{ d^{k_0} }{n}}+  5\tau^2  \nonumber \\
 &\stackrel{\circled{3}}{\le}
 \Theta\pth{\frac {\gamma_0^2}{\eta t}} + \frac {3r_0(\sigma_0^2+1)}n
 + \gamma_0^2 \log{\frac{2}{\delta}} \cdot \Theta\pth{\frac{ d^{k_0} }{n}},
\esals
which completes the proof. Here $\circled{1}$ follows by Cauchy-Schwarz inequality and $\ltwonorm{\bbe_2(t)} \le {\sqrt n} \tau$, and $\circled{2}$ follows from Lemma~\ref{lemma:bounded-Ut-f-in-RKHS} since $f^* \in \cF^* \subseteq \cH_{K^{(r_0)}}(\gamma_0) \subseteq \cH_K(\gamma_0)$, Lemma~\ref{lemma:auxiliary-lemma-1},
and (\ref{eq:residue-f-star-low-rank}) in Lemma~\ref{lemma:residue-f-star-low-rank} which holds with probability at least $1-2\delta$.
It follows from the concentration inequality about quadratic forms of sub-Gaussian random variables in \citet{quadratic-tail-bound-Wright1973} that
\bals
\Prob{\ltwonorm{{\bUr}^{\top} \bw}^2 -
\Expect{}{\ltwonorm{{\bUr}^{\top} \bw}^2} > r_0}
\le \exp\pth{-\Theta(r_0)},
\eals
so that with probability at least $1-\exp\pth{-\Theta(r_0)}$,
\bal\label{eq:Ur-w-bound}
\ltwonorm{{\bUr}^{\top} \bw}^2 \le
\Expect{}{\ltwonorm{{\bUr}^{\top} \bw}^2} + r_0
\le \sigma_0^2 \tr{\bUr {\bUr}^{\top}} +r_0
= r_0(\sigma_0^2+1),
\eal
which leads to $\circled{3}$ with $\tau = \sqrt{d^{k_0}/n}$.


\end{proof}

\subsection{Proofs of the Lemmas Required for the Proofs in Section~\ref{sec:proofs-key-results}}
\label{sec:lemmas-main-results}

\begin{proof}
{\textbf{\textup{\hspace{-3pt}of
Lemma~\ref{lemma:empirical-loss-convergence}}}.}
First, when $m \gsim  T^{\frac {15}{2}} d^{\frac 52}/{\tau^5}$ with a proper constant, it can be verified that $\bE_{m,n,\eta,R} \le {\tau {\sqrt n}}/{T}$ where $\bE_{m,n,\eta,R}$ is defined by
(\ref{eq:empirical-loss-Et-bound-Em}) of
Lemma~\ref{lemma:empirical-loss-convergence-contraction}.
Also, Theorem~\ref{theorem:sup-hat-g}
and Theorem~\ref{theorem:V_R} hold when (\ref{eq:m-cond-empirical-loss-convergence})
holds.
We then use mathematical induction to prove this lemma. We will first prove that $\bu(t) = \bv(t) + \be(t)$ where $\bv(t) \in \cV_t$,
$\be(t) \in \cE_{t,\tau}$, and $\ltwonorm{\bu(t)} \le c_{\bu} \sqrt n$ for all $t \in [0,T]$.

When $t = 0$, we have
\bal\label{eq:empirical-loss-convergence-seg1}
\bu(0) = - \by &= \bv(0) + \be(0),
\eal
where $\bv(0) \defeq -f^*(\bS) = -\pth{\bI-\eta \bK_n \bPr}^0 f^*(\bS)$,
$\be(0) = -\bw = \bbe_1(0) + \bbe_2(0)$ with
$\bbe_1(0) = -\pth{\bI-\eta \bK_n \bPr} ^0 \bw$
and $\bbe_2(0) = \bzero$. Therefore,
$\bv(0) \in \cV_{0}$ and $\be(0) \in \cE_{0,\tau}$.
Also, it follows from the proof of Lemma~\ref{lemma:yt-y-bound}
that $\ltwonorm{\bu(0)} \le  c_{\bu}{\sqrt n}$ with probability at least
$1 -  \exp\pth{-\Theta(n)}$
over the random noise $\bw$.

Suppose that for all $t_1 \in[0,t]$ with $t \in [0,T-1]$, $\bu(t_1) = \bv(t_1) + \be(t_1)$ where $\bv(t_1) \in \cV_{t_1}$, and
$\be(t_1) = \bbe_1(t_1) + \bbe_2(t_1)$ with
$\bv(t_1) \in \cV_{t_1}$ and $\be(t_1) \in \cE_{t_1,\tau}$ for all $t_1 \in[0,t]$. Then it follows from Lemma~\ref{lemma:empirical-loss-convergence-contraction} that the recursion
$\bu(t'+1)  = \pth{\bI- \eta \bK_n \bPr }\bu(t')
 +\bE(t'+1)$ holds for all $t' \in [0,t]$.
 As a result, we have
\bal\label{eq:empirical-loss-convergence-seg5}
\bu(t+1)  &= \pth{\bI- \eta \bK_n \bPr }\bu(t) +\bE(t+1) \nonumber \\
 & = -\pth{\bI-\eta \bK_n \bPr}^{t+1} f^*(\bS)
 -\pth{\bI-\eta \bK_n }^{t+1} \bw +\sum_{t'=1}^{t+1}
 \pth{\bI-\eta \bK_n\bPr}^{t+1-t'} \bE(t')
\nonumber \\
&=\bv(t+1) + \be(t+1),
\eal
where $\bv(t+1)$ and $\be(t+1)$ are defined as
\bal\label{eq:empirical-loss-convergence-vt-et-def}
\bv(t+1) \defeq-\pth{\bI-\eta \bK_n\bPr}^{t+1} f^*(\bS)\in \cV_{t+1},
\eal
\bal\label{eq:empirical-loss-convergence-et-pre}
&\be(t+1) \defeq \underbrace{-\pth{\bI-\eta \bK_n \bPr }^{t+1} \bw}_{\bbe_1(t+1)}
+ \underbrace{ \sum_{t'=1}^{t+1}
 \pth{\bI-\eta \bK_n\bPr}^{t+1-t'} \bE(t') }_{\bbe_2(t+1)}.
\eal
We now prove the upper bound for $\bbe_2(t+1)$.
With $\eta \in (0, 2)$, we have $\ltwonorm{\bI - \eta \bK_n\bPr} \in (0,1)$.
It follows that
\bal\label{eq:empirical-loss-convergence-et-bound}
&\ltwonorm{\bbe_2(t+1)} \le \sum_{t'=1}^{t+1} \ltwonorm{\bI-\eta \bK_n\bPr}^{t+1-t'}\ltwonorm{\bE(t')}
\le  \tau {\sqrt n},
\eal
where the last inequality follows from the fact
that $\ltwonorm{\bE(t)} \le \bE_{m,n,\eta,R} \le {\tau {\sqrt n}}/{T}$ for all $t \in [T]$. It follows that $\be(t+1) \in \cE_{t+1,\tau}$.
Also, it follows from Lemma~\ref{lemma:yt-y-bound}
that
with probability at least
$1-2\delta-\exp\pth{-\Theta(n)}$
over $\bS$ and $\bw$,
\bals
\ltwonorm{\bu(t+1)} &\le \ltwonorm{\bv(t+1)} + \ltwonorm{\bbe_1(t+1)}
+\ltwonorm{\bbe_2(t+1)} \nonumber \\
&\le \pth{\frac{\gamma_0}{ \sqrt{2e\eta } } + \sigma_0+\tau+1} {\sqrt n} \le \ c_{\bu}{\sqrt n}.
\eals
The above inequality completes the induction step, which also completes the proof. It is noted that
$\ltwonorm{\bbw_r(t) - \bbw_r(0)} \le R$ holds for all $t \in [T]$ by Lemma~\ref{lemma:weight-vector-movement}.

\end{proof}

\begin{lemma}\label{lemma:yt-y-bound}
Let $0 \le t \le T$, $\bv = -\pth{\bI-\eta \bK_n \bPr}^{t} f^*(\bS)$,
 $\be = -\pth{\bI-\eta \bK_n \bPr}^{t} \bw$, and $\eta \in (0,1)$. Suppose $\delta \in (0,1/2)$, then with probability at least
$1-2\delta-\exp\pth{-\Theta(n)}$
over the random training features $\bS$ and the random noise $\bw$,
\bal\label{eq:yt-y-bound}
\ltwonorm{\bv} + \ltwonorm{\be} \le \pth{\Theta(\gamma_0)+\sigma_0+1} \cdot
{\sqrt n}.
\eal
\end{lemma}
\begin{proof}
When $t \in [T]$, we have
\bal\label{eq:yt-y-bound-seg1}
\ltwonorm{\bv}^2 &=\sum\limits_{i=1}^{n}
\pth{1-\eta \hlambda_i }^{2t}
\bth{{\bU}^{\top} f^*(\bS)}_i^2  \nonumber \\
&= \sum\limits_{i=1}^{r_0}
\pth{1-\eta \hlambda_i }^{2t}
\bth{{\bU}^{\top} f^*(\bS)}_i^2 +
\sum\limits_{i=r_0+1}^{n}
\pth{1-\eta \hlambda_i }^{2t}
\bth{{\bU}^{\top} f^*(\bS)}_i^2 \nonumber \\
&\le  \sum\limits_{i=1}^{n}
\pth{1-\eta \hlambda_i }^{2t}
\bth{{\bU}^{\top} f^*(\bS)}_i^2 +
\ltwonorm{\Proj_{\bUminusr}(f^*(\bS))}^2  \nonumber \\
&\stackrel{\circled{1}}{\le}
\sum\limits_{i=1}^{n}
\frac{1}{2e\eta \hlambda_i  t}
\bth{{\bU}^{\top} f^*(\bS)}_i^2
+ n\gamma_0^2 \log{\frac{2}{\delta}} \cdot \Theta\pth{\frac{ d^{k_0} }{n}}
\nonumber \\
&\stackrel{\circled{2}}{\le}
\frac{n\gamma_0^2}{ 2e\eta t }+n\gamma_0^2 \log{\frac{2}{\delta}} \cdot \Theta\pth{\frac{ d^{k_0} }{n}} \le \frac{\gamma_0^2}{2e\eta} \cdot n.
\eal

Here $\circled{1}$ follows from Lemma~\ref{lemma:auxiliary-lemma-1} and
(\ref{eq:residue-f-star-low-rank}) in Lemma~\ref{lemma:residue-f-star-low-rank} which holds with probability at least $1-2\delta$, $\circled{2}$ follows
by Lemma~\ref{lemma:bounded-Ut-f-in-RKHS}
since $f^* \in \cF^* \subseteq \cH_{K^{(r_0)}}(\gamma_0) \subseteq \cH_K(\gamma_0)$. Moreover, it follows from the concentration inequality about quadratic forms of sub-Gaussian random variables in \citet{quadratic-tail-bound-Wright1973} that
$\Pr\{\ltwonorm{\bw}^2 -
\Expect{}{\ltwonorm{\bw}^2} > n\}
\le \exp\pth{-\Theta(n)}$,
so that $\ltwonorm{\be} \le \ltwonorm{\bw} \le
\sqrt{\Expect{}{\ltwonorm{\bw}^2}}  + {\sqrt n}= \sqrt{n} (\sigma_0+1)$ with probability at least $1-\exp\pth{-\Theta(n)}$. As a result, (\ref{eq:yt-y-bound}) follows from this inequality and (\ref{eq:yt-y-bound-seg1}) for $t \ge 1$.
When $t = 0$, $\ltwonorm{\bv} \le \gamma_0 {\sqrt n}$, so that (\ref{eq:yt-y-bound}) still holds.

\end{proof}

\begin{lemma}
\label{lemma:empirical-loss-convergence-contraction}
Let $0<\eta<1$, $0 \le t \le T-1$ for $T \ge 1$, and suppose that $\ltwonorm{\hat \by(t') - \by} \le
 c_{\bu}{\sqrt{n}}  $ holds for all $0 \le t' \le t$ and
 the random initialization $\bW(0) \in \cW_0$. Then
\bal\label{eq:empirical-loss-convergence-contraction}
\hat \by(t+1) - \by  &= \pth{\bI- \eta \bK_n }\pth{\hat \by(t) - \by} +\bE(t+1),
\eal
where $\ltwonorm {\bE(t+1)} \le \bE_{m,n,\eta,R}$,
and $\bE_{m,n,\eta,R}$ is defined by
\bal\label{eq:empirical-loss-Et-bound-Em}
\bE_{m,n,\eta,R} \defeq \eta c_{\bu} {\sqrt n}
\pth{ 4 \pth{\frac{2R}{\sqrt {2\pi} \kappa}+
C_2(m/2,d,1/n)} + 2C_1(m/2,d,1/n)}
\lsim {\sqrt {dn}} m^{-\frac 15} T^{\frac 12}.
\eal
\end{lemma}

\begin{proof}

Because $\ltwonorm{\hat \by(t') - \by} \le {\sqrt{n}} c_{\bu}$ holds for all $t' \in [0,t]$, by Lemma~\ref{lemma:weight-vector-movement}, we have
\bal\label{eq:empirical-loss-convergence-pre1}
\norm{\bbw_r(t') - \bbw_r(0)}{2} & \le  R, \quad \forall \, 0 \le t' \le t+1.
\eal%
We define $\bH^{(0)} \defeq \bF(\bW(0),\bS) \bF(\bW(0),\bS)^{\top}/m \in \RR^{n\times n}$.
We also define two sets of indices
\bals
E_{i,R} \defeq \set{r \in [m] \colon \abth{\bw_{r}(0)^{\top}\bbx_i} > R }, \quad \bar E_{i,R} \defeq [m] \setminus E_{i,R},
\eals%
then we have
\bal\label{eq:empirical-loss-convergence-contraction-seg1}
&\hat \by_i(t+1) - \hat \by_i(t) = \frac{1}{\sqrt m} \sum_{r=1}^m a_r \pth{ \relu{\bbw_{\bS,r}^{\top}(t+1) \bbx_i}  -  \relu{\bbw_{\bS,r}^{\top}(t) \bbx_i} } \nonumber \\
&\hspace{2.9cm}+ \frac{1}{\sqrt m} \pth{\bbw_{m+1}(t+1) - \bbw_{m+1}(t)}^{\top} \bsigma(\bW(0),\bbx_i) \nonumber \\
&=\underbrace{ \frac{1}{\sqrt m} \sum\limits_{r \in E_{i,R}} a_r \pth{ \relu{\bbw_{\bS,r}^{\top}(t+1) \bbx_i}  -  \relu{\bbw_{\bS,r}^{\top}(t) \bbx_i} }}_{\defeq \bD^{(1)}_i} \nonumber \\
&\phantom{=}{+} \underbrace{ \frac{1}{\sqrt m} \sum\limits_{r \in \bar E_{i,R}} a_r \pth{ \relu{\bbw_{\bS,r}^{\top}(t+1) \bbx_i}  -  \relu{\bbw_{\bS,r}^{\top}(t) \bbx_i} }}_{\defeq \bE^{(1)}_i}
\nonumber \\
&\phantom{=}{-} \frac{\eta}{nm}
\bF(\bW(0),\bbx_i)^{\top} \bF(\bW(0),\bS)^{\top}
\bPr (\hat \by(t) -  \by)  \nonumber \\
&=\bD^{(1)}_i + \bE^{(1)}_i - \frac{\eta}{n} \bth{\bH^{(0)}}_i \bPr \pth{\hat \by(t) - \by},
\eal%
and $\bD^{(1)}, \bE^{(1)} \in \RR^n$ are  vectors with their $i$-th element being $\bD^{(1)}_i$ and $\bE^{(1)}_i$ defined on the RHS of
 (\ref{eq:empirical-loss-convergence-contraction-seg1}).
Now we derive the upper bound for $\bE^{(1)}_i$. For all $i \in [n]$ we have
\bal\label{eq:empirical-loss-convergence-contraction-seg2}
\abth{\bE^{(1)}_i} &=  \abth{\frac{1}{\sqrt m}\sum\limits_{r \in \bar E_{i,R}} a_r \pth{ \relu{ \bbw_{\bS,r}(t+1)^\top \bbx_i }  -  \relu{ \bbw_{\bS,r}(t)^\top \bbx_i } } }  \nonumber \\
&\le \frac{1}{\sqrt m}\sum\limits_{r \in \bar E_{i,R}} \abth{ \bbw_{\bS,r}(t+1)^\top \bbx_i  - \bbw_{\bS,r}(t)^\top \bbx_i } \le \frac{1}{\sqrt m}\sum\limits_{r \in \bar E_{i,R}} \ltwonorm{\bbw_{\bS,r}(t+1) - \bbw_{\bS,r}(t) }  \nonumber \\
&\stackrel{\circled{1}}{=} \frac{1}{\sqrt m} \sum\limits_{r \in \bar E_{i,R}} \ltwonorm{\frac{\eta}{n} \bth{\bZ_{\bS}(t)}_{[(r-1)d+1:rd]} \bPr \pth{\hat \by(t) - \by}
}
\stackrel{\circled{2}}{\le} \frac{ c_{\bu}}{\sqrt m}  \sum\limits_{r \in \bar E_{i,R}}  \frac{\eta }{\sqrt m}
= \eta c_{\bu} \cdot \frac{\abth{\bar E_{i,R}}}{m}.
\eal%
Here $\circled{1}, \circled{2}$ follow from (\ref{eq:weight-vector-movement-seg1-pre})
and (\ref{eq:weight-vector-movement-seg1}) in the proof of
Lemma~\ref{lemma:weight-vector-movement}.

Let $m$ be sufficiently large such that $R \le R_0$ for the absolute positive constant $R_0 < \kappa$ specified in Theorem~\ref{theorem:good-random-initialization}.
Since $\bW(0) \in \cW_0$, we have
\bal\label{eq:empirical-loss-convergence-contraction-seg3}
&\sup_{\bx \in \cX}\abth{\hat v_R(\bW(0),\bx)} \le \frac{2R}{\sqrt {2\pi} \kappa} + C_2(m/2,d,1/n),
\eal%
where $\hat v_R(\bW(0),\bx) =  \frac 1m \sum\limits_{r=1}^m \indict{\abth{\bbw_r(0)^{\top} \bx} \le R }$, so that $ \hat v_R(\bW(0),\bbx_i) = \abth{\bar E_{i,R}}/m$.
It follows from (\ref{eq:empirical-loss-convergence-contraction-seg2}) and
(\ref{eq:empirical-loss-convergence-contraction-seg3}) that
$\abth{\bE^{(1)}_i} \le \eta c_{\bu}  \pth{ \frac{2R}{\sqrt {2\pi} \kappa}+ C_2(m/2,d,1/n)}$, so that
$\ltwonorm{\bE^{(1)}}$ can be bounded by
\bal\label{eq:empirical-loss-convergence-contraction-E1-bound}
\ltwonorm{\bE^{(1)}} & \le
 \eta c_{\bu} {\sqrt n} \pth{ \frac{2R}{\sqrt {2\pi} \kappa}+ C_2(m/2,d,1/n)} .
\eal
$\bD^{(1)}_i$ on the RHS of  (\ref{eq:empirical-loss-convergence-contraction-seg1})
is expressed by
\bal\label{eq:empirical-loss-convergence-contraction-seg5}
&\bD^{(1)}_i = \frac{1}{\sqrt m} \sum\limits_{r \in E_{i,R}} a_r \pth{ \relu{\bbw_{\bS,r}^{\top}(t+1) \bbx_i}  -  \relu{\bbw_{\bS,r}^{\top}(t) \bbx_i} } \nonumber \\
&=  \frac{1}{\sqrt m} \sum\limits_{r \in E_{i,R}} a_r \indict{\bbw_{\bS,r}(t)^{\top} \bbx_i \ge 0} \pth{ \bbw_{\bS,r}(t+1)  -  \bbw_{\bS,r}(t) }^\top \bbx_i   \nonumber \\
&=  \frac{1}{\sqrt m} \sum\limits_{r=1}^m a_r \indict{\bbw_{\bS,r}(t)^{\top} \bbx_i \ge 0} \pth{ -\frac{\eta}{n} \bth{\bZ_{\bS}(t)}_{[(r-1)d+1:rd]}
\bPr \pth{\hat \by(t) - \by}
 }^\top \bbx_i \nonumber \\
&\phantom{=}{+}  \frac{1}{\sqrt m} \sum\limits_{r \in \bar E_{i,R} } a_r \indict{\bbw_{\bS,r}(t)^{\top} \bbx_i \ge 0} \pth{\frac{\eta}{n}
\bth{\bZ_{\bS}(t)}_{[(r-1)d+1:rd]}
\bPr\pth{\hat \by(t) - \by}
 }^\top \bbx_i \nonumber \\
&=\underbrace{-\frac{\eta}{n} \bth{\bH^{(1)}(t)}_i \bPr\pth{\hat \by(t) - \by}}_{\defeq \bD^{(2)}_i} \nonumber \\
&\phantom{=}+ \underbrace{\frac{1}{\sqrt m} \sum\limits_{r \in \bar E_{i,R} } a_r \indict{\bbw_{\bS,r}(t)^{\top} \bbx_i \ge 0} \pth{\frac{\eta}{n} \bth{\bZ_{\bS}(t)}_{[(r-1)d+1:rd]}
\bPr\pth{\hat \by(t) - \by}
 }^\top \bbx_i }_{\defeq \bE^{(2)}_i}   \nonumber \\
&=  \bD^{(2)}_i + \bE^{(2)}_i,
\eal%
where $\bH^{(1)}(t) \in \RR^{n \times n}$ is a matrix specified by
\bals
\bH^{(1)}_{pq}(t) = \frac{\bbx_p^\top \bbx_q}{m} \sum_{r=1}^{m} \indict{\bbw_{\bS,r}(t)^\top \bbx_p \ge 0} \indict{\bbw_{\bS,r}(t)^\top \bbx_q \ge 0}, \quad
\forall \,p \in [n], q \in [n].
\eals
Let $\bD^{(2)}, \bE^{(2)} \in \RR^n$ be a vector with their $i$-the element being $\bD^{(2)}_i$ and
$\bE^{(2)}_i$ defined on the RHS of
(\ref{eq:empirical-loss-convergence-contraction-seg5}). $\bE^{(2)}$ can be expressed by $\bE^{(2)} = \frac{\eta}{n} \tilde \bE^{(2)}   \bPr\pth{\hat \by(t) - \by}$ with $\tilde \bE^{(2)} \in \RR^{n \times n}$ and
\bals
\tilde \bE^{(2)}_{pq} = \frac{1}{m} \sum\limits_{r \in \bar E_{i,R}} \indict{\bbw_{\bS,r}(t)^{\top} \bbx_p \ge 0} \indict{\bbw_{\bS,r}(t)^{\top} \bbx_q \ge 0} \bbx_{q}^\top \bbx_p
\le \frac {1}m \sum\limits_{r \in \bar E_{i,R}} 1 = \cdot \frac{\abth{\bar E_{i,R}}}{m}
\eals
for all $p \in [n], q \in [n]$.  The spectral norm of $\tilde \bE^{(2)}$ is bounded by
\bal\label{eq:empirical-loss-convergence-contraction-seg6}
\ltwonorm{\tilde \bE^{(2)}} \le \fnorm{\tilde \bE^{(2)}} \le n \frac{\abth{\bar E_{i,R}}}{m}
\stackrel{\circled{1}}{\le} n \pth{ \frac{2R}{\sqrt {2\pi} \kappa}+
C_2(m/2,d,1/n)},
\eal%
where $\circled{1}$ follows from (\ref{eq:empirical-loss-convergence-contraction-seg3}).
It follows from (\ref{eq:empirical-loss-convergence-contraction-seg6}) that $\ltwonorm{\bE^{(2)}}$ can be bounded by
\bal\label{eq:empirical-loss-convergence-contraction-E2-bound}
\ltwonorm{\bE^{(2)}} &\le \frac{\eta}{n} \ltwonorm{\tilde\bE^{(2)}}
\ltwonorm{\bPr} \ltwonorm{\by(t)-\by} \le \eta c_{\bu} {\sqrt n} \pth{\frac{2R}{\sqrt {2\pi} \kappa}+
C_2(m/2,d,1/n)}.
\eal
$\bD^{(2)}_i$ on the RHS of (\ref{eq:empirical-loss-convergence-contraction-seg5}) is expressed by
\bal\label{eq:empirical-loss-convergence-contraction-seg7}
&\bD^{(2)} = -\frac{\eta}{n} \bH^{(1)}(t)\bPr\pth{\hat \by(t) - \by} \nonumber
\\ &=\underbrace{-\frac{\eta}{n} \bK^{(1)}\bPr  \pth{\hat \by(t) - \by}}_{\defeq \bD^{(3)}} +  \underbrace{\frac{\eta}{n}  \pth{\bK^{(1)}- \bH^{(1)}(0)}\bPr    \pth{\hat \by(t) - \by}}_{\defeq \bE^{(3)}} \nonumber \\
   &\phantom{=}+  \underbrace{\frac{\eta}{n} \pth{\bH^{(1)}(0) - \bH^{(1)}(t)}  \bPr\pth{\hat \by(t) - \by}}_{\defeq \bE^{(4)}}
\nonumber \\
&=\bD^{(3)} + \bE^{(3)} + \bE^{(4)}.
\eal
On the RHS of  (\ref{eq:empirical-loss-convergence-contraction-seg7}), $\bD^{(3)},\bE^{(3)},\bE^{(4)} \in \RR^n$ are vectors which are analyzed as follows. We have
\bal\label{eq:empirical-loss-convergence-contraction-seg8}
\ltwonorm{\bK^{(1)}- \bH^{(1)}(0)} &\le \norm{\bK^{(1)}- \bH^{(1)}(0)}{F}
\le n C_1(m/2,d,1/n),
\eal%
where the last inequality is due to $\bW(0) \in \cW_0$. 

In order to bound  $\bE^{(4)}$, we first estimate the upper bound for $\abth{ \bH^{(1)}_{ij}(t) -\bH^{(1)}_{ij}(0) }$ for all $i,j \in [n]$.
We note that
\bal\label{eq:empirical-loss-convergence-contraction-seg9-pre}
&\indict{\indict{\bbw_{\bS,r}(t)^\top \bbx_i \ge 0} \neq \indict{\bw_{r}(0)^\top \bbx_i \ge 0} } \le \indict{\abth{\bw_{r}(0)^{\top} \bbx_i} \le R} + \indict{\ltwonorm{\bw_{
\bS,r}(t) - \bbw_r(0)} > R}.
\eal%
It follows from (\ref{eq:empirical-loss-convergence-contraction-seg9-pre}) that
\bal\label{eq:empirical-loss-convergence-contraction-seg9}
&\abth{ \bH^{(1)}_{ij}(t) - \bH^{(1)}_{ij}(0) } \nonumber \\
&= \abth{ \frac{\bbx_i^\top \bbx_j}{m} \sum_{r=1}^{m} \pth{ \indict{\bbw_{\bS,r}(t)^\top \bbx_i \ge 0} \indict{\bbw_{\bS,r}(t)^\top  \bbx_j \ge 0} - \indict{\bw_{r}(0)^\top \bbx_i \ge 0} \indict{\bw_{r}(0)^\top \bbx_j \ge 0} }} \nonumber \\
&\le \frac{1}m \sum_{r=1}^{m} \pth{\indict{\indict{\bbw_{\bS,r}(t)^\top \bbx_i \ge 0} \neq \indict{\bbw_r(0)^\top \bbx_i \ge 0}}
+ \indict{\indict{\bbw_{\bS,r}(t)^\top \bbx_j \ge 0} \neq \indict{\bbw_r(0)^\top \bbx_j \ge 0}}} \nonumber \\
&\le \frac{1}m \sum_{r=1}^{m} \pth{ \indict{\abth{\bbw_r(0)^{\top} \bbx_i} \le R}  +\indict{\abth{\bbw_r(0)^{\top} \bbx_j} \le R}  +2  \indict{\ltwonorm{\bw_{
\bS,r}(t) - \bbw_r(0)} > R}   } \nonumber \\
&\le 2v_R(\bW(0),\bbx_i) \stackrel{\circled{1}}{\le}
 \pth{\frac{4R}{\sqrt {2\pi} \kappa}+ 2C_2(m/2,d,1/n)},
\eal%
where $\circled{1}$ follows from
(\ref{eq:empirical-loss-convergence-contraction-seg3}).

It follows from
(\ref{eq:empirical-loss-convergence-contraction-seg8}) and
 (\ref{eq:empirical-loss-convergence-contraction-seg9})
 that $\ltwonorm{\bE^{(3)}},\ltwonorm{\bE^{(4)}} $ are bounded by
\bal
\ltwonorm{\bE^{(3)}} &\le\frac{\eta}{n}
\ltwonorm{\bK^{(1)}- \bH^{(1)}(0)} \ltwonorm{\bPr}\ltwonorm{\hat \by(t) - \by} \le \eta c_{\bu} {\sqrt n} C_1(m/2,d,1/n),
\label{eq:empirical-loss-convergence-contraction-E3-bound} \\
\ltwonorm{\bE^{(4)}} &\le\frac{\eta}{n}
\ltwonorm{\bH^{(1)}(0) - \bH^{(1)}(t)} \ltwonorm{\bPr}\ltwonorm{\hat \by(t) - \by} \nonumber \\
&\le\eta c_{\bu} {\sqrt n} \pth{\frac{4R}{\sqrt {2\pi} \kappa}+ 2C_2(m/2,d,1/n)} . \label{eq:empirical-loss-convergence-contraction-E4-bound}
\eal

It follows from (\ref{eq:empirical-loss-convergence-contraction-seg5})
and (\ref{eq:empirical-loss-convergence-contraction-seg7}) that
\bal\label{eq:empirical-loss-convergence-contraction-seg12}
\bD^{(1)}_i &=  \bD^{(3)}_i +  \bE^{(2)}_i+\bE^{(3)}_i + \bE^{(4)}_i.
\eal
We also have
\bal\label{eq:empirical-loss-convergence-contraction-seg13}
- \frac{\eta}{n}{\bH^{(0)}} \bPr\pth{\hat \by(t) - \by}
= \underbrace{- \frac{\eta}{n}\pth{\bH^{(0)}-\bK^{(0)}} \bPr\pth{\hat \by(t) - \by}}_{\defeq \bE^{(5)}}
- \frac{\eta}{n} \bK^{(0)} \bPr\pth{\hat \by(t) - \by}.
\eal
Similar to (\ref{eq:empirical-loss-convergence-contraction-E3-bound}),
$\bE^{(5)}$ is bounded by
\bal\label{eq:empirical-loss-convergence-contraction-E5-bound}
\ltwonorm{\bE^{(5)}} \le \frac{\eta}{n}
\ltwonorm{\bH^{(0)}-\bK^{(0)}} \ltwonorm{\bPr}\ltwonorm{\hat \by(t) - \by}
\le \eta c_{\bu} {\sqrt n} C_1(m/2,d,1/n).
\eal
It then follows from (\ref{eq:empirical-loss-convergence-contraction-seg1})
and (\ref{eq:empirical-loss-convergence-contraction-seg13}) that
\bal\label{eq:empirical-loss-convergence-contraction-seg13}
&\hat \by_i(t+1) - \hat \by_i(t) = \bD^{(1)}_i + \bE^{(1)}_i
- \frac{\eta}{n} \bth{\bH^{(0)}}_i\pth{\hat \by(t) - \by} \nonumber \\
&=\bD^{(3)}_i - \frac{\eta}{n} \bth{\bK^{(0)}}_i\bPr\pth{\hat \by(t) - \by}+ \underbrace{\bE^{(1)}_i + \bE^{(2)}_i+\bE^{(3)}_i
+ \bE^{(4)}_i + \bE^{(5)}_i}_{\defeq \bE_i}
\nonumber \\
&=-\frac{\eta}{n}\bK\pth{\hat \by(t) - \by}+ \bE_i,
\eal
where $\bE \in \RR^n$ with its $i$-th element being $\bE_i$, and $\bE = \bE^{(1)}
+\bE^{(2)}+\bE^{(3)} + \bE^{(4)} + \bE^{(5)}$. It then follows from
(\ref{eq:empirical-loss-convergence-contraction-E1-bound}),
(\ref{eq:empirical-loss-convergence-contraction-E2-bound}),
(\ref{eq:empirical-loss-convergence-contraction-E3-bound}),
(\ref{eq:empirical-loss-convergence-contraction-E4-bound}),
and (\ref{eq:empirical-loss-convergence-contraction-E5-bound})
that
\bal\label{eq:empirical-loss-convergence-contraction-E-bound}
&\ltwonorm {\bE} \le \eta c_{\bu} {\sqrt n}
\pth{ 4\pth{\frac{2R}{\sqrt {2\pi} \kappa}+
C_2(m/2,d,1/n)} + 2C_1(m/2,d,1/n)}.
\eal
Finally, (\ref{eq:empirical-loss-convergence-contraction-seg13})
can be rewritten as
\bals
\hat \by(t+1) - \by
&=\pth{\bI-\frac{\eta}{n} \bK}\bPr\pth{\hat \by(t) - \by} + \bE(t+1),
\eals
which proves (\ref{eq:empirical-loss-convergence-contraction})
with the upper bound for $\ltwonorm {\bE} $ in
(\ref{eq:empirical-loss-convergence-contraction-E-bound}).

\end{proof}

\begin{lemma}\label{lemma:weight-vector-movement}
Suppose that $t \in [0,T-1]$ for $T \ge 1$, and $\ltwonorm{\hat \by(t') - \by} \le {\sqrt n} c_{\bu} $ holds for all $0 \le t' \le t$. Then
\bal\label{eq:R}
\ltwonorm{\bbw_{\bS,r}(t') - \bbw_r(0)} \le R, \quad \forall\, 0 \le t' \le t+1.
\eal
\end{lemma}
\begin{proof}
Let $\bth{\bZ_{\bS}(t)}_{[(r-1)d+1:rd]}$ denote the submatrix of $\bZ_{\bS}(t)$ formed by the rows of $\bZ_{\bQ}(t)$ with row indices in $[(r-1)d+1:rd]$.
By the GD update rule we have for $t \in [0,T-1]$ that
\bal\label{eq:weight-vector-movement-seg1-pre}
&\bbw_{\bS,r}(t+1) - \bbw_{\bS,r}(t) = -\frac{\eta}{n} \bth{\bZ_{\bS}(t)}_{[(r-1)d+1:rd]} \bPr\pth{\hat \by(t) - \by},
\eal%

We have $\ltwonorm{\bth{\bZ_{\bS}(t)}_{[(r-1)d+1:rd]}} \le \sqrt{n /m}$.
It then follows from (\ref{eq:weight-vector-movement-seg1-pre}) that
\bal\label{eq:weight-vector-movement-seg1}
\ltwonorm{\bbw_{\bS,r}(t+1) - \bbw_{\bS,r}(t)}
&\le \frac{\eta}{n}
\ltwonorm{\bth{\bZ_{\bS}(t)}_{[(r-1)d+1:rd]}}\ltwonorm{\bPr}\ltwonorm{\hat \by(t)-\by}
\le  \frac{\eta c_{\bu}}{\sqrt m} .
\eal%
Note that (\ref{eq:R}) trivially holds for $t'=0$. For $t' \in [1,t+1]$, it follows from
 (\ref{eq:weight-vector-movement-seg1}) that
\bal\label{eq:weight-vector-movement-proof}
\ltwonorm{ \bbw_{\bS,r}(t') - \bbw_r(0) }
& \le \sum_{t''=0}^{t'-1} \ltwonorm{\bbw_{\bS,r}(t''+1) - \bbw_{\bS,r}(t'')}  \le \frac{\eta }{\sqrt m}  \sum_{t''=0}^{t'-1} c_{\bu}
  \le \frac{\eta c_{\bu}  T }{\sqrt m} =R,
\eal%
which completes the proof.
\end{proof}

\begin{lemma}\label{lemma:bounded-Linfty-vt-sum-et}
Suppose $n \ge \Theta(\log({2}/{\delta})\cdot d^{2k_0})$ and $\delta \in (0,1/2)$. Let
$h_t(\cdot) = \sum_{t'=0}^{t-1} h(\cdot,t')$ for $t \in [T]$, $T \le
\hat T$ where
\bals
h(\cdot,t') &= v(\cdot,t') + \hat e(\cdot,t'), \\
v(\cdot,t')  & -\frac{\eta}{n} \sum_{j=1}^n
K(\cdot, \bbx_j)  \bth{\bPr \bv(t')}_j  , \\
\hat e(\cdot,t') &= \frac{\eta}{n}
\sum\limits_{j=1}^n  K(\cdot, \bbx_j)\bth{\bPr \be(t')}_j,
\eals
where $\bv(t') \in \cV_{t'}$,
$\be(t') \in \cE_{t',\tau}$ for all $0 \le t' \le t-1$.
Suppose that $\tau \le 1/(\eta T)$,
then with probability at least $1 - \delta -\exp\pth{-\Theta(r_0)}$
over the random training features $\bS$ and the random noise $\bw$,
\bal\label{eq:bounded-h}
\norm{h_t}{\cH_K} \le B_h = \gamma_0 + \Theta(1),
\eal
where $r_0 = m_{k_0}$.
\end{lemma}
\begin{proof}
We have $\bv(t) = -\pth{\bI- \eta \bK_n \bPr }^t f^*(\bS)$,
$\be(t) = \bbe_1(t) + \bbe_2(t)$ with
$\bbe_1(t) = -\pth{\bI-\eta\bK_n \bPr}^t \bw$,
$\ltwonorm{\bbe_2(t)} \le {\sqrt n} \tau$.
We define
\bal\label{eq:bounded-Linfty-vt-sum-et-hat-e1-hat-e2}
\hat e_1(\cdot,t') \defeq- \frac{\eta}{n}
\sum\limits_{j=1}^n  K(\bbx_j,\bx) \bth{\bPr \bbe_1(t')}_j,
\quad
\hat e_2(\cdot,t') \defeq- \frac{\eta}{n}
\sum\limits_{j=1}^n  K(\bbx_j,\bx) \bth{\bPr\bbe_2(t')}_j,
\eal

Let $\bSigma$ be the diagonal matrix
containing eigenvalues of $\bK_n$, we then have
\bal\label{eq:bounded-Linfty-vt-sum-seg1}
\sum_{t'=0}^{t-1} v(\bx,t') &=\frac{\eta}{n} \sum\limits_{j=1}^n  \sum_{t'=0}^{t-1}
\bth{\bPr \pth{\bI- \eta \bK_n \bPr}^{t'} f^*(\bS)}_j K(\bbx_j,\bx) \nonumber \\
&=\frac{\eta}{n} \sum\limits_{j=1}^n \sum_{t'=0}^{t-1}
\bth{\bPr \bU \pth{\bI-\eta\bSigma^{(r_0)} }^{t'} {\bU}^{\top} f^*(\bS)}_j K(\bbx_j,\bx).
\eal
We then have
\bal\label{eq:bounded-Linfty-vt-sum-seg2}
&\norm{\sum_{t'=0}^{t-1} v(\cdot,t')}{\cH_K}^2 \nonumber \\
&= \frac{\eta^2}{n^2} f^*(\bS)^{\top}
\bU \sum_{t'=0}^{t-1} \pth{\bI-\eta \bSigma^{(r_0)}}^{t'} {\bU}^{\top}
\bPr \bK\bPr  \bU \sum_{t'=0}^{t-1} \pth{\bI-\eta \bSigma^{(r_0)}}^{t'}
{\bU}^{\top} f^*(\bS) \nonumber \\
&= \frac 1n \ltwonorm{\eta\pth{\bK_n}^{1/2} \bPr \bU \sum_{t'=0}^{t-1} \pth{\bI-\eta \bSigma^{(r_0)}}^{t'} {\bU}^{\top} f^*(\bS)}^2 \nonumber \\
&\le \frac 1n \sum\limits_{i=1}^{r_0} \frac{\pth{1-
\pth{1-\eta \hlambda_i }^t}^2}
{\hlambda_i}\bth{{\bU}^{\top} f^*(\bS)}_i^2
\le
\frac 1n \sum\limits_{i=1}^{n} \frac{\pth{1-
\pth{1-\eta \hlambda_i }^t}^2}
{\hlambda_i}\bth{{\bU}^{\top} f^*(\bS)}_i^2
\le  \gamma_0^2,
\eal
where the last inequality follows from
Lemma~\ref{lemma:bounded-Ut-f-in-RKHS} since $f^* \in \cF^* \subseteq \cH_{K^{(r_0)}}(\gamma_0) \subseteq \cH_K(\gamma_0)$.
We define
$E_1 \defeq \norm{\sum_{t'=0}^{t-1} \hat e_1(\cdot,t')}{\cH_K}^2$
and
$E_2 \defeq \norm{\sum_{t'=0}^{t-1} \hat e_2(\cdot,t')}{\cH_K}$.
It follows from  (\ref{eq:Ur-w-bound}) in the proof of
Theorem~\ref{theorem:empirical-loss-bound} that with probability at least $1-\exp\pth{-\Theta(r_0)}$,
$\ltwonorm{{\bUr}^{\top} \bw}^2
\lsim r_0 = \Theta(d^{k_0})$.

With $n \ge \Theta(\log({2}/{\delta})\cdot d^{2k_0})$
and $r \in [r_0]$, it follows from Lemma~\ref{lemma:gap-spectrum-Tk-Tn} that with probability $1-\delta$ over $\bS$,
 we have
\bal\label{eq:bounded-Linfty-vt-sum-hlambda-bound}
\hlambda_{r} \ge \hlambda_{r_0} \ge \lambda_{r-1}
-2\sqrt{\frac{2\log{\frac{2}{\delta}}}{n}}
\ge  \mu_{k_0} -2\sqrt{\frac{2\log{\frac{2}{\delta}}}{n}}
\ge \Theta(d^{-k_0}).
\eal
It then follows from (\ref{eq:bounded-Linfty-vt-sum-hlambda-bound}) that
\bal\label{eq:bounded-Linfty-hat-et-sum-E1}
&E_1
\le \frac 1n \sum\limits_{i=1}^{r_0} \frac{\pth{1-
\pth{1-\eta \hlambda_i }^t}^2}
{\hlambda_i}\bth{{\bU}^{\top} \bw}_i^2
\le\frac {\Theta(d^{k_0})}{n} \cdot \Theta(d^{k_0}) \le \Theta(1).
\eal
We now find the upper bound for $E_2$. We have
\bals
\norm{\hat e_2(\cdot,t')}{\cH_K}^2
&\le \frac{\eta^2}{n^2} \bbe_2^{\top}(t')\bK\bbe_2(t')
\le \eta^2 \hlambda_1 \tau^2,
\eals
so that
\bal\label{eq:bounded-Linfty-hat-et-sum-E2}
&E_2
\le \sum_{t'=0}^{t-1} \norm{\hat e_2(\cdot,t')}{\cH_K}
\le  T \eta \sqrt{\hlambda_1} \tau \le 1,
\eal
if $\tau \le 1/(\eta T) $ since $\hlambda_1 \in (0, 1)$.

Finally, it follows from (\ref{eq:bounded-Linfty-vt-sum-seg1}),
(\ref{eq:bounded-Linfty-hat-et-sum-E1}),
and (\ref{eq:bounded-Linfty-hat-et-sum-E2})
that
\bals
\norm{h_t}{\cH_K}  &\le \norm{\sum_{t'=0}^{t-1} \hat v(\cdot,t')}{\cH_K}
+\norm{\sum_{t'=0}^{t-1} \hat e_1(\cdot,t')}{\cH_K} + \norm{\sum_{t'=0}^{t-1} \hat e_2(\cdot,t')}{\cH_K} \le \gamma_0 + \Theta(1).
\eals
\end{proof}

\begin{lemma}[In the proof of
{\citep[Lemma 8]{RaskuttiWY14-early-stopping-kernel-regression}}]
\label{lemma:bounded-Ut-f-in-RKHS}
For any $f \in \cH_{K}(\gamma_0)$, we have
\bal\label{eq:b.ounded-Ut-f-in-RKHS}
\frac 1n \sum_{i=1}^n \frac{\bth{\bU^{\top}f(\bS')}_i^2}{\hlambda_i} \le \gamma_0^2.
\eal

\end{lemma}

\begin{lemma}
\label{lemma:auxiliary-lemma-1}
For any positive real number $a \in (0,1)$ and natural number $t$,
we have
\bal\label{eq:auxiliary-lemma-1}
(1-a)^t \le e^{-ta} \le \frac{1}{eta}.
\eal
\end{lemma}
\begin{proof}
The result follows from the facts that
$\log(1-a) \le a$ for $a \in (0,1)$ and $\sup_{u \in \RR}
ue^{-u} \le 1/e$.
\end{proof}

\noindent \textbf{Background about the Integral Operator on $\cH_{\bS}$.} Suppose $K$ is a PSD kernel defined over $\cX \times \cX$ and let the empirical Gram matrix computed by $K$ on the training features $\bS$
be $\bK_n$ with the eigenvalues
$\hlambda_1 \ge \ldots \ge \hlambda_n \ge 0$. We need the following background in the RKHS spanned by
$\set{K(\cdot,\bbx_i)}_{i=1}^n$ for the proof of Lemma~\ref{lemma:residue-f-star-low-rank}. Herein we introduce the operator $T_n \colon \cH_{\bS}  \to \cH_{\bS} $ which is  defined by $T_n g \defeq \frac 1n \sum_{i=1}^n K(\cdot,\bbx_i) g(\bbx_i)$ for every $g \in \cH_{\bS} $. It can be verified that the
eigenvalues of $T_n$ coincide with the eigenvalues of $\bK_n$, that is,
the eigenvalues of $T_n$ are
$\set{\hlambda_i}_{i=1}^n$. By the spectral theorem, all the normalized eigenfunctions of $T_n$, denoted by $\set{{\Phi}^{(k)}}_{k = 0}^{n-1}$ with ${\Phi}^{(k)} = 1/{\sqrt{n \hlambda_{k+1}}} \cdot \sum_{j=1}^n
K(\cdot,\bbx_j) \bth{\bU^{[k+1]}}_j$ for $k \in [0\colon n-1]$, is an orthonormal basis of $\cH_{\bS}$. The eigenvalue of $T_n$ corresponding to the eigenfunction ${\Phi}^{(k)}$
is $\hat \lambda_{k+1}$ for $0 \le k \le n-1$. Since $\cH_{\bS} \subseteq \cH_{K}$, we can complete $\set{{\Phi}^{(k)}}_{k = 0}^{n-1}$ so that
$\set{{\Phi}^{(k)}}_{k \ge 0}$ is an orthonormal basis of the RKHS $\cH_{K}$.

\begin{lemma}
\label{lemma:residue-f-star-low-rank}
Suppose $\delta \in (0,1/2)$ and $n \ge \Theta(\log({2}/{\delta})\cdot d^{2k_0})$. Let $\Proj_{\bUminusr} = \Proj_{\Span({\bUr})^{\perp}}$. Then with probability at least $1-2\delta$ over the random training features $\bS$,
\bal\label{eq:residue-f-star-low-rank}
\ltwonorm{\Proj_{\bUminusr}(f^*(\bS))}^2 \le n\gamma_0^2 \log{\frac{2}{\delta}} \cdot \Theta\pth{\frac{ d^{k_0} }{n}}.
\eal
\end{lemma}
\begin{proof}
We have
$\Proj_{\cH_{\bS}}(f^*) =
\sum\limits_{k=0}^{n-1} \iprod{f^*}{\Phi^{(k)}}\Phi^{(k)}$,
$\Proj_{\cH_{\bS,r_0}}(f^*) = \sum\limits_{k=0}^{r_0-1} \iprod{f^*}{\Phi^{(k)}}\Phi^{(k)}$,
and define
\bals
\bar f^{*,r_0} &\defeq \Proj_{\cH_{\bS}}(f^*) - \Proj_{\cH_{\bS,r_0}}(f^*)
=\sum\limits_{q=r_0}^n \iprod{f^{*}}{{\Phi}^{(q)}}{\Phi}^{(q)}.
\eals
Let $\bUminusr \in \RR^{n \times (n-r_0)}$ be the submatrix formed by all the columns of $\bU$ except for the top $r_0$ columns in $\bUr$.
It follows by the introduction to the space $\cH_{\bS}$ before
Lemma~\ref{lemma:residue-f-star-low-rank}
that $\set{{\Phi}^{(k)}}_{k = 0}^{n-1}$ is an orthonormal basis of $\cH_{\bS}$, and ${\Phi}^{(k)}$ is the eigenfunction of the operator $T_n$ with the corresponding eigenvalue $\hlambda_{k+1}$.
 Therefore, $\bUminusr {\bUminusr}^{\top} \Phi^{(k)}(\bS)   = 0$ for all $k \in [r_0-1]$. As a result, with probability at least $1-\delta$,
\bal\label{eq:residue-f-star-low-rank-seg1}
\frac 1n \ltwonorm{\bUminusr {\bUminusr}^{\top}
f^*(\bS)}^2 =
\frac 1n \ltwonorm{\bUminusr {\bUminusr}^{\top}
\pth{\Proj_{\cH_{\bS}}(f^*)}(\bS) }^2
= \frac 1n \sum\limits_{i=1}^n \pth{\bar f^{*,r_0}(\bbx_i)}^2
&.
\eal
We have
\bal\label{eq:residue-f-star-low-rank-seg2}
\iprod{T_n \bar f^{*,r_0}}{\bar f^{*,r_0}}
=  \iprod{\frac 1n \sum\limits_{i=1}^n K(\cdot,\bbx_i) \bar f^{*,r_0}(\bbx_i)}{\bar f^{*,r_0}}_{\cH_K}
=  \frac 1n \sum\limits_{i=1}^n \pth{\bar f^{*,r_0}(\bbx_i)}^2.
\eal
On the other hand, with probability $1-\delta$,
\bal\label{eq:residue-f-star-low-rank-seg3}
\iprod{T_n \bar f^{*,r_0}}{\bar f^{*,r_0}}
&=\iprod{T_n \sum\limits_{q=r_0}^n \iprod{\bar f^{*,r_0}}{{\Phi}^{(q)}}{\Phi}^{(q)} }{ \sum\limits_{q=r_0}^n \iprod{\bar f^{*,r_0}}{{\Phi}^{(q)}}{\Phi}^{(q)} }_{\cH_K} \nonumber \\
&= \sum\limits_{q=r_0}^n \hlambda_{q+1}\iprod{\bar f^{*,r_0}}{{\Phi}^{(q)}}^2
\le\ \hlambda_{r_0+1} \sum\limits_{q=r_0}^n \iprod{f^*}{{\Phi}^{(q)}}^2
\stackrel{\circled{1}}{\le} \hlambda_{r_0+1}  \zeta_{n,\gamma_0,r_0,\delta},
\eal
where $\circled{1}$ is due to Theorem~\ref{theroem:low-dim-target-function-satisfy-assump}.
It follows from (\ref{eq:residue-f-star-low-rank-seg1})-(\ref{eq:residue-f-star-low-rank-seg3}) that
\bal\label{eq:residue-f-star-low-rank-seg4}
\frac 1n \ltwonorm{\bUminusr {\bUminusr}^{\top}
f^*(\bS)}^2 \le \hlambda_{r_0+1}  \zeta_{n,\gamma_0,r_0,\delta}.
\eal
We now find the upper bound for $\hlambda_{r_0+1}$. It follows from Lemma~\ref{lemma:gap-spectrum-Tk-Tn}
that $\abth{\lambda_j - \hlambda_j} \le 2\sqrt{\frac{2\log{\frac{2}{\delta}}}{n}}$ for all $j \in [n]$ with probability at least $1-\delta$. Furthermore, it follows from
Theorem~\ref{theorem:eigenvalue-NTK} that $\lambda_{r_0} = \mu_{k_0+1} = \Theta(d^{-k_0-1})$ with $r_0 = m_{k_0}$. As a result, we have
\bal\label{eq:residue-f-star-low-rank-seg5}
\hlambda_{r_0+1} \le \lambda_{r_0} + 2\sqrt{\frac{2\log{\frac{2}{\delta}}}{n}}
\le \Theta(d^{-k_0}),
\eal
where the last inequality holds with probability $1-\delta$ over $\bS$ due to Lemma~\ref{lemma:gap-spectrum-Tk-Tn} and $n \ge \Theta(\log({2}/{\delta})\cdot d^{2k_0})$.
It then follows from (\ref{eq:residue-f-star-low-rank-seg4}) and (\ref{eq:residue-f-star-low-rank-seg5}) that
\bals
\frac 1n \ltwonorm{\bUminusr {\bUminusr}^{\top}
f^*(\bS)}^2 &\le \Theta(d^{-k_0}) \cdot \zeta_{n,\gamma_0,r_0,\delta}
= \Theta(d^{-k_0})  \cdot \frac{32\gamma_0^2 \log{\frac{2}{\delta}} }{\pth{\mu_{k_0} - \mu_{k_0+1}}^2 n}
\nonumber \\
&= \Theta(d^{-k_0}) \cdot  \frac{32\gamma_0^2 \log{\frac{2}{\delta}} }{\pth{\Theta(d^{-k_0}) - \Theta(d^{-k_0-1})}^2 n}
\nonumber \\
&=  \gamma_0^2 \log{\frac{2}{\delta}} \cdot \Theta\pth{\frac{ d^{k_0} }{n}},
\eals
which proves (\ref{eq:residue-f-star-low-rank}).

\end{proof}

\begin{lemma}
[{\citep[Proposition 10]{RosascoBV10-integral-operator}}]
\label{lemma:gap-spectrum-Tk-Tn}

Let $\delta \in (0,1)$, then with probability $1-\delta$ over the training features $\bS$, for all $j \in [n]$,
\bal\label{eq:gap-spectrum-Tk-Tn}
&
\abth{\lambda_{j-1} - \hlambda_j} \le 2\sqrt{\frac{2\log{\frac{2}{\delta}}}{n}}.
\eal%
\end{lemma}
\begin{remark}
We remark that the sequence $\set{\lambda_j}_{j \ge 0}$ starts with index $0$, so that $\lambda_{j-1} $ is in fact the $j$-th element in the extended enumeration of the distinct eigenvalues of $T_K$. The extended enumeration \citep{RosascoBV10-integral-operator}
of the distinct eigenvalues of $T_K$ is a sequence where each nonzero eigenvalue of $T_K$ appears as many times as its multiplicity and the other values (if any) are zero.
\end{remark}

\begin{lemma}\label{lemma:pupulartion-RC-low-rank-proj}
Suppose $n \ge r_0$. Then with probability at least $1-\delta$ over the random training features $\bS$, for every $r > 0$, we have
\bal\label{eq:pupulartion-RC-low-rank-proj}
\cfrakR \pth{\set{f-f^* \colon f \in \cF(B_h,w,\bS,r_0), \Expect{P}{(f-f^*)^2} \le r}}
\le  \sqrt{\log{\frac{2}{\delta}} } \cdot \Theta\pth{\frac{d^{k_0}}{n}} + \sqrt{\frac{rr_0}{n}}  + 2w.
\eal
\end{lemma}

\begin{proof}
Let $\cH_{K,r_0} = \overline{\Span{\set{v_q}_{q=0}^{r_0-1}}}$ be the subspace in $\cH_K$ spanned by $\set{v_q}_{q=0}^{r_0-1}$, and we
define  $\hat \cF_r \defeq \set{f \in \cF(B_h,w,\bS,r_0), \Expect{P}{(f-f^*)^2} \le r}$.
For every $f \in \hat \cF_r$, we have
$f = h+e$ such that $\supnorm{e} \le w$ and $h \in \cH_K(B_h)\cap \cH_{\bS,r_0}$, and
$\Expect{P}{(h-f^*)^2} \le 2(r+w^2)$. Furthermore, we have $\Proj_{\cH_{r_0}}(h) = \sum_{j = 0}^{r_0-1} \alpha_j v_j$
with $\alpha_j = \iprod{h}{v_j}_{\cH_K}$ for all $j \ge 0$.
We define $\bar h =h - \Proj_{\cH_{r_0}}(h)$, then $\bar h\in \cH_K(B_h)$. We have
\bal\label{eq:pupulartion-RC-low-rank-proj-seg1}
\Expect{P}{{\bar h}^2}
&= \Expect{P}{\pth{\sum\limits_{j \ge r_0} \alpha_j v_j }^2}
= \sum\limits_{j \ge r_0} \alpha_j^2 \lambda_j
\le \lambda_{r_0} \cdot \sum\limits_{j \ge r_0} \alpha_j^2
\stackrel{\circled{1}}{\le}
\lambda_{r_0} \zeta_{n,B_h,r_0,\delta} \nonumber \\
&\stackrel{\circled{2}}{\le} \log{\frac{2}{\delta}} \cdot \Theta\pth{\frac{ d^{k_0} }{n}} \defeq r_{n,k_0,\delta},
\eal
where $\circled{1}$ holds with probability at least $1-\delta$
 over $\bS$ by (\ref{eq:training-func-low-rank-residue}) of
Theorem~\ref{theroem:low-dim-target-function-satisfy-assump}, and $\circled{2}$ follows from the similar argument in the last part of the proof
of Lemma~\ref{lemma:residue-f-star-low-rank} with
$\lambda_{r_0} = \mu_{k_0+1} = \Theta(d^{-k_0-1})$.
It then follows from (\ref{eq:pupulartion-RC-low-rank-proj-seg1})
and the Cauchy-Schwarz inequality that for every $f \in \hat \cF_r$,
\bal\label{eq:pupulartion-RC-low-rank-proj-seg2}
\Expect{P}{(\Proj_{\cH_{r_0}}(h)-f^*)^2}
&\le 2 \Expect{P}{(h- f^*)^2} +
 2\Expect{P}{{\bar h}^2}
 \le 4(r+w^2)+ 2r_{n,k_0,\delta}.
\eal
We then have
\bal\label{eq:pupulartion-RC-low-rank-proj-seg3}
&\cfrakR \pth{\set{\Proj_{\cH_{r_0}}(h)-f^* \colon  f \in  \hat \cF_r}} \nonumber \\
&\stackrel{\circled{3}}{\le} \cfrakR \pth{\set{\Proj_{\cH_{r_0}}(h)-f^* \colon\Expect{P}{(\Proj_{\cH_{r_0}}(h)-f^*)^2}\le4(r+w^2)+ 2r(n,k_0,\delta)} }  \nonumber \\
&\stackrel{\circled{4}}{\le}
2\cfrakR \pth{\set{f \in \cH_{K^{(r_0)}}(B_h) \colon  \Expect{P}{f^2}\le  r+w^2+ \frac{r_{n,k_0,\delta}}{2}} }
\stackrel{\circled{5}}{\le}
\sqrt{r+w^2+ \frac{r_{n,k_0,\delta}}{2}} \cdot \sqrt{\frac{r_0}{n}}.
\eal
Here $\circled{3}$ follows from (\ref{eq:pupulartion-RC-low-rank-proj-seg2}).
Since $\Proj_{\cH_{r_0}}(f),f^* \in \cH_{r_0} \cap \cH_K(B_h) \subseteq \cH_{K^{(r_0)}}(B_h)$, we have $(\Proj_{\cH_{r_0}}(f)-f^*)/2 \in \cH_{K^{(r_0)}}(B_h)$ due to the fact that
$\cH_{K^{(r_0)}}(B_h)$ is symmetric and convex, and
it follows that $\circled{4}$
holds.  $\circled{5}$ follows from Lemma~\ref{lemma:LRC-population-NN} with $Q = r_0$ in (\ref{eq:varphi-LRC-population-NN}) of
Lemma~\ref{lemma:LRC-population-NN}.

We then derive the upper bound for
$\cfrakR \pth{\set{\bar h \colon f \in  \hat \cF_r}}$. First, it follows from Theorem~\ref{theroem:low-dim-target-function-satisfy-assump} and the argument similar to (\ref{eq:pupulartion-RC-low-rank-proj-seg1}) that
\bal\label{eq:pupulartion-RC-low-rank-proj-bar-h-RKHS-norm}
\norm{\bar h}{\cH_K}^2 =\sum\limits_{j \ge r_0} \alpha_j^2
\le \zeta_{n,B_h,r_0,\delta} \le B_h^2 \log{\frac{2}{\delta}} \cdot \Theta\pth{\frac{ d^{2k_0} }{n}} \defeq B_{\bar h}^2.
\eal
We then have
\bal\label{eq:pupulartion-RC-low-rank-proj-seg4}
&\cfrakR \pth{\set{\bar h \colon f \in  \hat \cF_r}}
= \Expect{\set{\bbx_i}_{i=1}^n, \set{\sigma_i}_{i=1}^n}{\sup_{\bar h \in \cH_K(B_{\bar h})} {\frac{1}{n} \sum\limits_{i=1}^n {\sigma_i}{\bar h(\bbx_i)}} } \nonumber \\
&\le \frac{B_{\bar h}}{n}\Expect{\set{\bbx_i}_{i=1}^n, \set{\sigma_i}_{i=1}^n}{\sup_{f \in \cF} { \norm{\sum\limits_{i=1}^n \sigma_i K(\cdot,\bbx_i)}{\cH_K}  }}
\le \frac{B_{\bar h}}{\sqrt n} \le \sqrt{\log{\frac{2}{\delta}}}
\cdot \Theta\pth{\frac{ d^{k_0} }{n}}.
\eal
Finally, it follows from
(\ref{eq:pupulartion-RC-low-rank-proj-seg3}) and (\ref{eq:pupulartion-RC-low-rank-proj-seg4})
that
\bal
&\cfrakR \pth{\set{f-f^* \colon f \in  \hat \cF_r}}
\le \cfrakR \pth{\set{\Proj_{\cH_{r_0}}(h)-f^* \colon  f \in  \hat \cF_r}} +\cfrakR \pth{\set{\bar h \colon f \in  \hat \cF_r}}  + w \nonumber \\
&\le  \sqrt{\log{\frac{2}{\delta}} } \cdot \Theta\pth{\frac{d^{k_0}}{n}}
+w  \sqrt{\frac{r_0}{n}} + \sqrt{\frac{rr_0}{n}} + w
\le \ \sqrt{\log{\frac{2}{\delta}} } \cdot \Theta\pth{\frac{d^{k_0}}{n}} + \sqrt{\frac{rr_0}{n}}  + 2w,
\eal
which proves (\ref{eq:pupulartion-RC-low-rank-proj}).

\end{proof}

\begin{lemma}
[{\citet[Lemma C.9]{Yang2025-generalization-two-layer-regression},\citet[Lemma VI.4]{yang2024gradientdescentfindsoverparameterized}}]
\label{lemma:LRC-population-NN}
For every $B,w > 0$, the function class $\cF(B,w)$ is defined as
$\cF(B,w) \defeq \set{f \colon f = h+e, h \in \cH_{K}(B),
\supnorm{e} \le w}$. Then for every $r > 0$,
\bal\label{eq:LRC-population-NN}
&\cfrakR
\pth{\set{f \in \cF(B,w) \colon \Expect{P}{f^2} \le r}}
\le \varphi_{B,w}(r),
\eal%
where
\bal\label{eq:varphi-LRC-population-NN}
\varphi_{B,w}(r) &\defeq
\min_{Q \colon Q \ge 0} \pth{({\sqrt r} + w) \sqrt{\frac{Q}{n}} +
B
\pth{\frac{\sum\limits_{q = Q+1}^{\infty}\lambda_q}{n}}^{1/2}} + w.
\eal
\end{lemma}

\begin{lemma}
[{\citet[Lemma B.9]{yang2024gradientdescentfindsoverparameterized}}]
\label{lemma:sub-root-fix-point-properties}
Suppose $\psi \colon [0,\infty) \to [0,\infty)$ is a sub-root function with the unique fixed point $r^*$. Then the following properties hold.

\begin{itemize}[leftmargin=8pt]
\item[(1)] Let $a \ge 0$, then $\psi(r) + a$ as a function of $r$ is also a sub-root function with fixed point $r^*_a$, and
$r^* \le r^*_a \le r^* + 2a$.
\item[(2)] Let $b \ge 1$, $c \ge 0$ then $\psi(br+c)$ as a function of $r$ is also a sub-root function with fixed point $r^*_b$, and
$r^*_b \le br^* +2c/b$.
\item[(3)] Let $b \ge 1$, then $\psi_b(r) = b\psi(r)$ is also a sub-root function with fixed point $r^*_b$, and
$r^*_b \le b^2r^* $.
\end{itemize}
\end{lemma}

\subsection{Proofs of Theorem~\ref{theroem:low-dim-target-function-satisfy-assump}}
\label{sec:remaining-results}

\begin{proof}
{\textbf{\textup{\hspace{-3pt}of Theorem~\ref{theroem:low-dim-target-function-satisfy-assump}}}.}
With probability at least $1-\delta$, we have
\bals
\sum\limits_{q=r_0}^{\infty} \iprod{f^*}{{\Phi}^{(q)}}_{\cH_K}^2
&=\sum\limits_{q=r_0}^{\infty} \iprod{\sum\limits_{j=0}^{r_0-1} \beta_j v_j }{{\Phi}^{(q)}}_{\cH_K}^2
\le \sum\limits_{q=r_0}^{\infty} \sum\limits_{j=0}^{r_0-1} \beta_j^2  \cdot \sum\limits_{j=0}^{r_0-1} \iprod{v_j}{{\Phi}^{(q)}}^2
\nonumber \\
&\le \gamma_0^2 \sum\limits_{q=r_0}^{\infty} \sum\limits_{j=0}^{r_0-1}
\iprod{v_j}{{\Phi}^{(q)}}^2 \le \frac{32\gamma_0^2\log{\frac{2}{\delta}} }{\pth{\mu_{k_0} - \mu_{k_0+1}}^2 n}.
\eals
Here the last inequality follows by
Lemma~\ref{lemma:small-correlation-eigenvector-eigenfunction} with
$\tau_0^2 = 1$ and $m_{k_0} = r_0$, which proves
(\ref{eq:target-func-low-rank-residue}).
Since $f \in \cF(B_h,w,\bS,r_0)$, we have
$f = \sum_{j=0}^{r_0-1} \alpha_j {\Phi}^{(j)}$ with
$\alpha_j = \iprod{f}{{\Phi}^{(j)}}_{\cH_K}$
for $j\in[0,r_0-1]$. Following a similar argument, we have
\bals
\sum\limits_{q=r_0}^{\infty} \iprod{f}{v_q}_{\cH_K}^2
&=\sum\limits_{q=r_0}^{\infty} \iprod{\sum_{j=0}^{r_0-1} \alpha_j {\Phi}^{(j)} }{v_q}_{\cH_K}^2
\le \sum\limits_{q=r_0}^{\infty} \sum\limits_{j=0}^{r_0-1}
\alpha_j^2  \cdot \sum\limits_{j=0}^{r_0-1} \iprod{{\Phi}^{(j)}}{v_q}^2
\nonumber \\
&\le B_h^2 \sum\limits_{q=r_0}^{\infty} \sum\limits_{j=0}^{r_0-1}
\iprod{{\Phi}^{(j)}}{v_q}^2 \le \frac{32B_h^2\log{\frac{2}{\delta}} }{\pth{\mu_{k_0} - \mu_{k_0+1}}^2 n},
\eals
which proves (\ref{eq:training-func-low-rank-residue}).
\end{proof}

\begin{lemma}\label{lemma:small-correlation-eigenvector-eigenfunction}
Let $\sup_{\bx \in \cX} K(\bx,\bx) = \tau_0^2$. For any $\delta \in (0,1)$, with probability at least $1-\delta$ over the random training features $\bS$,
\bal
\sum_{i = 0}^{m_{k_0}-1}\sum_{j \ge m_{k_0}} \iprod{\Phi^{(i)}}{v_j}_{\cH}^2 +
\sum_{i \ge m_{k_0}}\sum_{j = 0}^{m_{k_0}-1} \iprod{\Phi^{(i)}}{v_j}_{\cH}^2 &\le \frac{32   \tau_0^4 \log{\frac{2}{\delta}} }{\pth{\mu_{k_0} - \mu_{k_0+1}}^2 n}. \label{eq:small-correlation-eigenvector-eigenfunction-1}
\eal%

\end{lemma}

\begin{proof}
Define operator $T_n \colon \cH_K \to \cH_K$ by $T_n g = \frac 1n \sum_{i=1}^n K(\cdot,\bbx_i)g(\bbx_i)$ as introduced before
Lemma~\ref{lemma:residue-f-star-low-rank}, and let $\set{{\Phi}^{(k)}}_{k \ge 0}$ be an orthonormal basis of the RKHS $\cH_K$.

Let $P^T_N$ be an orthogonal projection operator which projects any input onto the subspace spanned by eigenfunctions corresponding to the top $N$ eigenvalues of the operator $T$, and $T$ is defined on the RKHS $\cH$.

We now work on the following two orthogonal projection operators, $P^{T_K}_{m_{k_0}}$ and $P^{T_n}_{m_{k_0}}$. Each of the two operators projects its input onto the space spanned by all the eigenfunctions of the corresponding operator, that is.
\bal\label{eq:small-correlation-eigenvector-eigenfunction-seg1}
&P^{T_K}_{m_{k_0}} h = \sum_{j = 0}^{m_{k_0}-1} \iprod{h}{v_j}_{\cH} v_j, \quad  P^{T_n}_{m_{k_0}} h = \sum_{j = 0}^{m_{k_0}-1} \iprod{h}{\Phi^{(j)}}_{\cH} \Phi^{(j)}.
\eal%
The Hilbert-Schmidt norm of $P^{T_K}_{m_{k_0}}-P^{T_n}_{m_{k_0}}$ is
\bal\label{eq:small-correlation-eigenvector-eigenfunction-seg2}
\norm{P^{T_K}_{m_{k_0}}-P^{T_n}_{m_{k_0}}}{\textup{HS}}^2 = \sum_{i \ge 0, j \ge 0} \iprod{ \pth{ P^{T_K}_{m_{k_0}}-P^{T_n}_{m_{k_0}} }\Phi^{(i)}   }{v_j}_{\cH}^2,
\eal%
which is due to the fact that both $\set{\Phi^{(j)}}_{j \ge 0}$ and $\set{v_j}_{j \ge 0}$ are orthonormal bases of $\cH$. It can be verified that

\bal\label{eq:small-correlation-eigenvector-eigenfunction-seg3}
\iprod{ \pth{ P^{T_K}_{m_{k_0}}-P^{T_n}_{m_{k_0}} }\Phi^{(i)}   }{v_j}_{\cH}
=\begin{cases}
      0 & \textup{if $i < m_{k_0}, j <  m_{k_0}$}, \\
      -\iprod{\Phi^{(i)}}{v_j}_{\cH} & \textup{if $i < m_{k_0}, j \ge m_{k_0}$}, \\
      \iprod{\Phi^{(i)}}{v_j}_{\cH} & \textup{if $i \ge m_{k_0}, j < m_{k_0}$}, \\
      0  & \textup{if $i \ge m_{k_0}, j \ge m_{k_0}$,}
   \end{cases}
\eal%
and similar results are obtained in the proof of \citet[Theorem 12]{RosascoBV10-integral-operator}.

Because $T_K$ and $T_n$ are Hilbert-Schmidt operators, by \citet[Theorem 7]{RosascoBV10-integral-operator}, for all $\delta \in (0,1)$, with probability at least $1-\delta$,
\bal\label{eq:small-correlation-eigenvector-eigenfunction-seg4}
\norm{T_K-T_n}{\textup{HS}} &\le \frac{{2{\sqrt 2}\tau_0^2} \sqrt{\log{\frac{2}{\delta}}}}{\sqrt n}.
\eal%
When $n \ge \frac{128 \tau_0^4 \log{\frac{2}{\delta}} }{\pth{\mu_{k_0} - \mu_{k_0+1}}^2}$,
$\norm{T_K-T_n}{\textup{HS}} \le \frac{\mu_{k_0} - \mu_{k_0+1}}{4}$. It follows from \citet[Proposition 6]{RosascoBV10-integral-operator} and noting that the operator norm in \citet[Proposition 6]{RosascoBV10-integral-operator} can be replaced by the Hilbert-Schmidt norm,
\bal\label{eq:small-correlation-eigenvector-eigenfunction-seg5}
\norm{P^{T_K}_{m_{k_0}}-P^{T_n}_{m_{k_0}}}{\textup{HS}}^2 &\le \frac{4}{\pth{\mu_{k_0} - \mu_{k_0+1}}^2} \norm{T_K-T_n}{\textup{HS}}^2 \le \frac{32\tau_0^4 \log{\frac{2}{\delta}} }{\pth{\mu_{k_0} - \mu_{k_0+1}}^2 n}.
\eal%
Then (\ref{eq:small-correlation-eigenvector-eigenfunction-1}) follows from (\ref{eq:small-correlation-eigenvector-eigenfunction-seg2}), (\ref{eq:small-correlation-eigenvector-eigenfunction-seg3}),
and (\ref{eq:small-correlation-eigenvector-eigenfunction-seg5}).

%
%

\end{proof}

\subsection{Results about Eigenvalues of the Integral Operators}
\label{sec:NTK-eigenvalues}

The following theorem is a refined version of the Mercer's theorem on the PSD kernel $K$ defined in (\ref{eq:kernel-two-layer}), with the exact estimation about the decaying rate of the distinct eigenvalues $\set{\mu_{\ell}}_{\ell \ge 0}$.

\begin{theorem}[Eigenvalue of
the Integral Operator Associated with the NTK (\ref{eq:kernel-two-layer})]
\label{theorem:eigenvalue-NTK}
Let the distinct eigenvalues of the integral operator $T_K$ associated with the PSD kernel $K$ defined in (\ref{eq:kernel-two-layer}) be $\set{\mu_{\ell} \colon \ell \ge 0}$ with $\mu_0 > \mu_1 > \ldots$, where $\mu_{\ell}$ is the eigenvalue corresponding to $\cH_{\ell}$.
Suppose that $\bar k_0 = \Theta(1)$ and $d \ge \Theta(1)$. Then
$\mu_{k} = \Theta(d^{-k})$ for $0 \le k \le \bar k_0$. Moreover,
for all $\bx, \bx' \in \cX = \unitsphere{d-1}$,
\bal\label{eq:eigenvalue-NTK-Mercer-decomp}
K(\bx,\bx') = \sum\limits_{\ell \ge 0} \mu_{\ell} \sum\limits_{j=1}^{N(d,\ell)} Y_{\ell j}(\bx)Y_{\ell j}(\bx')
=  \sum\limits_{\ell \ge 0} \mu_{\ell} N(d,\ell) P_{\ell}(\iprod{\bx}{\bx'}),
\eal
where $\mu_{\ell}$ is the eigenvalue of the integral operator $T_K$ associated with $K$ corresponding to $\cH_{\ell}$, and $\set{Y_{\ell j}}_{j=1}^{N(d,\ell)}$ are the eigenfunctions corresponding to the eigenvalue $\mu_{\ell}$. That is, $T_K Y_{\ell j} = \mu_{\ell} Y_{\ell j}$ for all $\ell \ge 0$ and $j \in [N(d,\ell)]$. The series on the RHS of
(\ref{eq:eigenvalue-NTK-Mercer-decomp}) converges absolutely and uniformly
on $\cX \times \cX$.
\end{theorem}
\begin{proof}
(\ref{eq:eigenvalue-NTK-Mercer-decomp}) follows from the background about Harmonic Analysis on spheres in Section~\ref{sec:harmonic-analysis-detail} and the Mercer's theorem.
Since $K$ is a continuous PSD kernel defined on the compact set $\cX \times \cX$, it follows from the Mercer's theorem again that the series on the RHS of (\ref{eq:eigenvalue-NTK-Mercer-decomp}) converges absolutely and uniformly on $\cX \times \cX$ to $K$.

We now set to compute the eigenvalues $\set{\lambda_{k} \colon 0 \le  k \le \bar k_0}$.
Let the distinct eigenvalues of the PSD kernel $K^{(0)}$, which is defined in (\ref{eq:kernel-two-layer}) and repeated below
\bals
K^{(0)}(\bx,\bx') = \Expect{\bw \sim \cN(\bzero,\bI_d)}
{\indict{\bx^{\top}\bw \ge 0} \indict{\bx'^{\top}\bw\ge 0}}
= \frac{\pi-\arccos ({\bx}^{\top}{\bx'})}{2\pi},
\quad \forall \bx,\bx' \in \cX = \cX,
\eals
be $\set{\lambda_{0,k} \colon k \ge 0}$, where
$\lambda_{0,k}$ is the eigenvalue corresponding to  $\cH_{k}$, the space of degree-$\ell$ homogeneous harmonic polynomials on $\cX = \unitsphere{d-1}$.

Define
\bals
s_k \defeq \frac{\omega_{d-2}}{\omega_{d-1}}
\int_{-1}^1 \indict{t \ge 0} P_k(t) (1-t^2)^{(d-3)/2} \diff t,
\eals
It then follows by the computation in
\citet[Section D.2]{Bach17-breaking-curse-dim}
that $s_0 = \Theta(1)$. Also, for all $t \in \NN$,
 $s_{2t} = 0$,
 and
\bals
s_{2t-1} &= \frac{\omega_{d-2}}{\omega_{d-1}}
\pth{\frac 12}^{2t-1} (-1)^{t-1}
 \frac{\Gamma((d-1)/2)\Gamma(2t-1)}
 {\Gamma(t) \Gamma(t+(d-1)/2)}
 \\
&\stackrel{\circled{1}}{\asymp} (-1)^{t-1}{\sqrt d}\frac{(d-1)^{\frac d2-1}(2t-1)^{2t-1.5}}{(2t)^{t-0.5}(2t+d-1)^{t+\frac d2-1}}
\stackrel{\circled{2}}{\asymp}\frac{1}{d^{t-0.5}},
\eals
where we used the approximation to the Gamma function \citep{Gosper1978-dr}
$\Gamma(x) \asymp x^{x-0.5} \exp(-x) \sqrt{2\pi}$ and the fact that
$\frac{\omega_{d-2}}{\omega_{d-1}} \asymp \sqrt d$ in $\circled{1}$.
 $\circled{2}$ is due to $t = \Theta(1)$.

It follows from \citet{BiettiM19} that
$\lambda_{0,k} = s_k^2$ for all $k \ge 0$. When
$k = 2t-1$ for $t \in \NN$, we have $\lambda_{0,k} = s_k^2 =\Theta(d^{-(2t-1)})
=\Theta(d^{-k})$. Moreover, $\lambda_{0,k} = 0$ for all $k = 2t$ with $t \in \NN$, and $\lambda_{0,0} = s_0^2 = \Theta(1)$. As a result,
we have $\lambda_{0,0} = \Theta(1)$, and
\bal\label{eq:eigenvalue-NTK-seg1}
\lambda_{0, k} =
\begin{cases}
0 & k = 2t, t \in \NN, k \le \bar k_0, \\
\Theta(d^{-k})
& k = 2t-1, t \in \NN, k \le \bar k_0.
\end{cases}
\eal
Let the distinct eigenvalues of the PSD kernel $K^{(1)}$ which is also defined in (\ref{eq:kernel-two-layer}) be $\set{\lambda_{1,k} \colon k \ge 0}$, where $\lambda_{0,k}$ is the eigenvalue corresponding to  $\cH_{k}$.
Define $\kappa(t) = t \kappa^{(0)}(t)$ with
$\kappa^{(0)}(t) \defeq \frac{\pi - \arccos t}{2\pi}$ for
$t \in [-1,1]$. Then for $k \ge 1$ we have
\bal\label{eq:eigenvalue-NTK-seg2}
\lambda_{1,k} &=
\frac{\omega_{d-2}}{\omega_{d-1}}
\int_{-1}^1 \kappa(t) P_k(t) (1-t^2)^{(d-3)/2}  \diff t
= \frac{\omega_{d-1}}{\omega_{d-2}}
\int_{-1}^1 \kappa^{(0)}(t) t  (1-t^2)^{(d-3)/2} P_k(t) \diff t
\nonumber \\
&=\frac{\omega_{d-2}}{\omega_{d-1}}
\int_{-1}^1 \kappa^{(0)}(t)
\pth{\frac{k}{2k+d-2}P_{k-1}(t) + \frac{k+d-2}{2k+d-2}P_{k+1}(t)}   (1-t^2)^{(d-3)/2} \diff t \nonumber \\
&=\frac{k}{2k+d-2} \lambda_{0, k-1} +
\frac{k+d-2}{2k+d-2} \lambda_{0, k+1}.
\eal
Moreover,
\bal\label{eq:eigenvalue-NTK-seg3}
\lambda_{1,0} &=
\frac{\omega_{d-2}}{\omega_{d-1}}
\int_{-1}^1 \kappa(t) P_0(t) (1-t^2)^{(d-3)/2}  \diff t
=\frac{\omega_{d-2}}{\omega_{d-1}}
\int_{-1}^1  \kappa^{(0)}(t) P_1(t) (1-t^2)^{(d-3)/2}  \diff t
=\lambda_{0, 1}.
\eal
It follows from (\ref{eq:eigenvalue-NTK-seg1})-(\ref{eq:eigenvalue-NTK-seg3})
that $\lambda_{1,0} = \Theta(1/d)$, $\lambda_{1,1} = \Theta(1/d)$. Moreover,
\bal\label{eq:eigenvalue-NTK-seg4}
\lambda_{1, k} =
\begin{cases}
\Theta(d^{-k}) & k = 2t, t \in \NN, k \le \bar k_0, \\
0 & k  = 2t-1, t \in \NN, t \ge 2, k \le \bar k_0.
\end{cases}
\eal
It then follows from (\ref{eq:eigenvalue-NTK-seg1}) and
(\ref{eq:eigenvalue-NTK-seg4}) that
$\mu_k = \lambda_{0,k} + \lambda_{1,k} =\Theta(d^{-k}) $ for $0 \le k \le \bar k_0$

\end{proof}

\vskip 0.2in
\bibliography{ref}

\end{document}